\pdfoutput=1

\def\MODE{1} 
\if\MODE1

\else

\fi

\documentclass[10pt]{article}

\usepackage{palatcm}
\usepackage{graphicx}
\usepackage{subfigure} 
\usepackage{soul}
\usepackage{color, xcolor}
\usepackage[thinlines]{easytable}
\usepackage{relsize}
\usepackage{xfrac}
\usepackage{verbatim}
\usepackage{algorithm}
\usepackage[noend]{algorithmic}
\usepackage{amsmath}
\usepackage{amssymb}
\usepackage{amsthm}
\usepackage{epsfig}
\usepackage{wrapfig}
\usepackage{url}
\usepackage[colorlinks=true,citecolor=blue,linkcolor=blue]{hyperref}
\usepackage{multirow}
\usepackage{fullpage}

\usepackage[T1]{fontenc}
\usepackage{upgreek}

\newcommand{\Exp}{\mathbb{E}}
\usepackage[greek,english]{babel}
\usepackage[utf8x]{inputenc}
\usepackage{pifont}
\usepackage{float}

\usepackage{wrapfig}
\usepackage{bbm}
\usepackage{url}            
\usepackage{booktabs}       
\usepackage{amsfonts}       
\usepackage{microtype}      
\usepackage{microtype}
\usepackage{graphicx}
\usepackage{booktabs} 
\usepackage{amsthm,amsmath,amssymb}
\usepackage{float,url,amsfonts,alltt}
\usepackage{mathtools,rotating}
\usepackage{ifpdf,fancyvrb}
\usepackage{enumitem}

\usepackage{microtype}
\usepackage{graphicx}
\usepackage{booktabs} 

\usepackage{graphicx,xspace,verbatim,comment}
\usepackage{hyperref,array,color,balance,multirow}
\usepackage{balance,float,url,amsfonts,alltt}
\usepackage{mathtools,rotating,amsmath,amssymb}
\usepackage{color,ifpdf,fancyvrb,array}
\usepackage{etoolbox,listings}
\usepackage{bigstrut,morefloats}

\usepackage{pbox}
\usepackage[boxruled,algo2e]{algorithm2e}
\DeclarePairedDelimiterX{\inp}[2]{\langle}{\rangle}{#1, #2}

\newtheorem{theorem}{Theorem}

\newtheorem{remark}{Remark}

\newtheorem{lemma}[theorem]{Lemma}
\newtheorem{definition}{Definition}
\newcommand{\eat}[1]{}

\newcommand{\R}{\mathbb{R}}
\numberwithin{equation}{section}

\usepackage{amsmath}
\usepackage{accents}
\newlength{\dhatheight}

\usepackage{bbm}

\newcommand{\WH}[1]{}
\newcommand{\ST}[1]{}
\newcommand{\XP}[1]{}

\newtoggle{tr}
\toggletrue{tr}
\iftoggle{tr}{
	\makeatother
}{}

\newcommand{\systemname}{Frugal\textsc{ML}}
\usepackage{subfigure}

\title{  \systemname{}: How to Use ML Prediction APIs \\ More Accurately and Cheaply}

\author{
Lingjiao Chen, Matei Zaharia, James Zou\\
Stanford University
} 

\date{}

\begin{document}

\maketitle

\begin{abstract}
Prediction APIs offered for a fee are a fast-growing industry and an important part of machine learning as a service.
While many such services are available, the
 heterogeneity in their price and  performance makes it challenging for users to decide which API or combination of APIs to use for their own data and budget.
 We  take a  first step towards addressing this challenge by proposing  \systemname{},  a principled framework that jointly learns the strength and weakness of each API on different data, and performs an efficient optimization to automatically identify the best sequential strategy to adaptively use the available APIs within a budget constraint.
 Our theoretical analysis shows that natural sparsity in the formulation can be leveraged to make \systemname{} efficient.
 We conduct systematic experiments using ML APIs from Google, Microsoft, Amazon, IBM, Baidu and other providers for tasks including facial emotion recognition, sentiment analysis and speech recognition. 
 Across various tasks, \systemname{} can achieve up to 90\% cost reduction while matching the accuracy of the best single API, or up to 5\% better accuracy while matching the best API's cost.
\end{abstract}

\section{Introduction}\label{Sec:Intro}
{Machine learning as a service (MLaaS) is a rapidly growing industry. For example, one could use Google prediction API~\cite{GoogleAPI} to classify an image for \$0.0015 or to classify the sentiment of a text passage for \$0.00025. MLaaS services are appealing because using such APIs reduces the need to develop one's own ML models. The MLaaS market size was estimated at \$1 billion in 2019, and it is expected to grow to \$8.4 billion by 2025 \cite{MLasS_MarketInfo}. 

Third-party ML APIs come with their own challenges, however. A major challenge is that different companies charge quite different amounts for similar tasks. For example, for image classification, Face++ charges \$0.0005 per image \cite{FacePPAPI}, which is 67\% cheaper than Google~\cite{GoogleAPI}, while Microsoft charges \$0.0010 \cite{MicrosoftAPI}. Moreover, the prediction APIs of different providers perform better or worse on different types of inputs. For example, accuracy disparities in  gender classification were observed for different skin colors \cite{pmlr-v81-buolamwini18a,kim2019multiaccuracy}. As we will show later in the paper, these APIs' performance also varies by class---for example, we found that on the FER+ dataset, the Face++ API had the best accuracy on \emph{surprise} images while the Microsoft API had the best performance on \emph{neutral} images. The more expensive APIs are not uniformly better; and APIs tend to have specific classes of inputs where they perform better than alternatives. 
This heterogeneity in price and in performance makes it challenging for users to decide which API or combination of APIs to use for their own data and budget.  

In this paper, we propose \systemname{}, a principled framework to address this challenge. \systemname{} jointly learns the strength and weakness of each API on different data, then performs an efficient optimization to automatically identify the best adaptive strategy to use all the available APIs given the user's budget constraint. \systemname{} leverages the modular nature of APIs by designing adaptive strategies that can call APIs sequentially. For example, we might first send an input to API A. If A returns the label ``dog'' with high confidence---and we know A tends to be accurate for dogs---then we stop and report ``dog''. But if A returns ``hare'' with lower confidence, and we have learned that A is less accurate for ``hare,'' then we might adaptively select a second API B to make additional assessment.

\systemname{} optimizes such adaptive strategies to substantially improve prediction performance over simpler approaches such as model cascades with a fixed quality threshold (Figure~\ref{fig:FAME:Example}). Through  experiments with real commercial ML APIs on diverse tasks, we observe that \systemname{} typically reduces costs more than 50\% and sometimes up to 90\%. Adaptive strategies are challenging to learn and optimize, because the choice of the 2\textsuperscript{nd} predictor, if one is chosen, could depend on the prediction and confidence of the first API, and because \systemname{} may need to allocate different fractions of its budget to predictions for different classes. We prove that under quite general conditions, there is natural sparsity in this problem that we can leverage to make \systemname{} efficient. }

\begin{figure}
	\centering
	\includegraphics[width=1\linewidth]{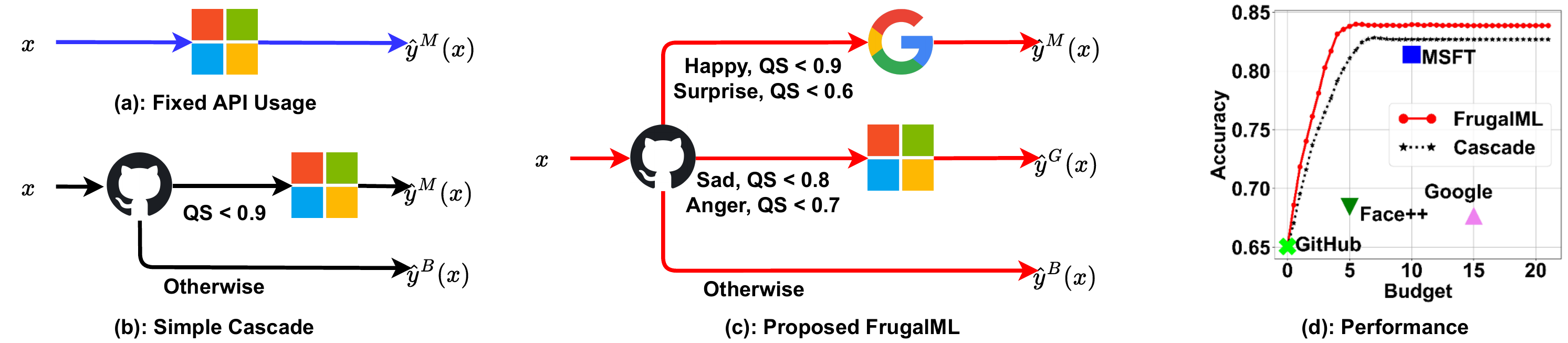}
	\caption{Comparison of different approaches to use ML APIs. Naively calling a fixed API in (a) provides a fixed cost and accuracy. The simple cascade in (b) uses the quality score (QS) from a low-cost open source model to decide whether to call an additional service. Our proposed \systemname{} approach, in (c), exploits both the quality score and  predicted label to select APIs. Figure (d) shows the benefits of \systemname{} on FER+, a facial emotion dataset. }
	\label{fig:FAME:Example}
\end{figure}

\textbf{Contributions} To sum up, our contributions are:
\begin{enumerate}

    \item We formulate and study the problem of learning to optimally use commercial ML APIs given a budget. This is a growing area of importance and is under-explored. 

    \item We propose \systemname{}, a framework that jointly learns the strength and weakness of each API, and performs an optimization to identify the best strategy for using those APIs within a budget constraint.  
    By leveraging natural sparsity in this optimization problem, we design an efficient algorithm to solve it with provable guarantees.
    
    \item We evaluate  \systemname{}  using real-world APIs from diverse providers (e.g., Google, Microsoft, Amazon, and Baidu) for classification tasks including facial emotion recognition, text sentiment analysis, and speech recognition. 
    We find that \systemname{} can match the accuracy of the best individual API with up to 90\% lower cost, or significantly improve on this accuracy, up to 5\%, with the the same cost.
    
    \item We release our dataset of 612,139 samples annotated by commercial APIs as a broad resource to further investigate differences across APIs and improve usage.
\end{enumerate}

\paragraph{Related Work.} 
\textbf{MLaaS:} With the growing importance  of MLaaS APIs \cite{AmazonAPI, BaiduAPI,FacePPAPI, GoogleAPI, IBMAPI, MicrosoftAPI}, existing research has largely focused on evaluating individual API for their performance  \cite{MLasS_EmpiricalAnalysis2017},  robustness  \cite{MLasS_GoogleNotRobust_2017}, and applications \cite{ pmlr-v81-buolamwini18a, MLasS_GoogleDigitalMedia_2019, MLasS_Google_Microsoft_Blind18}.
On the other hand, \systemname{} aims at finding strategies to select from or use multiple APIs to reduce costs and increase accuracy.   

\textbf{Mixtures of Experts:}  A natural approach to exploiting multiple predictors is mixture of experts \cite{MixtureExpert_JJ94, MixtureExpert_First1991, MixtureExpertsSurvey12}, which uses a gate function to decide which expert to use.
Substantial research has focused on developing gate function models, such as SVMs \cite{mixtureexpertSVM02, Mixture_Experts_SVMNIPS}, Gaussian Process \cite{MixtureExpertGP15,MixtureExpert_DP2011}, and neutral networks  
\cite{MixtureExpert_DNN_MMDLHD17, MixtureExpert_NN2019}.
However, applying mixture of experts for MLaaS would result in fixed cost and thus would not allow users to specify a budget constraint as in \systemname{}. 
As we will show later, sometimes \systemname{} with  a budget constraint can even outperform mixture of experts algorithms while using less budget.

\textbf{Model Cascades:} Cascades consisting of a sequence of models are useful to balance the quality and runtime of inference \cite{Viola01robustreal_time, ModelCascade_NIPS2001, PedestrianDetectionCascadeCai15,Noscope_kang17, FacePointCascade13, ModelCascade_Linear11, CascadeXu14,ModelCascade_facedetection}.
While model cascades use predicted quality score \textit{alone} to avoid calling computationally expensive models,
\systemname{}' strategies can utilize {both quality score and predicted class} to select a downstream expensive add-on service.
Designing such strategies requires solving a significantly harder optimization problem, e.g., choosing how to divide the available budget between classes (\S\ref{Sec:FAME:Theory}), but also improves performance substantially over using the quality score alone (\S\ref{Sec:FAME:Experiment}). 

\section{Preliminaries}\label{Sec:FAME:Preli}

\paragraph{Notation.}
In our exposition, we denote matrices and vectors in bold, and scalars, sets,  and functions in standard script. We let $\mathbf{1}_m$ denote the $m\times 1$ all ones vector, while $\mathbf{1}_{n\times m}$ denotes the all ones $n\times m$ matrix. We define $\mathbf{0}_{m}, \mathbf{0}_{n\times m}$ analogously.
The subscripts are omitted when clear from context. 
Given a matrix $\mathbf{A} \in \mathbb{R}^{n\times m}$, we let $\mathbf{A}_{i,j}$ denote its entry at location $(i,j)$, $\mathbf{A}_{i,\cdot} \in \mathbb{R}^{1\times m}$ denote its $i$th row, and $\mathbf{A}_{\cdot, j} \in \mathbb{R}^{n\times 1}$ denote its $j$th column. Let $[n]$ denote $\{1,2,\cdots,n\}$.
Let $\mathbbm{1}$ represent the indicator function.

\paragraph{ML Tasks.}
Throughout this paper, we focus on (multiclass) classification tasks, where the goal is to classify a data point $x$ from a distribution $D$ into $L$ label classes. 
Many real world ML APIs aim at such tasks, including facial emotion recognition, where $x$ is a face image and label classes are  emotions (happy, sad, etc), and  text sentiment analysis, where $x$ is a text passage and the label classes are attitude sentiment (either positive or negative). 

\paragraph{MLaaS Market.}
Consider a MLaaS market consisting of $K$ different ML services which aim at the same classification task.
Taken a data point $x$ as input, the $k$th service returns to the user a predicted label $y_k(x)\in [L]$ and its quality score $q_k(x) \in [0,1]$, where
larger score indicates higher confidence of its prediction. 
This is typical for many popular APIs.
There is also a unit cost associated with each service. 
Let the vector $\mathbf{c} \in \R^{K}$
denote the unit cost of all services.
Then $\mathbf{c}_k=0.005$ simply means that  users need to pay $0.005$ every time they call the $k$th service. 
We use $y(x)$ to denote $x$'s true label, and let $r^k(x)\triangleq \mathbbm{1}_{y_k(x)=y(x)}$ be the reward of using the $k$ service on $x$.

\section{\systemname: a Frugal Approach to Adaptively Leverage ML Services}\label{Sec:FAME:Theory}

\begin{figure}
	\centering
	\includegraphics[width=1.0\linewidth]{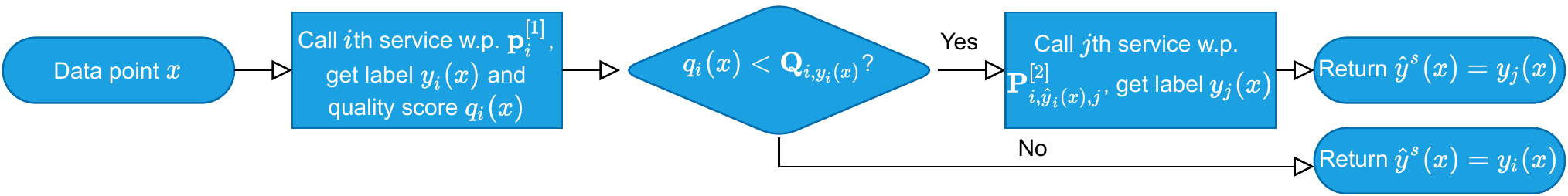}
	\caption{In \systemname{}, a base service is first selected and called. If its quality score is smaller than the threshold for its predicted label, \systemname{} chooses an add-on service to invoke and returns its prediction. Otherwise, the base service's prediction is returned.}
	\label{fig:FAME:Framework}
\end{figure}
In this section, we present \systemname{}, a formal framework for API calling strategies to obtain accurate and cheap predictions from a MLaaS market. All proofs are left to the  appendix.
We generalize the scheme in Figure  \ref{fig:FAME:Example} (c) to $K$ ML services and $L$ label classes.
Let a tuple $s\triangleq (\mathbf{p}^{[1]}, \mathbf{Q}, \mathbf{P}^{[2]})$ represent a calling strategy produced by \systemname{}.
Given an input data $x$,  \systemname{} first calls a \textit{base service}, denoted by $A^{[1]}_s$, which with probability  $\mathbf{p}^{[1]}_i$ is the $i$th service and returns quality score $q_i(x)$ and label $y_i(x)$.
Let $D_s$ be the indicator of whether the quality score is smaller than the threshold value $\mathbf{Q}_{i, y_i(x)}$.
If $D_s=1$, then  \systemname{} invokes an \textit{add-on service}, denoted by $A^{[2]}_s$,  with probability $\mathbf{P}^{[2]}_{i, y_i(x),j}$ being the $j$th service and producing $y_j(x)$ as the predicted label $\hat{y}^{s}(x)$.
Otherwise, \systemname{} simply returns label $\hat{y}^{s}(x) = y_i(x)$ from the base service.
This process is summarized in Figure \ref{fig:FAME:Framework}. 
Note that the strategy is adaptive: the choice of the add-on API can depend on the predicted label and quality score of the base model.

The set of possible strategies can be parametrized as  $S\triangleq \{(\mathbf{p}^{[1]}, \mathbf{Q}, \mathbf{P}^{[2]})| \mathbf{p}^{[1]} \succcurlyeq \mathbf{0} \in \R^{K}, 
\mathbf{1}^T\mathbf{p}^{[1]} =1,
\mathbf{Q} \in \R^{K\times L}, \mathbf{0} \preccurlyeq \mathbf{Q} \preccurlyeq \mathbf{1}, \mathbf{P}^{[2]} \in \R^{K\times L \times K},  \mathbf{P}^{[2]} \succcurlyeq \mathbf{0}, \mathbf{1}^T \mathbf{P}^{[2]}_{k,\ell,\cdot}  = 1    \}$. Our goal is to choose the optimal strategy $s^*$ that maximizes the expected accuracy while satisfies the user's budget constraint $b$.
This is formally stated as below.

\begin{definition}\label{def:FAME:prob}
Given a user budget $b$, the optimal \systemname{} strategy $s^* = (\mathbf{p}^{[1]*}, \mathbf{Q}^*, \mathbf{P}^{[2]*})$ is 
\begin{equation}\label{prob:FAME:optimaldefinition}
    s^* \triangleq \arg \max_{s\in S} \Exp[{r}^s(x)] \textit{  s.t. } \Exp[\eta^{[s]}(x,\mathbf{c})] \leq b, 
\end{equation}
where $r^s(x)\triangleq \mathbbm{1}_{\hat{y}^{s}(x)=y(x)}$ is the reward  and $\eta^{[s]}(x,\mathbf{c})$  the total cost of strategy $s$ on $x$. 
\end{definition}
\begin{remark}
The above definition can be generalized to wider settings.
For example, instead of 0-1 loss, the reward  can be negative square loss to handle regression tasks.
We pick the concrete form for demonstration purposes.
The cost of strategy $s$, $\eta^{[s]}(x,\mathbf{c})$, is the sum of all services called on $x$. For example, if service $1$ and $2$ are called for predicting $x$, then $\eta^{[s]}(x,\mathbf{c})$ becomes $\mathbf{c}_1+\mathbf{c}_2$.
\end{remark}

Given the above formulation, a natural question is how to solve it efficiently.
In the following, 
We first highlight an interesting property of the optimal strategy, \textit{sparsity}, 
which inspires the design of the efficient solver, and then present the algorithm for the solver.

\subsection{Sparsity Structure in the Optimal Strategy}
We show that if problem \ref{prob:FAME:optimaldefinition} is feasible and has unique optimal solution, then we must have $\|\mathbf{p}^{[1]*}\|\leq 2$.
In other words, the optimal strategy should only choose the base service from at most two services (instead of $K$) in the MLaaS market. This is formally stated in Lemma \ref{lemma:FAME:sparsitymainpaper}.
\begin{lemma}\label{lemma:FAME:sparsitymainpaper}
If problem \ref{prob:FAME:optimaldefinition} is feasible, 
then there exists one optimal solution $s^* = (\mathbf{p}^{[1]*}, \mathbf{Q}^*, \mathbf{P}^{[2]*})$ such that $\|\mathbf{p}^{[1]*}\|\leq 2$.
\end{lemma}
To see this, let us first expand $\Exp[r^s(x)]$ and $\Exp[\eta^s(x)]$ by  the law of total expectation.
\begin{lemma}\label{lemma:FAME:expectedacccostmainpaper}
The expected accuracy is $\Exp[r^s(x)]=\sum_{i=1}^{K} \Pr[A^{[1]}_s=i] \Pr[D_s=0|A^{[1]}_s=i] \Exp[r^i(x)|D_s=0,A^{[1]}_s=i]        + \sum_{i,j=1}^{K} \Pr[A^{[1]}_s=i] \Pr[D_s=1|A^{[1]}_s=i]
      \Pr[A_s^{[2]}=j|D_s=1,A_s^{[1]}=i]\Exp[r^j(x)|D_s=1,A_s^{[1]}=i)]$.
The expected cost is
$\Exp[\eta^s(x)]=\sum_{i=1}^{K} \Pr[A_s^{[1]}=i] \Pr[D_s=0|A_s^{[1]}=i] \mathbf{c}_i    + \sum_{i,j=1}^{K} \Pr[A_s^{[1]}=i] \Pr[D_s=1|A_s^{[1]}=i]
      \Pr[A_s^{[2]}=j|D_s=1,A_s^{[1]}=i]\left(\mathbf{c}_i+\mathbf{c}_j\right)$.
      
\end{lemma}
Note that both  $\Exp[r^s(x)]$ and $\Exp[\eta^s(x)]$ are linear in $\Pr[A_s^{[1]}=i]$, which by definition equals $\mathbf{p}^{[1]}_i$.
Thus, fixing $\mathbf{Q}$ and $\mathbf{P}^{[2]}$,  problem \ref{prob:FAME:optimaldefinition} becomes a  linear programming in $\mathbf{p}^{[1]}$.
Intuitively, the corner points of its feasible region must be 2-sparse, since except $\Exp[\eta^s(x)]\leq b$ and $\mathbf{1}^T\mathbf{p}^{[1]}\leq 1$ , all other constraints ($ \mathbf{p}^{[1]}\succcurlyeq \mathbf{0} $) force sparsity.
As the optimal solution of a linear programming should be a corner point, $\mathbf{p}^{[2]}$ must also be 2-sparse.

This sparsity structure helps reduce the computational complexity for solving problem \ref{prob:FAME:optimaldefinition}. 
In fact, the sparsity structure implies problem \ref{prob:FAME:optimaldefinition} becomes equivalent to a \textit{master problem} 
\begin{equation}\label{prob:FAME:master}
    \max_{(i_1, i_2, p_1, p_2, b_1, b_2) \in \mathit{C} }\textit{ }  p_1 g_{i_1}(b_1/p_1) + p_2 g_{i_2}(b_2/p_2)  \textit{ } s.t. b_1+b_2\leq b
\end{equation}
where $\mathit{c}=\{(i_1,i_2,p_1,p_2,b_1,b_2)|i_1,i_2\in[K],p_1,p_2\geq0, p_1+p_2=1, b_1,b_2\geq 0\}$, and $g_i(b')$ is the optimal value of the \textit{subproblem}
\begin{equation}\label{prob:FAME:subproblem1}
    \begin{split}
       & \max_{\mathbf{Q},\mathbf{P}^{[2]}:s=(\mathbf{e}_i,\mathbf{Q},\mathbf{P}^{[2]})\in S}  \Exp[r^s(x) \textit{ } s.t. \textit{ }   \Exp[\eta^s(x)] \leq b'\\ 
    \end{split}
\end{equation}
Here, the master problem decides which two services ($i_1,i_2$) can be the base service, how often ($p_1, p_2$) they should be invoked, and how large budgets ($b_1, b_2$) are assigned, 
while for a fixed base service $i$ and budget $b'$,  the subproblem  maximizes the expected reward.

\subsection{A Practical Algorithm}\label{sec:FAME:thm:algorithm}
Now we are ready to give the sparsity-inspired algorithm for generating an approximately optimal strategy $\hat{s}$, summarized in Algorithm \ref{Alg:FAME:TrainingAlgorithm}. 
\begin{algorithm}
\caption{\systemname{} Strategy Training.}
	\label{Alg:FAME:TrainingAlgorithm}
	\SetKwInOut{Input}{Input}
	\SetKwInOut{Output}{Output}
	\Input{$K, M, \mathbf{c}, b$,  $\{y(x_i), \{q_k(x_i), y_k(x_i)\}_{k=1}^{K}\}_{i=1}^N$}
	\Output{\systemname{} strategy tuple  $\hat{s} = \left(\hat{\mathbf{p}}^{[1]}, \hat{\mathbf{Q}}, \hat{\mathbf{P}}^{[2]}  \right)$}  
  \begin{algorithmic}[1]
  	\STATE Estimate  $\Exp[r_i(x)|D_s,A_s^{[1]}]$ from the training data $\{y(x_i), \{q_k(x_i), y_k(x_i)\}_{k=1}^{K}\}_{i=1}^N$
 
\STATE 	For $i\in[K]$, $b'_m\in[0,\frac{\|\mathbf{2c}\|_\infty}{M},\cdots,  \|2\mathbf{c}\|_\infty]$, solve  problem \ref{prob:FAME:subproblem1} to find optimal value $g_i(b'_m)$ 

\STATE	For $i\in[K]$, construct function $g_i(\cdot)$ by linear interpolation on $b'_0, b'_1,\cdots, b'_M$.
	
\STATE	Solve problem \ref{prob:FAME:master} to find optimal solution $i_1^*, i_2^*, p_1^*, p_2^*, b_1^*, b_2^*$
 
\STATE    For $t\in [2]$, let $i=i_t^*, b' = b_t^*/p_t^*$, solve problem \ref{prob:FAME:subproblem1} to find the optimal solution $\mathbf{Q}^{}_{[i_t^*]}, \mathbf{P}^{[2]}_{[i_t^*]}$
    
\STATE  $\hat{\mathbf{p}}^{[1]} = p_1^* \mathbf{e}_{i_1^*}   + p_2^* \mathbf{e}_{i_2^*}, \hat{\mathbf{Q}} = \mathbf{Q}^{}_{[i_1^*]}+\mathbf{Q}^{}_{[i_2^*]}, \hat{\mathbf{P}}^{[2]} = \mathbf{P}^{[2]}_{[i_1^*]}+\mathbf{P}^{[2]}_{[i_2^*]}$
 
  \STATE  Return  $\hat{s} = \left(\hat{\mathbf{p}}^{[1]}, \hat{\mathbf{Q}}, \hat{\mathbf{P}}^{[2]}  \right)$
  \end{algorithmic}
\end{algorithm}

Algorithm \ref{Alg:FAME:TrainingAlgorithm} consists of three main steps. First,  the conditional accuracy $\Exp[r_i(x)|D_s,A_s^{[i]}]$ is estimated from the training data (line 1).
Next (line 2 to line 4), we find the optimal solution $i_1^*, i_2^*, p_1^*, p_2^*, b_1^*, b_2^*$ to problem \ref{prob:FAME:master}.
To do so, we first evaluate $g_i(b')$ for $M+1$ different budget values  (line 2), and then construct the functions   $g_i(\cdot)$ via linear interpolation (line 3) while enforce $g_i(b')=0, \forall b' \leq \mathbf{c}_i$.
Given (piece-wise linear) $g_i(\cdot)$, problem \ref{prob:FAME:master} can be solved by enumerating a few linear programming  (line 4).
Finally, the algorithm seeks to find the optimal solution in the original domain of the strategy, by solving subproblem \ref{prob:FAME:subproblem1} for base service being $i_1^*$ and $i_2^*$ separately (line 5), and then align those solutions appropriately (line 6).  
We leave the details of solving subproblem \ref{prob:FAME:subproblem1} to the supplement material due to space constraint. 
Theorem \ref{thm:FAME:mainbound} provides the performance analysis of Algorithm \ref{Alg:FAME:TrainingAlgorithm}.
\begin{theorem}\label{thm:FAME:mainbound}
Suppose $\Exp[r_i(x)|D_s,A_s^{[1]}]$ is Lipschitz continuous with constant $\gamma$ w.r.t. each element in $\mathbf{Q}$.
Given $N$ i.i.d. samples  $\{y(x_i), \{(y_k(x_i), q_k(x_i))\}_{k=1}^{K}\}_{i=1}^{N}$, the computational cost of Algorithm \ref{Alg:FAME:TrainingAlgorithm} is $O\left(NMK^2+K^3M^3L+M^LK^2\right)$. 
With probability $1-\epsilon$, the produced strategy $\hat{s}$ satisfies  
 $\Exp[r^{\hat{s}}(x)]-\Exp[r^{s^*}(x)]\geq  - O\left(\sqrt{\frac{\log \epsilon + \log M +\log K +\log L }{N}} + \frac{\gamma }{M}\right)$, and 
$\Exp[\eta^{[\hat{s}]}(x,\mathbf{c})]  \leq b$.
\end{theorem}

As Theorem \ref{thm:FAME:mainbound} suggests, the parameter $M$ is used to balance between computational cost and accuracy drop of $\hat{s}$.

For practical cases where $K$ and $L$ (the number of classes) are around ten and $N$ is more than a few thousands, we have found  
$M=10$ is a good value for good accuracy and small computational cost. 
Note that the coefficient of the $K^L$ terms is small: in experiments, we observe it takes only a few seconds for $L=31, M=40$.  
For datasets with very large number of possible labels, we can always cluster those labels into a few ''supclasses'', or adopt approximation algorithms to reduce $O(M^L)$ to $O(M^2)$ (see details in the supplemental materials). 
In addition, slight modification of $\hat{s}$ can satisfy  \textit{strict budget constraint}: if budgets allows, use $\hat{s}$ to pick APIs; otherwise, switch to the cheapest API.

\section{Experiments}\label{Sec:FAME:Experiment}
We compare the accuracy and incurred costs of \systemname{} to that of real world ML services for various tasks.
Our goal is four-fold: (i) understanding when and why \systemname{} can reduce cost without hurting accuracy, (ii) evaluating the cost savings by \systemname{}, (iii) investigating the trade-offs between accuracy and cost achieved by \systemname{}, and (iv) measuring the effect of training data size on \systemname{}'s performance.

\paragraph{Tasks,  ML services, and Datasets.} 
We focus on three common ML tasks in different application domains: facial emotion recognition (\textit{FER}) in computer vision, sentiment analysis  (\textit{SA}) in natural langauge processing), and speech to text (\textit{STT}) in speech recognition. 
The ML services used for each task as well as their prices are summarized in Table \ref{tab:MLservice}.
For each task we also found a small open source model from GitHub, which is much less expensive to execute per data point than the commercial APIs.
{{
\begin{table}[t]
	\centering
	\small
	\caption{\small ML services used for each task. Price unit: USD/10,000 queries. A publicly available (and thus free) GitHub model is also used per task: a  convolutional neural network (CNN) \cite{FER_github} pretrained on FER2013 \cite{Dataset_FER2013} for \textit{FER} , a rule based tool (Bixin \cite{SA_Chinese_github} for Chinese and Vader \cite{SA_English_github, VanderICWSM2014} for English ) for  \textit{SA}, and a recurrent neural network (DeepSpeech)
	\cite{STT_Deepspeech_github,DeepSpeech_ICML16} pretrained  on Librispeech \cite{panayotov2015librispeech} for \textit{STT}. } 
	\begin{tabular}{|c||c|c|c|c|c|c|}
		\hline
		Tasks  & ML service & Price & ML service & Price & ML service & Price \bigstrut \\
		\hline
		\hline
		\textit{FER} & Google Vision \cite{GoogleAPI} & 15    & MS  Face \cite{MicrosoftAPI}& 10    & Face++ \cite{FacePPAPI}& 5 \bigstrut\\
		\hline
		\textit{SA} & Google NLP \cite{GoogleNLPAPI}& 2.5     & AMZN Comp \cite{AmazonAPI} & 0.75& Baidu NLP \cite{BaiduAPI} & 3.5  \bigstrut\\
		\hline
		\textit{STT} & Google Speech \cite{GoogleSpeechAPI}& 60    & MS Speech  \cite{MicrosoftSpeechAPI}& 41    & IBM Speech \cite{IBMAPI}& 25 \bigstrut\\
		\hline
	\end{tabular}%
	\label{tab:MLservice}%
\end{table}%
}
}
Table \ref{tab:DatasetStats} lists the statistics for all the datasets used for different tasks.
More details  can be found in the appendix.
\begin{table}[t]
  \centering
  \small
  \caption{\small Datasets sample size and number of classes.}
    \begin{tabular}{|c||c|c|c||c|c||c|}
    \hline
    Dataset & Size & \# Classes & Dataset & Size & \# Classes & Tasks \bigstrut\\
    \hline
     \hline
    FER+ \cite{dataset_FERP_BarsoumZCZ16}  & 6358  & 7     & RAFDB \cite{Dataset_FAFDB_li2017reliable} & 15339 & 7     & \multirow{2}[4]{*}{\textit{FER}} \bigstrut\\
\cline{1-6}    EXPW \cite{Dataset_EXPW_SOCIALRELATION_ICCV2015}  & 31510 & 7     & AFFECTNET \cite{Dataset_AFFECTNET_MollahosseiniHM19}& 287401 & 7     & \bigstrut \\
    \hline
    YELP \cite{Dataset_SEntiment_YELP} & 20000  & 2     & SHOP \cite{Dataset_SENTIMENT_SHOP}& 62774 & 2     & \multirow{2}[4]{*}{\textit{SA}} \bigstrut \\
\cline{1-6}    IMDB \cite{Dataset_SEntiment_IMDB_ACL_HLT2011} & 25000 & 2     & WAIMAI \cite{Dataset_SENTIMENT_WAIMAI}& 11987 & 2     &  \bigstrut\\
    \hline
    DIGIT \cite{Dataset_Speech_DIGIT} & 2000  & 10    & AUDIOMNIST \cite{Dataset_Speech_AudioMNIST_becker2018interpreting} & 30000 & 10      & \multirow{2}[4]{*}{\textit{STT}} \bigstrut \\
\cline{1-6}  FLUENT \cite{Dataset_Speech_Fluent_LugoschRITB19} & 30043  & 31   &   COMMAND \cite{Dataset_Speech_GoogleCommand} & 64727 & 31   & \bigstrut\\
    \hline
    \end{tabular}%
  \label{tab:DatasetStats}%
\end{table}%
\begin{figure}[htbp]
	\centering
	\includegraphics[width=0.78\linewidth]{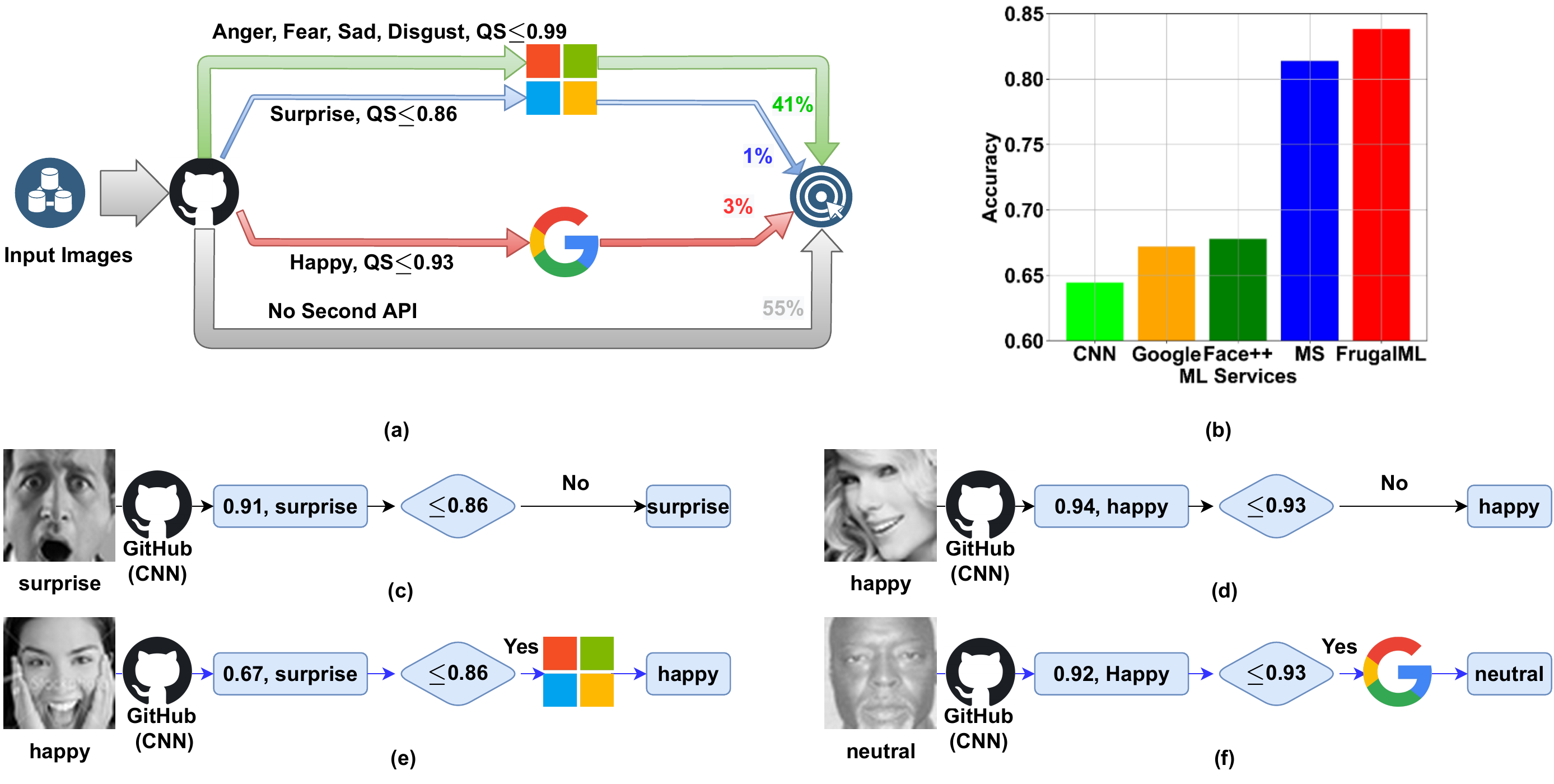}
	\caption{A \systemname{} strategy learned  on the dataset FER+. (a): data flow.
	(b):  accuracy of all ML services and \systemname{} which matches the cost of the cheapest API (FACE++). (c-f): \systemname{} prediction process on a few testing data.  
	As shown in (a), on most data (55\%),  calling the cheap open source CNN from GitHub is  sufficient.
	Thus, \systemname{} incurs  <50\% cost than the most accurate API (Microsoft).
    Note that unique quality score thresholds for different labels predicted by the base service are learned: e.g.,, given label, ''surprise``, 0.86 is used to determine whether (e) or not (c) to call Microsoft, while for label ``happy'', the learned threshold is 0.93 ((d) and (f)).  
    Such unique thresholds are critical for both accuracy improvement and cost reduction: universally using 0.86 leads to misclassification on (f), while globally adopting 0.93 creates extra cost by called unnecessary add-on service on (c).
	}
	\label{fig:FAME:casestudyflow}
	\end{figure}
	
	\paragraph{Facial Emotion Recognition: A Case Study.}
Let us start with facial emotion recognition on the FER+ dataset.
We set budget $b=5$, the price of FACE++, the cheapest API (except the open source CNN model from GitHub) and obtain a \systemname{} strategy by training on half of FER+. 
Figure \ref{fig:FAME:casestudyflow} demonstrates the learned  \systemname{} strategy.
Interestingly, as shown in Figure \ref{fig:FAME:casestudyflow}(b), \systemname{}'s accuracy is higher than that of the best ML service (Microsoft Face), while its cost is much lower. 
This is because base service's quality score, utilized by \systemname{}, is a better signal than raw image to identify if its prediction is trustworthy. 
Furthermore, the quality score threshold, produced by \systemname{} also depends on label predicted by the base service.
This flexibility helps to increase accuracy as well as to reduce costs.
For example, using a universal threshold $0.86$ leads to misclassfication on Figure \ref{fig:FAME:casestudyflow}(f),
while $0.93$ causes unnecessary add-on service call on Figure \ref{fig:FAME:casestudyflow} (c).

For comparison, we also train a mixture of experts strategy with a softmax gating network and the majority voting ensemble method. 
The learned mixture of experts always uses Microsoft API, leading to the same accuracy (81\%) and same cost (\$10).
The accuracy of majority voting on the test data is slightly better at 82\%, but substantially worse than the performance of \systemname{} using a small budget of $\$5$.
 Majority vote, and other standard ensemble methods, needs to collect the prediction of all services, resulting in a cost (\$30) which is 6 times the cost of \systemname{}. 
Moreover, both mixture of experts and ensemble method require fixed cost, while \systemname{} gives the users flexibility to choose a budget.

\begin{table}[htbp]
    \small
  \centering
  \caption{\small Cost savings achieved by \systemname{} that reaches same accuracy as the best commercial API.}
    \begin{tabular}{|c||c|c|c|c|c||c|c|c|c|}
        \hline
    Dataset & Acc & Price  & Cost& \multicolumn{1}{c|}{Save} & Dataset & Acc & Price  & Cost & \multicolumn{1}{c|}{Save} \bigstrut \\
        \hline
        \hline
    FER+  & 81.4  & 10    & 3.3  & \textbf{67\%}  & RAFDB & 71.7  & 10    & 4.3  & \textbf{57\%} \bigstrut\\
        \hline
	    EXPW  & 72.7  & 10    & 5.0  & \textbf{50\%}  & AFFECTNET & 72.2  & 10    & 4.7  & \textbf{53\%} \bigstrut\\
        \hline
    YELP  & 95.7  & 2.5   &     1.9  & \textbf{24\%}      & SHOP  & 92.1  & 3.5   &    1.9   &  \textbf{46\%}\bigstrut\\
    \hline
    IMDB  & 86.4  & 2.5   &      1.9 &    \textbf{24\%}   & WAIMAI & 88.9  & 3.5   & 1.4    & \textbf{60\%} \bigstrut\\
    \hline
    DIGIT & 82.6  & 41    & 23    & \textbf{44\%}  & COMMAND & 94.6  & 41    & 15    & \textbf{63\%} \bigstrut\\
    \hline
    FLUENT & 97.5  & 41    &   26    &    \textbf{37\%}   & AUDIOMNIST & 98.6  & 41    & 3.9    & \textbf{90\%} \bigstrut\\
    \hline
    \end{tabular}%
  \label{tab:FAME:costsaving}%
\end{table}%
\paragraph{Analysis of Cost Savings.} Next, we evaluate how much cost can be saved by \systemname{} to reach the highest accuracy produced by a single API  on different tasks, to obtain some qualitative sense of \systemname{}.
As shown in Table \ref{tab:FAME:costsaving},  \systemname{} can typically save more than half of the cost.
In fact, the cost savings can be as high as 90\% on the AUDIOMNIST dataset.
This is likely because the base service's quality score is highly correlated to its prediction accuracy, and thus \systemname{} only needs to call expensive services for a few difficult data points.
A relatively small saving is reached for \textit{SA} tasks (e.g., on IMDB). 
This might be that the quality score of the rule based \textit{SA} tool is not highly reliable.
Another possible reason is that \textit{SA} task has only two labels (positive and negative), limiting the power of \systemname{}.

\paragraph{Accuracy and Cost Trade-offs.}
Now we dive deeply into the accuracy and cost trade-offs achieved by \systemname{}, shown in Figure \ref{fig:FAME:AccCostTradeoff}.
Here we also compare with two oblations to \systemname{}, ``Base=GH'', where the base service is forced to be the GitHub model, and ``QS only'', which further forces a universal quality score threshold across all labels.

While using any single ML service incurs a fixed cost, \systemname{} allows users to pick any point in its trade-off curve, offering substantial flexibility.
In addition to cost saving, \systemname{} sometimes can achieve higher accuracy than any ML services it calls.
For example, on FER+ and AFFECTNET, more than 2\% accuracy improvement can be reached with small cost, and on RAFDB, when a large cost is allowed, more than 5\% accuracy improvement is gained.
It is also worthy noting that each component in \systemname{} helps improve the accuracy.
On WAIMAI, for instance, ``Base=GH'' and ''QS only'' lead to significant accuracy drops.
For speech datasets such as COMMAND, the drop is negligible, as there is no significant accuracy difference between different labels (utterance).  
Another interesting observation is that there is no universally ``best'' service for a fixed task.
For \textit{SA} task, Baidu NLP achieves the highest accuracy for WAIMAI and SHOP datasets, but Google NLP has best performance on YELP and IMDB.
Fortunately,  \systemname{} adaptively learns  the optimal strategy. 

\begin{figure}
\centering     
\begin{subfigure}[FER+]{\label{fig:a}\includegraphics[width=0.242\linewidth]{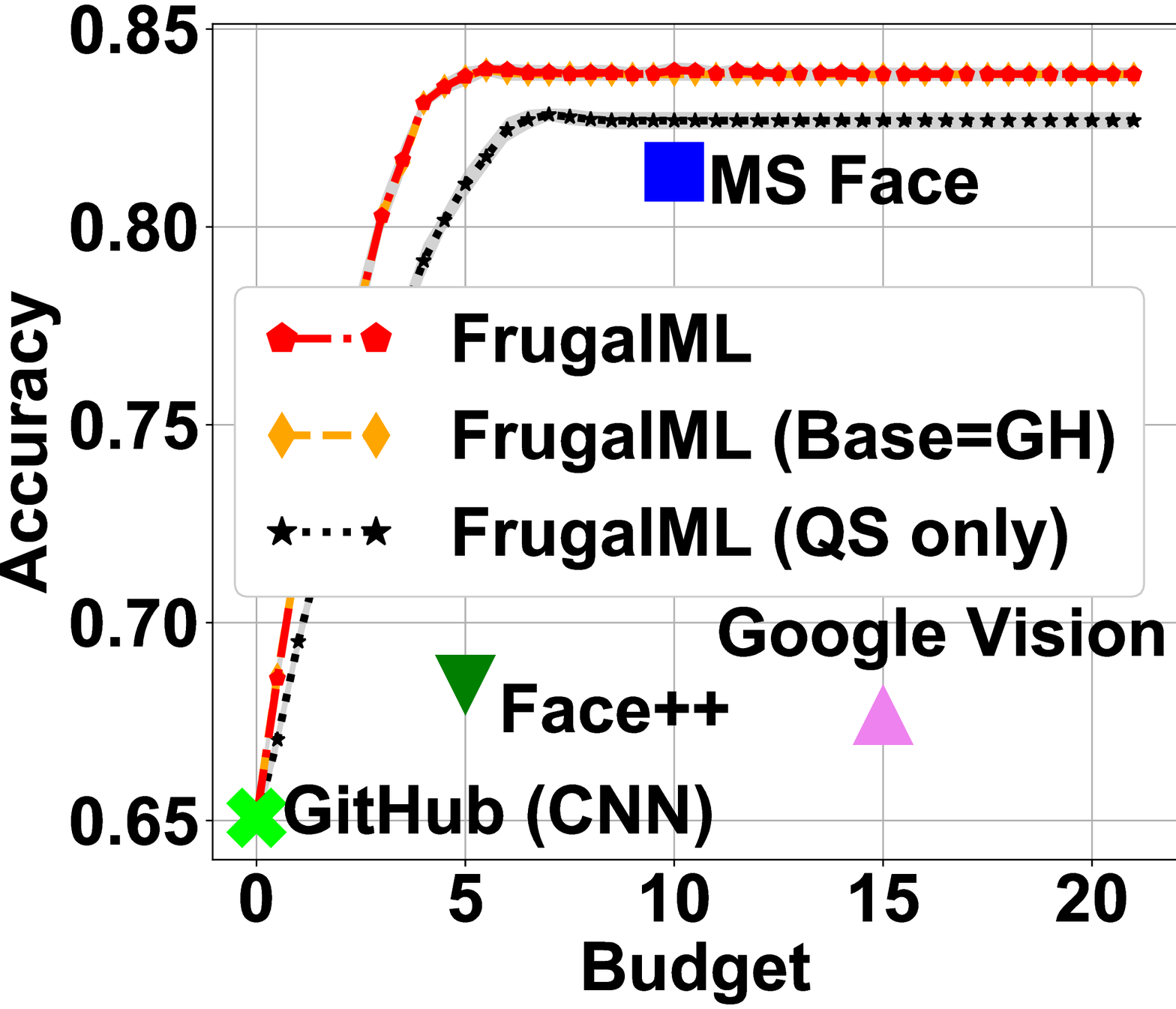}}
\end{subfigure}
\begin{subfigure}[EXPW]{\label{fig:b}\includegraphics[width=0.242\linewidth]{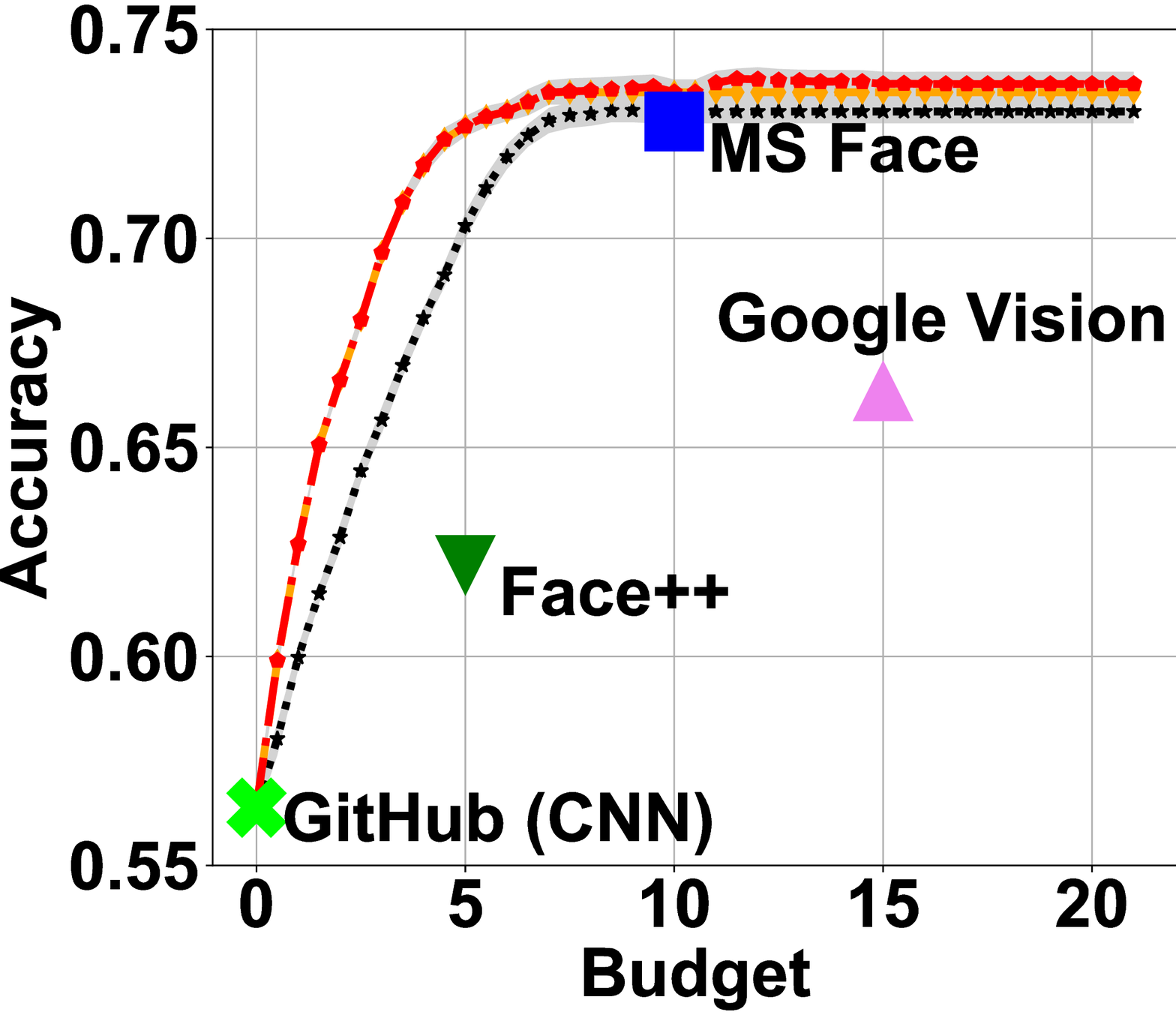}}
\end{subfigure}
\begin{subfigure}[RAFDB]{\label{fig:c}\includegraphics[width=0.242\linewidth]{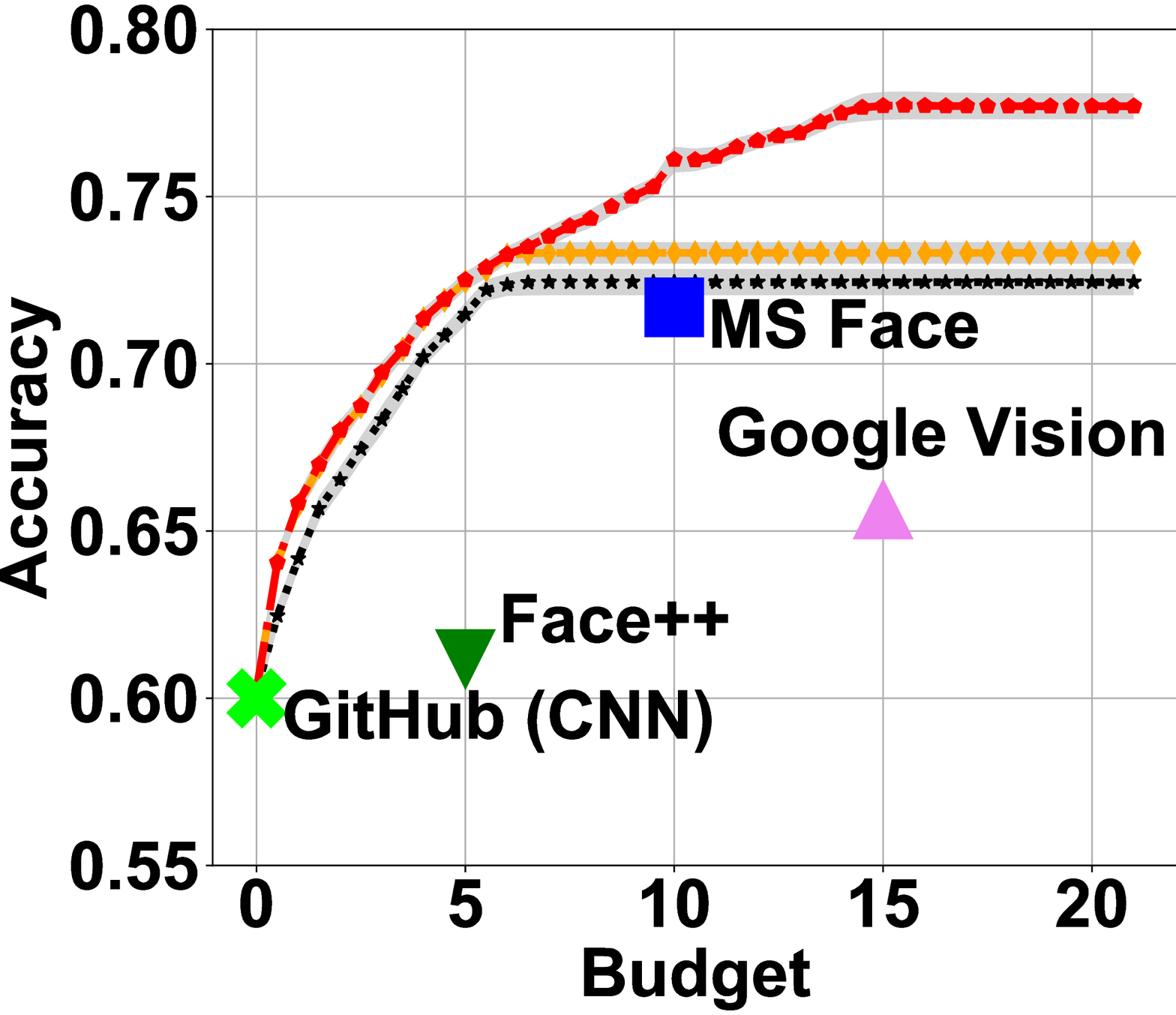}}
\end{subfigure}
\begin{subfigure}[AFFECTNET]{\label{fig:d}\includegraphics[width=0.242\linewidth]{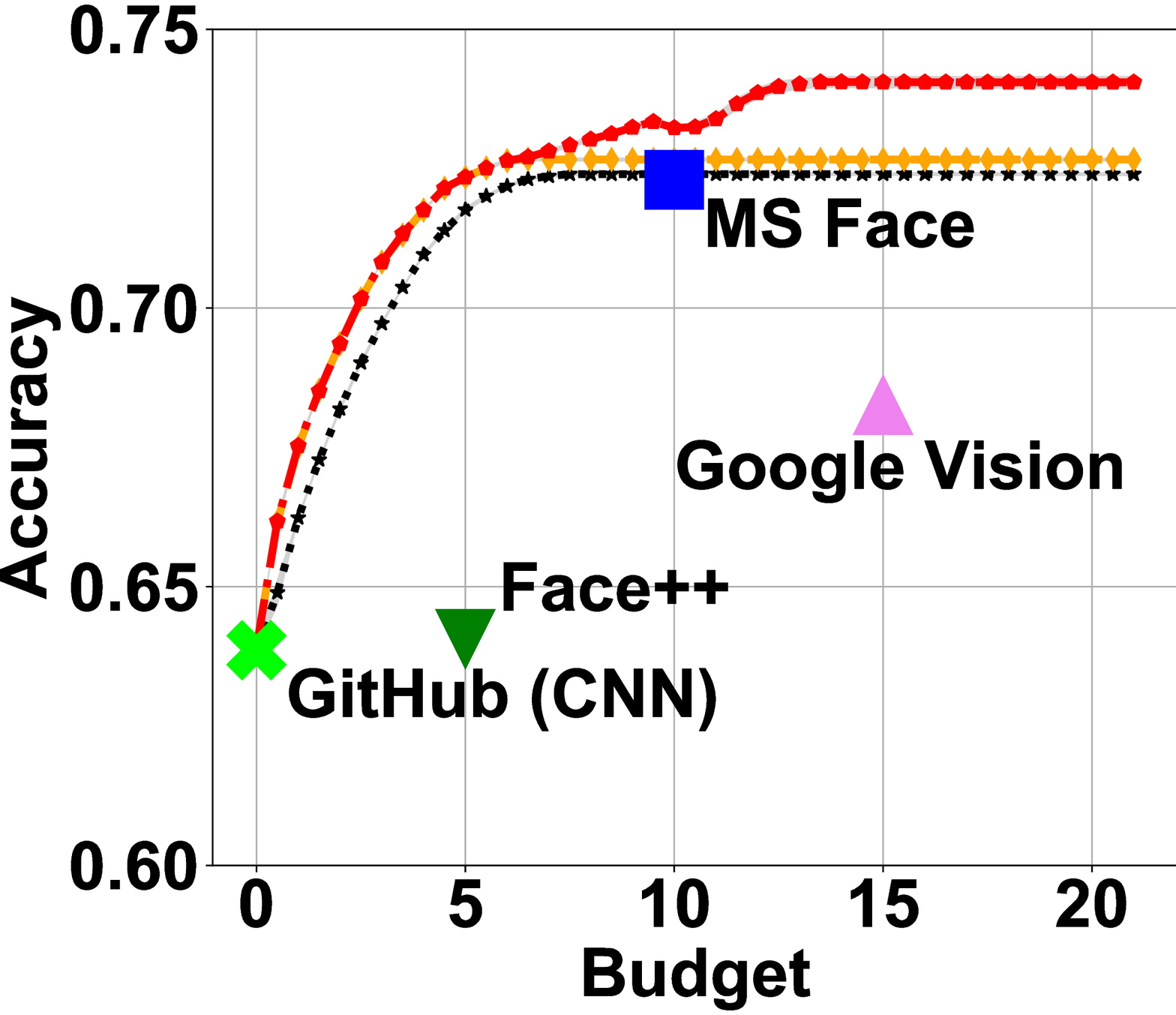}}
\end{subfigure}

\begin{subfigure}[YELP]{\label{fig:e}\includegraphics[width=0.242\linewidth]{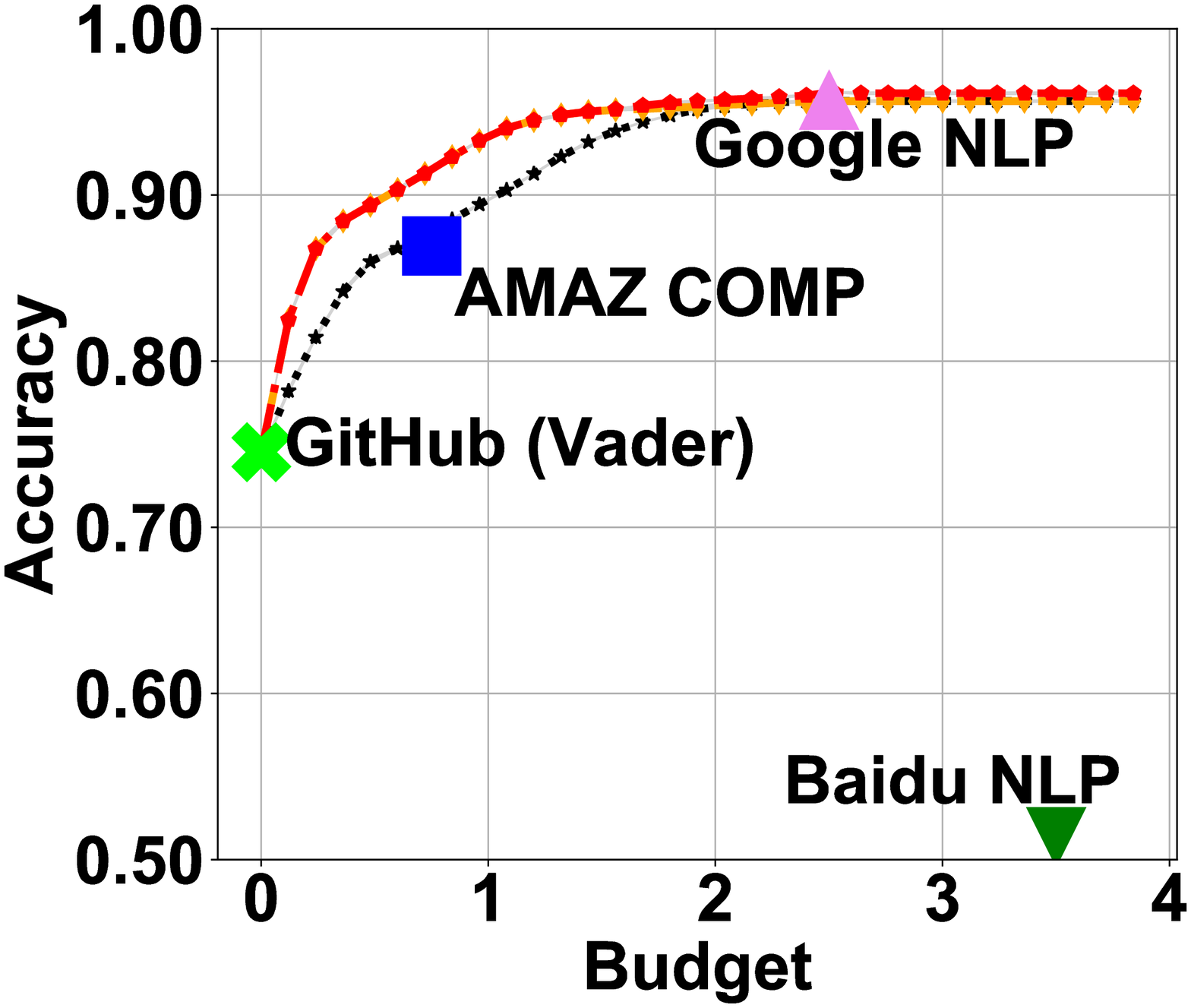}}
\end{subfigure}
\begin{subfigure}[IMDB]{\label{fig:f}\includegraphics[width=0.242\linewidth]{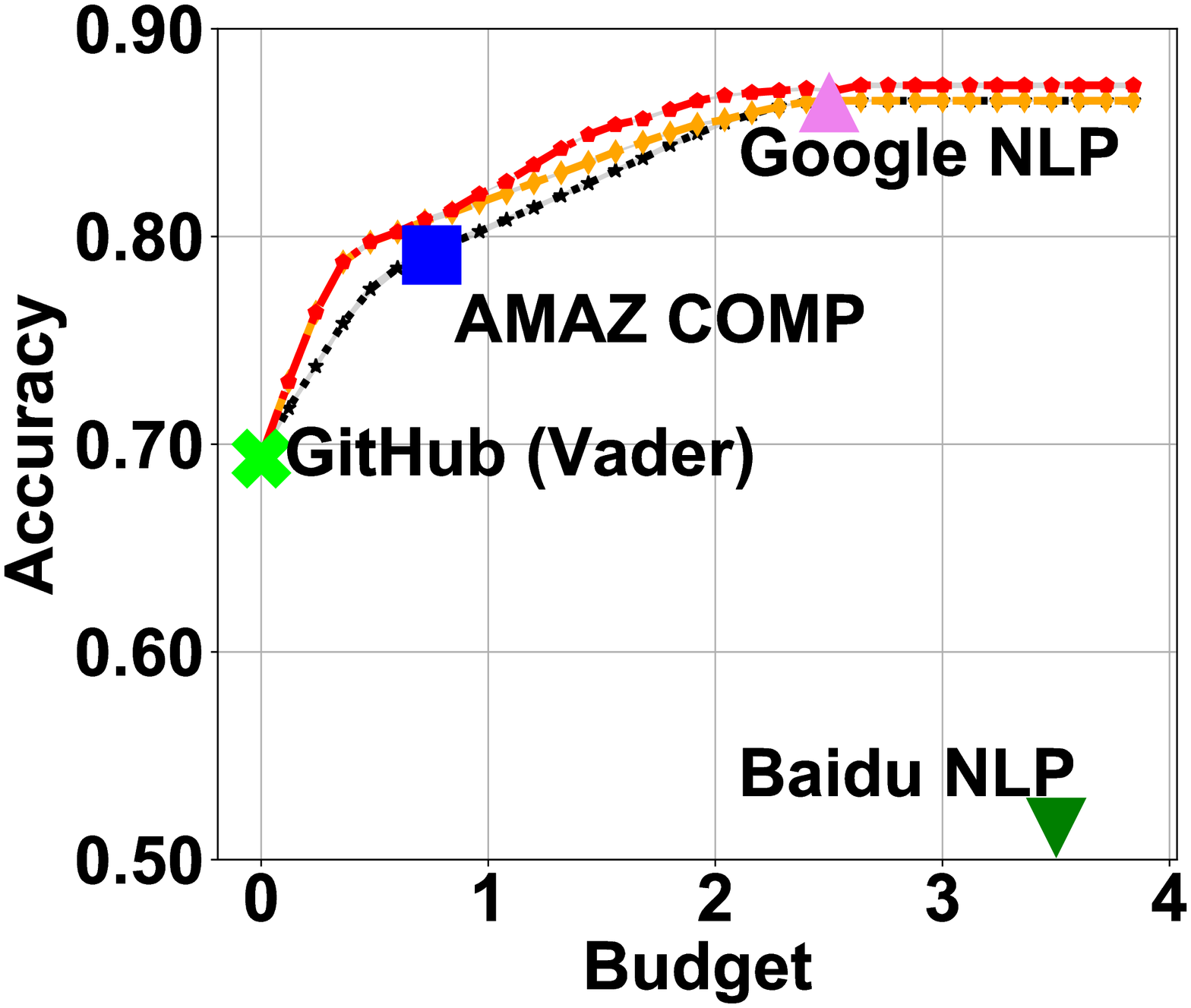}}
\end{subfigure}
\begin{subfigure}[WAIMAI]{\label{fig:g}\includegraphics[width=0.242\linewidth]{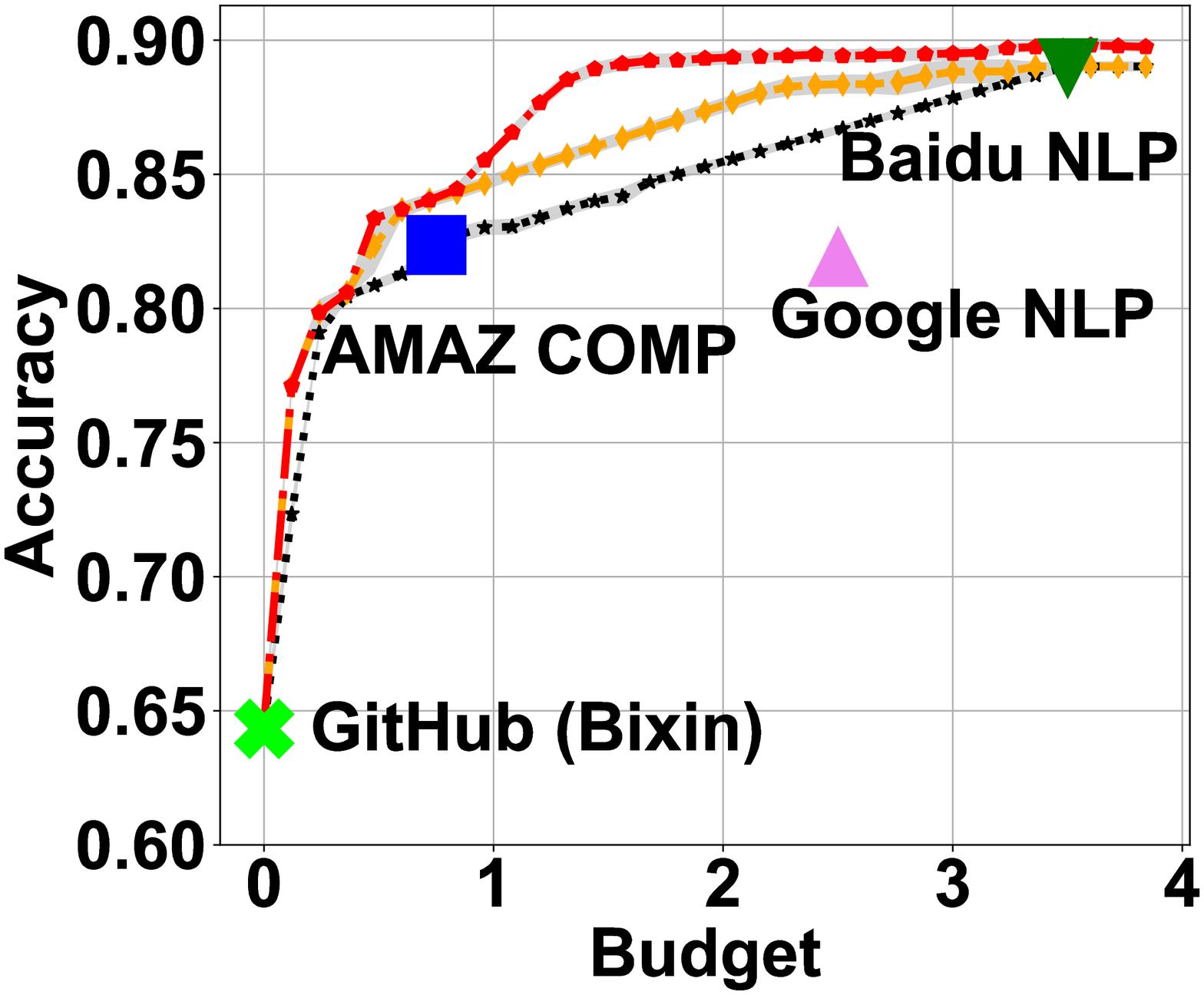}}
\end{subfigure}
\begin{subfigure}[SHOP]{\label{fig:h}\includegraphics[width=0.242\linewidth]{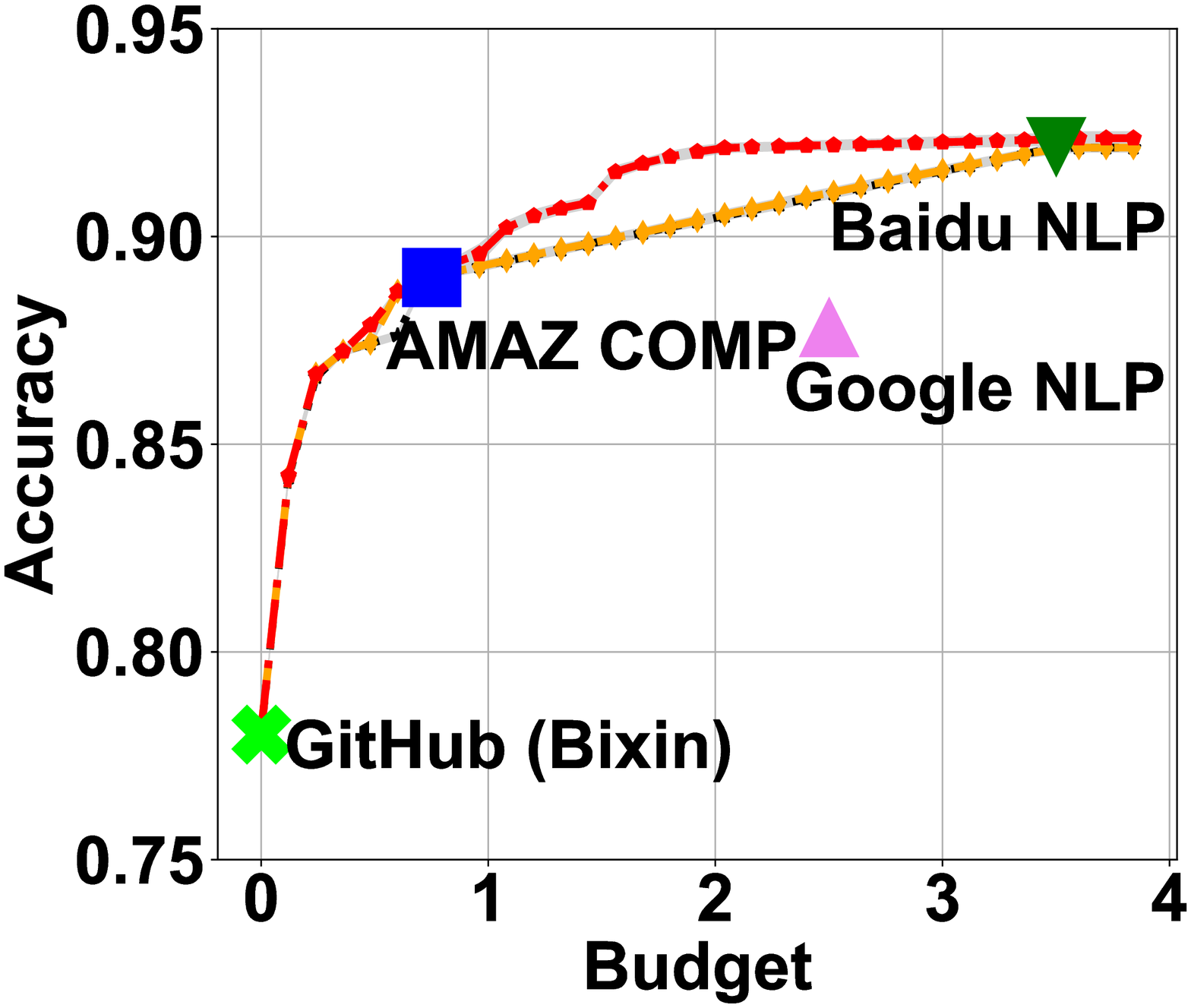}}
\end{subfigure}

\begin{subfigure}[DIGIT]{\label{fig:i}\includegraphics[width=0.242\linewidth]{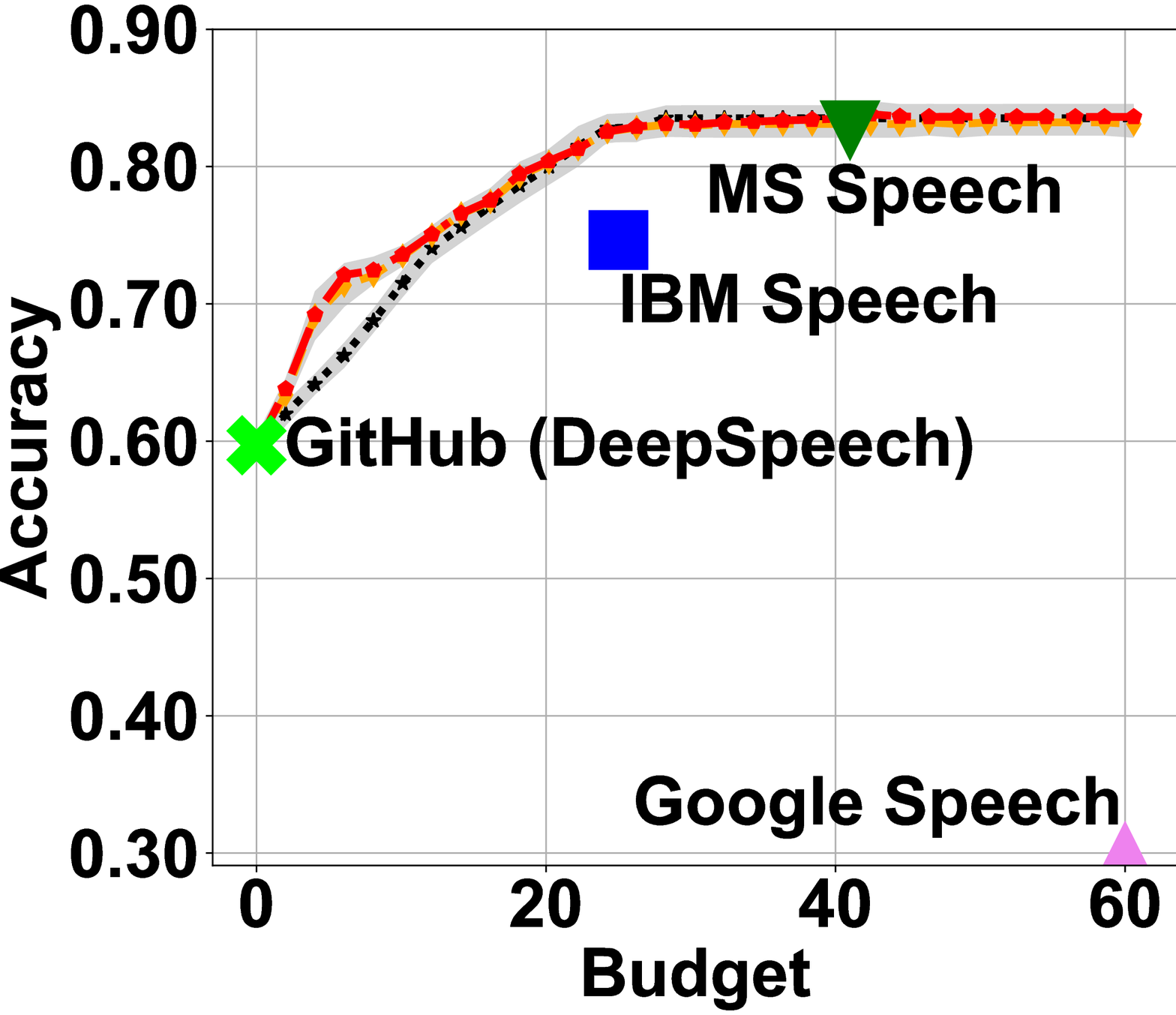}}
\end{subfigure}
\begin{subfigure}[AUDIOMNIST]{\label{fig:j}\includegraphics[width=0.242\linewidth]{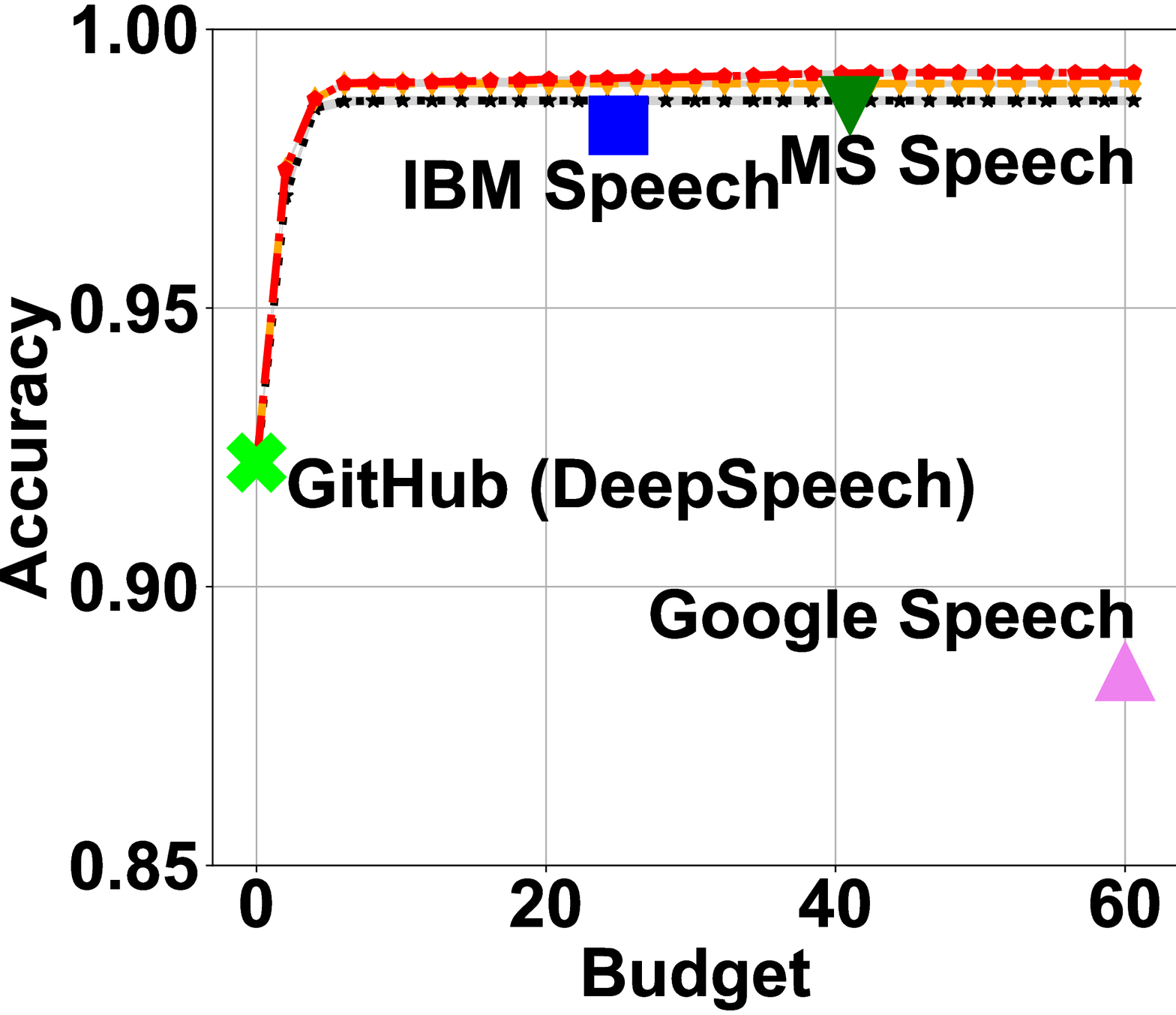}}
\end{subfigure}
\begin{subfigure}[COMMAND]{\label{fig:k}\includegraphics[width=0.242\linewidth]{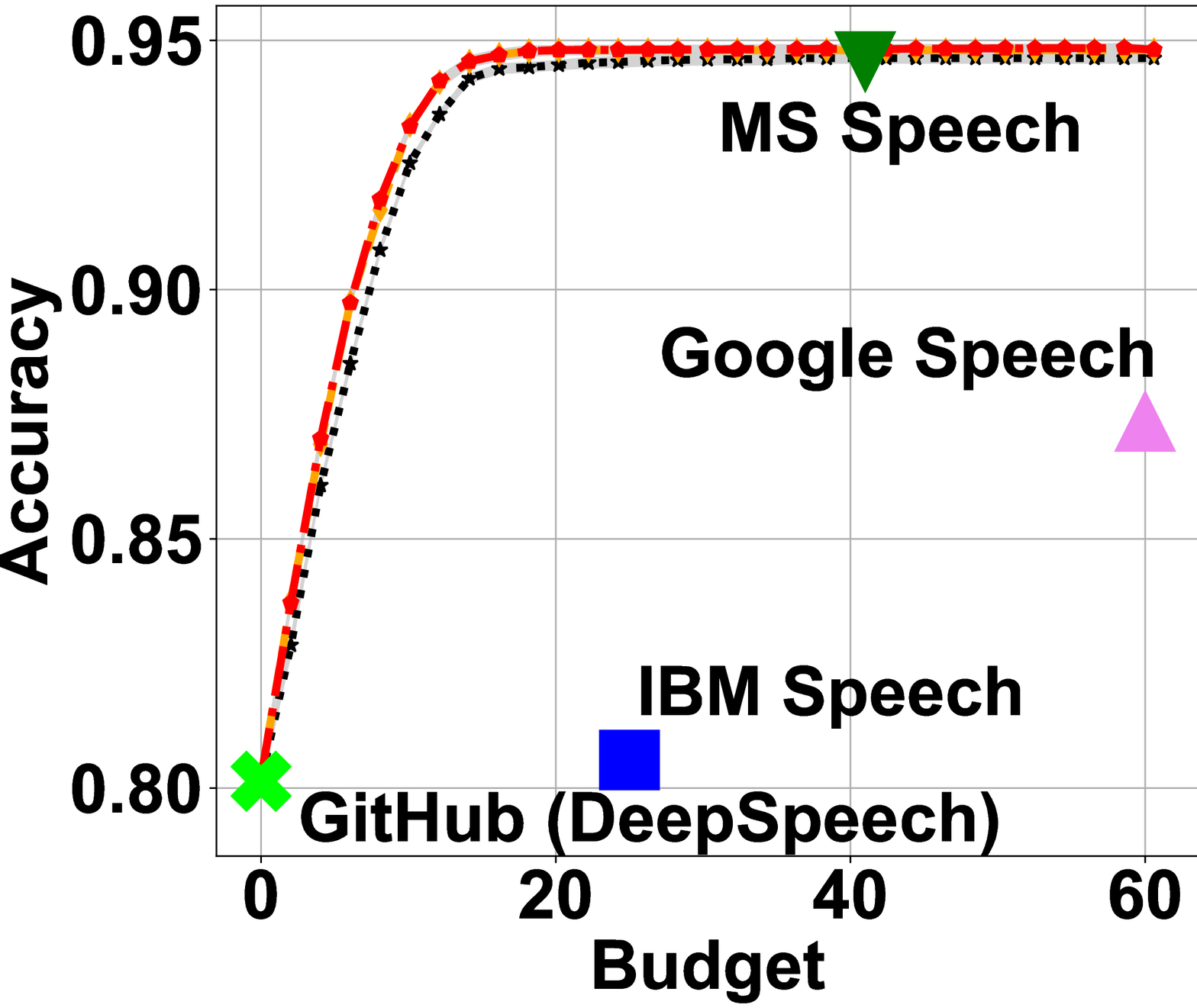}}
\end{subfigure}
\begin{subfigure}[FLUENT]{\label{fig:l}\includegraphics[width=0.242\linewidth]{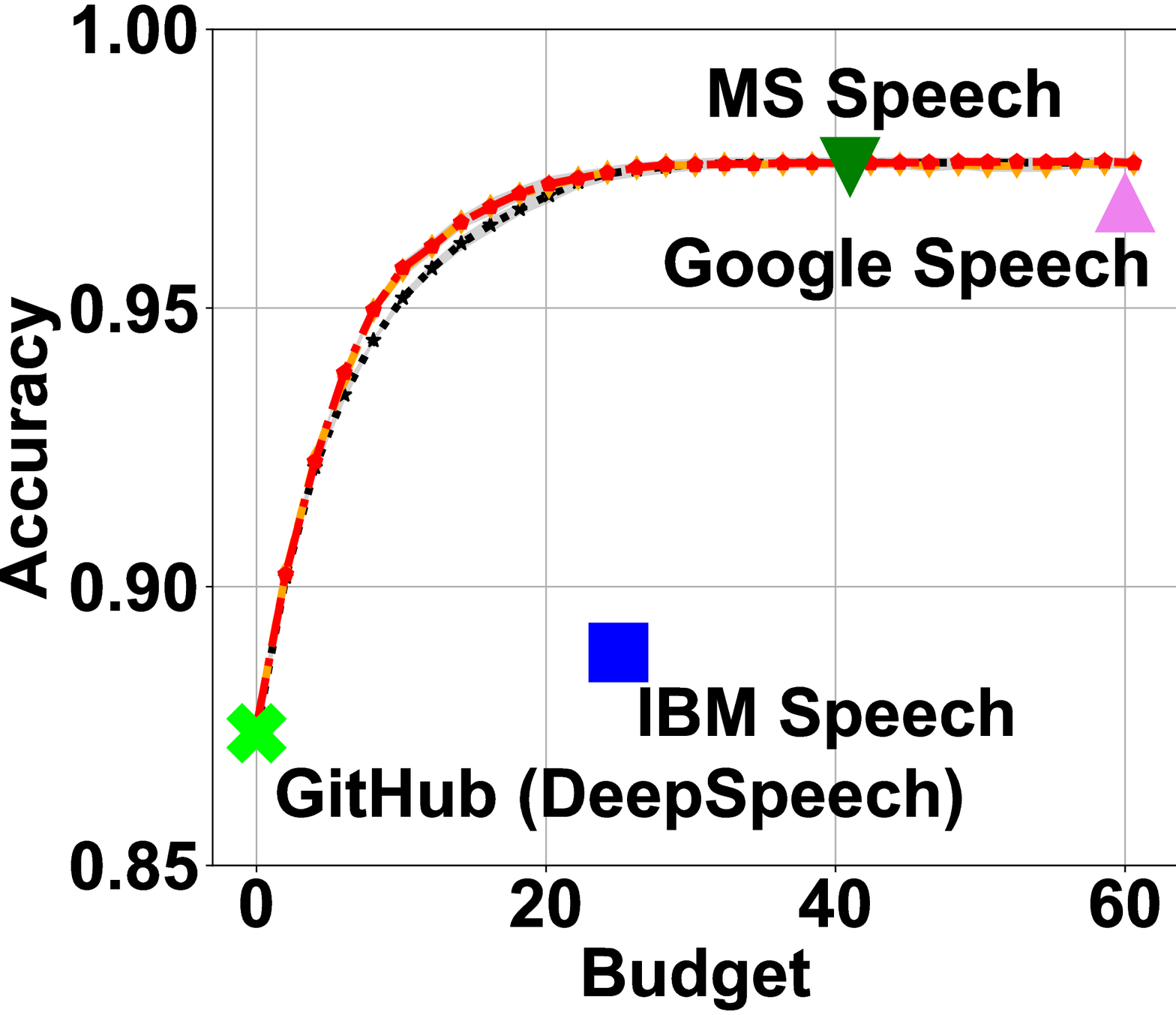}}
\end{subfigure}\caption{Accuracy cost trade-offs.	Base=GH simplifies \systemname{} by fixing  the free GitHub model as base service, and QS only further uses a universal quality score threshold for all labels. The task of row 1, 2, 3 is \textit{FER}, \textit{SA}, and \textit{STT},  respectively.}\label{fig:FAME:AccCostTradeoff}
\end{figure}

\paragraph{Effects of Training Sample Size }
Finally we evaluate how the  training sample size affects \systemname{}'s performance, shown in Figure \ref{fig:FAME:samplesize}.
Overall, while larger number of classes need more samples, we observe 3000 labeled samples are enough across all datasets.  

\begin{figure} \centering
\begin{subfigure}[FER+ (\# label:7)]{\label{fig:sample_a}\includegraphics[width=0.29\linewidth]{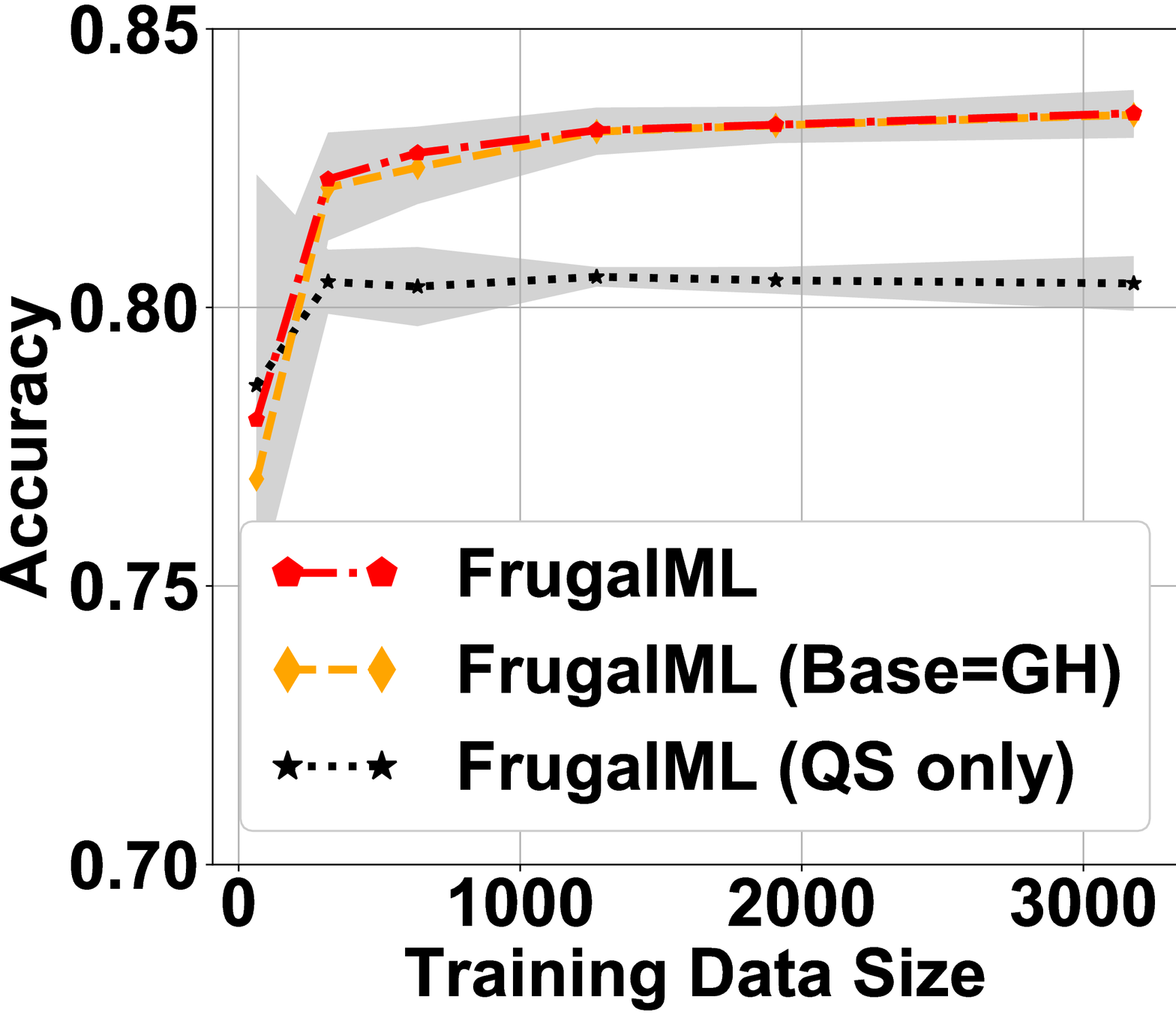}}
\end{subfigure}
\begin{subfigure}[WAIMAI (\# label: 2)]{\label{sample_b}\includegraphics[width=0.29\linewidth]{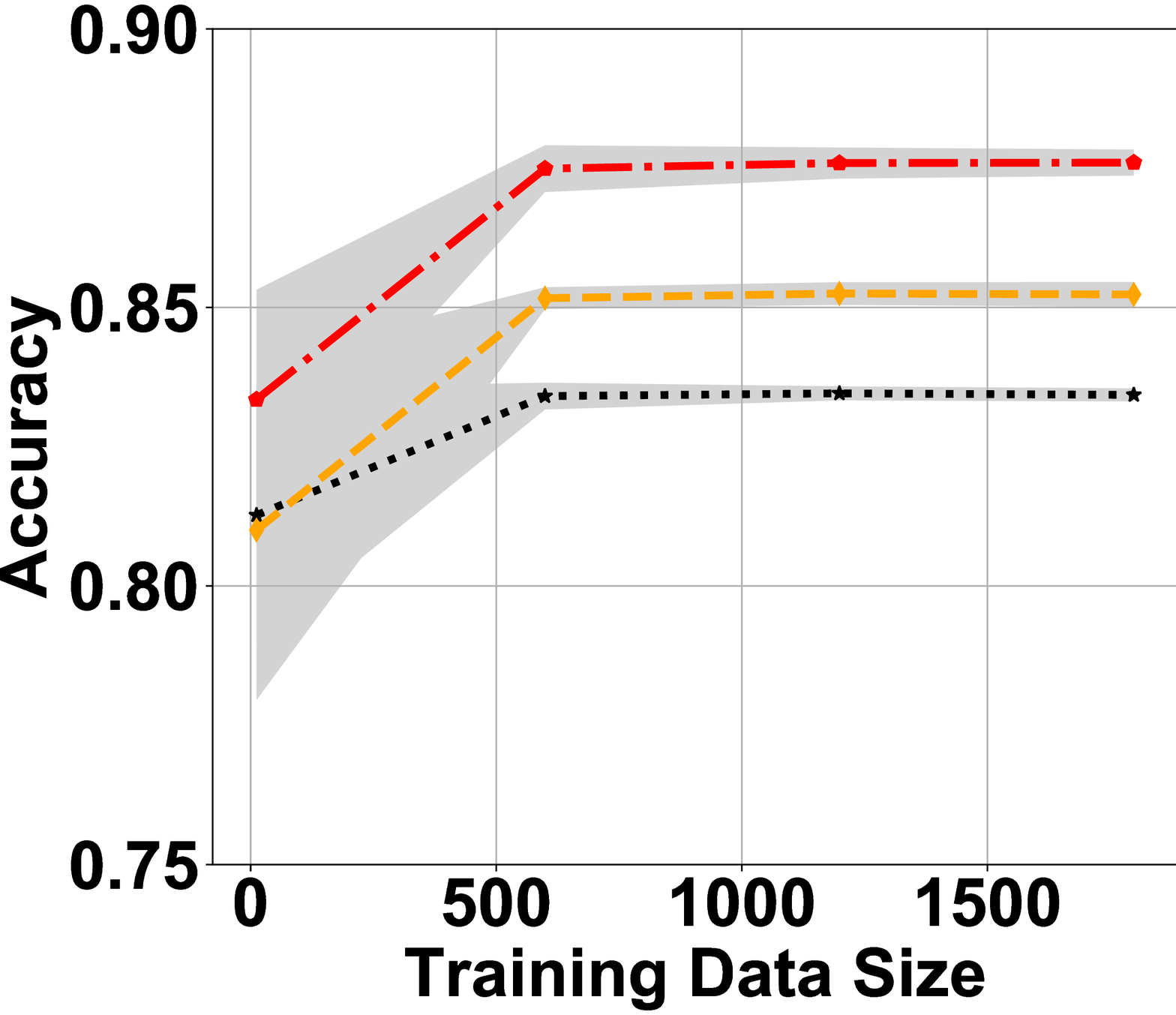}}
\end{subfigure}
\begin{subfigure}[AUDIOMNIST (\# label: 10)]{\label{sample_c}\includegraphics[width=0.29\linewidth]{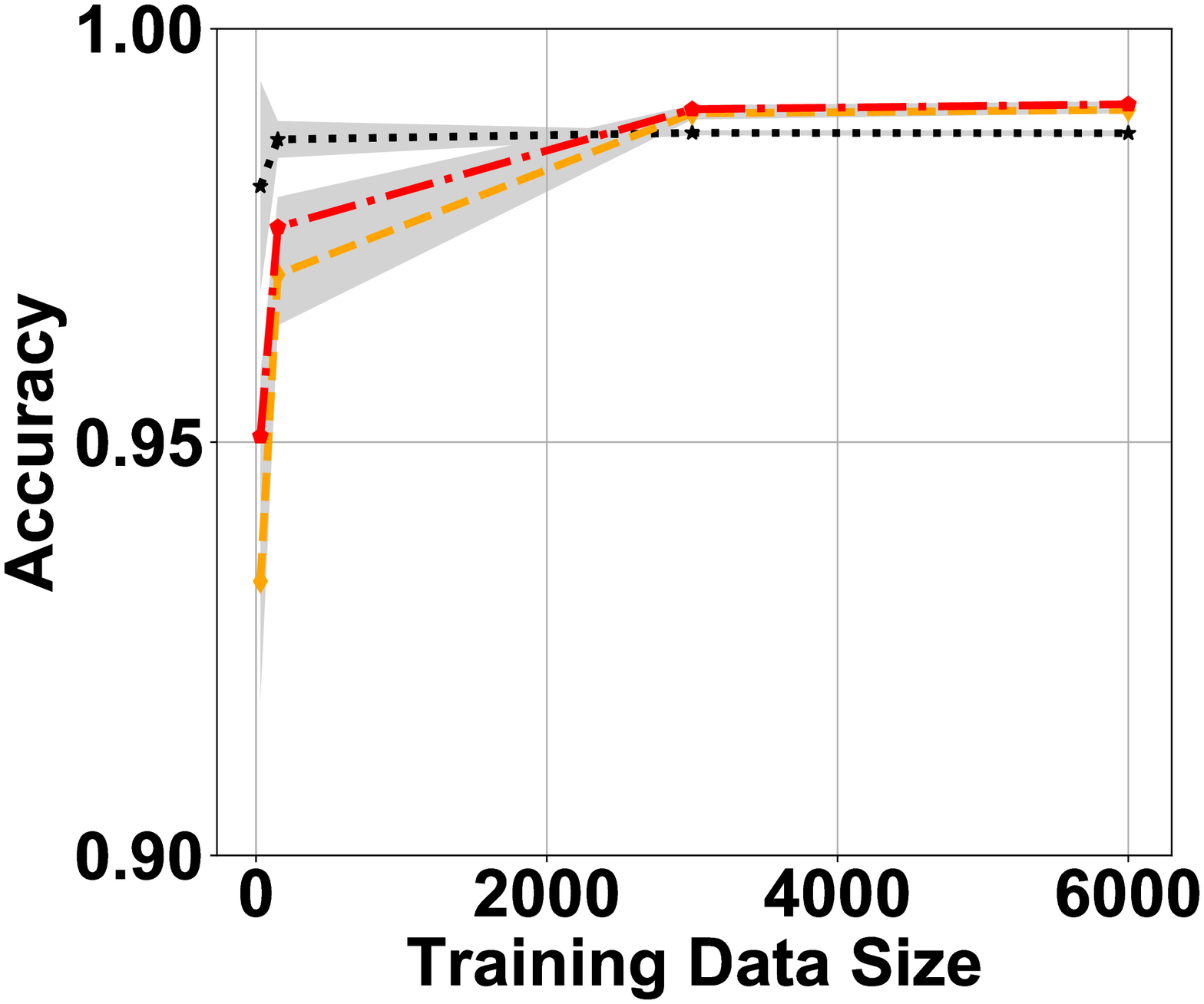}}
\end{subfigure}
	\caption{{Testing accuracy v.s.training data size. The fixed budget is 5, 1.2, 20, separately.}}\label{fig:FAME:samplesize}
\end{figure}
\vspace{-3mm}
\section{Conclusion and Open Problems}\label{Sec:FAME:Conclusion}
\vspace{-2mm}
In this work we proposed \systemname{},  a formal framework for identifying the best strategy to call ML APIs given a user's budget.
Both theoretical analysis and empirical results demonstrate that \systemname{} leads to significant cost reduction and accuracy improvement. 
\systemname{} is also efficient to learn: it typically takes a few minutes on a modern machine. Our research characterized the substantial heterogeneity in cost and performance across available ML APIs, which is useful in its own right and also leveraged by \systemname{}. 
Extending \systemname{} to produce calling strategies for ML tasks beyond classification (e.g., object detection and language translation) is an interesting future direction.
As a resource to stimulate further research in MLaaS, we will also release a dataset used to develop \systemname{}, consisting of 612,139 samples annotated by the APIs, and our code.


\newpage
\appendix

\section{Extra Notations}\label{Sec:FAME:AppendixNotations}
Here we introduce a few more notations.  

We first let $\cdot, \odot,\otimes$ denote inner, element-wise, and Kronecker product, respectively.
Next, Let us introduce a few notations: a matrix $\mathbf{A}\in\R^{K\times L}$, a scalar function $F_{k,\ell}(\cdot): \R \mapsto \R$ for $k \in [K], \ell \in [L]$, a scalar function $\psi_{k_1,k_2,\ell}(\cdot): \R \mapsto \R $ for $k_1, k_2 \in [K], \ell \in [L]$, 
matrix to matrix functions  $\mathbf{r}^{a}(\cdot) : \R^{K\times L} \mapsto \R^{K \times KL}$, $\mathbf{r}^{b}(\cdot): \mathbb{R}^{K\times L} \mapsto \mathbb{R}^{K\times L}$, and  $\mathbf{r}^{[-]}(\cdot) : \R^{K\times L} \mapsto \R^{K \times KL}$.
$\mathbf{A}$ is given by $\mathbf{A}_{k,\ell} \triangleq \Pr[y_k(x) = \ell]$, which represents the probability of $k$th service producing label $\ell$.
The scalar function $F_{k,\ell}(X)\triangleq \Pr[q_k(x) \leq X| y(x) =\ell]$ is the probability of the produced quality score from the $k$th service less than a threshold $X$ conditional on that its predicted label is $\ell$. 
The scalar function $\psi_{k_1,k_2,\ell}(\cdot)$ is defined as $\psi_{k_1,k_2,\ell}(\alpha) \triangleq \Exp\left[  r_{k_1}(x) | y_{k_1}(x) = \ell, q_{k_1}(x) \leq F^{-1}_{k_1,\ell}(\alpha) \right]$, i.e., the executed accuracy of the $k_2$ service conditional on that the $k_1$ services produces a label $\ell$ and quality score that is less than $F_{k,\ell}^{-1}(\pmb{\rho}_{k_1,\ell})$.
Then those matrix to matrix functions are given by ${\mathbf{r}}^{a}_{k_1,K(\ell-1)+k_2}(\pmb \rho) \triangleq \psi_{k_1,k_2,\ell}(\pmb \rho_{k_1,\ell})$, 
${\mathbf{r}}^{b}_{k,\ell}(\pmb \rho) \triangleq \psi_{k,k,\ell}(\pmb\rho_{k,\ell})$, and  
$ {\mathbf{r}}^{[-]}(\pmb \rho) \triangleq {\mathbf{r}}^{a}(\pmb \rho)- \mathbf{r}^{b}(\pmb \rho) \otimes \mathbf{1}^T_K$.

\section{Algorithm Subroutines}\label{Sec:FAME:AppendixAlgorithmDetails}

In this section we provide the details of the subroutines used in the training algorithm for \systemname{}.
There are in total four components: (i) estimating parameters, (ii) solving subproblem \ref{prob:FAME:subproblem1} to obtain its optimal value and solution, (iii) constructing the function $g_i(\cdot)$, and (iv) solving the master problem \ref{prob:FAME:master}.

\paragraph{Estimating Parameters.}
Instead of directly estimating $\Exp[r_i(x)|D_s,A_s^{[i]}]$, we estimate $\mathbf{A}, {\mathbf{r}}^{b}(\mathbf{1}_{K\times L})$, and $ {\mathbf{r}}^{[-]}(\cdot)$ as defined in Section \ref{Sec:FAME:AppendixNotations},  which are sufficient for the subroutines to solve the subproblem \ref{prob:FAME:subproblem1}.
Let $\hat{\mathbf{A}},  \hat{\mathbf{r}}^{b}(\mathbf{1}_{K\times L})$, and $ \hat{\mathbf{r}}^{[-]}(\cdot)$ be the corresponding estimation from the training datasets. Now we describe how to obtain them from a dataset $\{y(x_i), \{q_k(x_i), y_k(x_i)\}_{k=1}^{K}\}_{i=1}^N$.

To estimate $\mathbf{A}$, we simply apply the empirical mean estimator and obtain $\hat{\mathbf{A}}_{k,\ell} \triangleq \frac{1}{N} \sum_{i=1}^{N} \mathbbm{1}_{\{y_k(x_i) = \ell\}}$.
To estimate ${\mathbf{r}}^{b}(\mathbf{1}_{K\times L})$, and $ {\mathbf{r}}^{[-]}(\cdot)$, we first  compute
$\hat{\psi}_{k_1,k_2,\ell}(\alpha_m) \triangleq \frac{\sum_{i=1}^{N}\mathbbm{1}_{\{y_{k_1}(x_i)=\ell, \hat{q}_{m,k_1,\ell} \geq q_{k_1}(x_i), y_{k_2}(x_i)=y(x_i)\}}}{ \sum_{i=1}^{N}\mathbbm{1}_{\{y_{k_1}(x_i)=\ell, \hat{q}_{m,k_1,\ell} \geq q_{k_1}(x_i) \}}}$, for $\alpha_m = \frac{m}{M}, m\in \{0\}\cup [M]$, where $\hat{q}_{m,k,\ell} \triangleq Quantile(\{ q_k(x_i)|y_k(x_i)=\ell, i\in [N] \}, \alpha_m) $ is the empirical $\alpha_m$-quantile of the quality score of the $k$th service conditional on its predicted label being $\ell$.
Next we estimate $\psi_{k_1,k_2,\ell}(\cdot)$ by linear interpolation, i.e., generating $\hat{\psi}_{k_1,k_2,\ell}(\alpha) \triangleq \frac{\hat{\psi}_{k_1,k_2,\ell}(\alpha_m) -\hat{\psi}_{k_1,k_2,\ell}(\alpha_{m+1})}{\alpha_m - \alpha_{m+1}} (\alpha-\alpha_m) + \hat{\psi}_{k_1,k_2,\ell}(\alpha_m), \alpha \in [\alpha_m, \alpha_{m+1}]$.
We can now estimate $\hat{{\mathbf{r}}}^{a}_{k_1,K(\ell-1)+k_2}(\pmb \rho) \triangleq \psi_{k_1,k_2,\ell}(\pmb \rho_{k_1,\ell})$, $\hat{{\mathbf{r}}}^{b}_{k,\ell}(\pmb \rho) \triangleq \hat{\psi}_{k,k,\ell}(\pmb \rho_{k,\ell})$, 
and finally compute  $\hat{{\mathbf{r}}}^{b}(\mathbf{1}_{K\times L})_{k,\ell} $ and 
$ \hat{\mathbf{r}}^{[-]}(\pmb \rho) \triangleq \hat{\mathbf{r}}^{a}(\pmb \rho)- \hat{\mathbf{r}}^{b}(\pmb \rho) \otimes \mathbf{1}^T_K$.

\paragraph{Solving subproblem \ref{prob:FAME:subproblem1}.}

\begin{algorithm}
\caption{Solver for Problem \ref{prob:FAME:fixbaselabel}.}\label{Alg:FAME:Algorithm_fixbaseandlabel}
	\SetKwInOut{Input}{Input}
	\SetKwInOut{Output}{Output}
	\Input{$\beta, k, \ell,  \hat{\mathbf{r}}^{b}(\mathbf{1}_{K\times L}) $, $\hat{\mathbf{r}}^{[-]}(\cdot) $}
	\Output{the optimal solution $\hat{\rho}^{k,\ell}(\beta), \hat{\pmb \Pi}^{k,\ell}(\beta)$, and the optimal value $\hat{h}^{k,\ell}(\beta)$}  
  \begin{algorithmic}[1]
  
  \STATE Construct $\tilde{\mathbf{r}}^{kl}(\rho)\triangleq \left[ \hat{\mathbf{r}}^{[-]}_{k,K(\ell-1)+1:K\ell}(\rho\mathbf{1}_{K\times L})\right]^T$.

  \STATE Construct $\phi_{i}(\mu) \triangleq \hat{\mathbf{r}}^{b}_{k,\ell}(\mathbf{1}_{K\times L})  + \min\{\frac{\beta}{\mathbf{c}_i},\mu\} \tilde{\mathbf{r}}^{k,\ell}_{i}(\mu) $
  
  \STATE Construct $\phi_{i,j}(\mu) \triangleq \hat{\mathbf{r}}^{b}_{k,\ell}(\mathbf{1}_{K\times L})  + \frac{\beta - \mu \mathbf{c}_j}{\mathbf{c}_i-\mathbf{c}_j} \tilde{\mathbf{r}}^{k,\ell}_{i}(\mu)  + \frac{\mu \mathbf{c}_i-\beta}{\mathbf{c}_i-\mathbf{c}_j} \tilde{\mathbf{r}}^{k,\ell}_{j}(\mu) $
  
  \STATE Compute $(\mu_1, i_1) = \arg \max_{\mu \in[0,1], i\in [K] }  \phi_i(\mu)$ 
  
  \STATE Compute $(\mu_2, i_2, j_2) = \arg \max_{\mu \in[\frac{\beta}{\mathbf{c}_i},\min\{\frac{\beta}{\mathbf{c}_j},1\}], i,j\in [K], \mathbf{c}_i>\mathbf{c}_j }  \phi_{i,j}(\mu)$.

\IF {$\phi_{i_1}(\mu_1) \geq \phi_{i_2,j_2}(\mu_2)$} 
        \STATE $\hat{\rho}^{k,\ell}(\beta) = \mu_1 $, $\hat{\pmb \Pi}^{k,\ell}(\beta) = \left[ \mathbbm{1}_{\mu_1< \frac{\beta}{\mathbf{c}_{i_1}}} + \frac{\beta}{\mathbf{c}_i} \mathbbm{1}_{\mu_1\geq \frac{\beta}{\mathbf{c}_{i_1}}} \right]\mathbf{e}_{i_1}$, $\hat{h}_{k,\ell}(\beta) = \phi_{i_1}(\mu_1)$
\ELSE
     \STATE $\hat{\rho}^{k,\ell}(\beta) = \mu_2 $, $\hat{\pmb \Pi}^{k,\ell}(\beta) =  \frac{\beta/\mu_2-\mathbf{c}_{j_2}}{\mathbf{c}_{j_2}-\mathbf{c}_{j_2}}\mathbf{e}_{i_2}  + \frac{\mathbf{c}_{i_2}-\beta/\mu_2}{\mathbf{c}_{i_2}-\mathbf{c}_{i_2}}\mathbf{e}_{j_2}$, $\hat{h}_{k,\ell}(\beta) = \phi_{i_1}(\mu_1)$.
\ENDIF 

Return $\hat{\rho}^{k,\ell}(\beta), \hat{\pmb \Pi}^{k,\ell}(\beta), \hat{h}_{k,\ell}(\beta)$
\end{algorithmic}
\end{algorithm}

There are 3 steps for solving problem \ref{prob:FAME:subproblem1}. 
First, for $k =i, \ell \in [L]$,  invoke Algorithm \ref{Alg:FAME:Algorithm_fixbaseandlabel}
to compute  $\hat{\rho}^{k,\ell}(\beta_m), \hat{\pmb \Pi}^{k,\ell}(\beta_m), \hat{h}_{k,\ell}(\beta_m)$ where $\beta_m=\frac{m}{M} (b'-\mathbf{c}_k), m=0,1,\cdots, M$.
Next  compute $  t_1^*, t_2^*,\cdots, t_L^* = \arg      \max_{t_1,\cdots, t_L\in[L]\cup \{0\}} \textit{ }  \sum_{\ell=1}^{L} \hat{\mathbf{A}}_{k,\ell} \hat{h}_{k,\ell}(\beta_{t_\ell} ) \textit{ s.t. } \sum_{\ell=1}^{L}t_\ell=M$.
Finally  return $\hat{g}_i(b')\triangleq \sum_{\ell=1}^{L} \hat{\mathbf{A}}_{k,\ell} \hat{h}_{k,\ell}(\beta_{t_\ell^*} )$ as an approximation to the de facto optimal value $g_i(b')$, and the approximately optimal solution $\hat{\mathbf{Q}}_i(b')$ and  $\hat{\mathbf{P}^{[2]}}_i(b')$, where for $\ell \in [L], j\in [K]$, 
$[\hat{\mathbf{P}}^{[2]}_i(b')]_{i,\ell, j} \triangleq \hat{\pmb \Pi}^{i,\ell}_j(\beta_{t_\ell^*})$, $[\hat{\mathbf{P}}^{[2]}_i(b')]_{i',\ell, j} \triangleq 0, i'\not=i$, 
$[\hat{\mathbf{Q}}_i(b')]_{i,\ell} \triangleq Quantile(\{ q_i(x_i)|y_i(x_j)=\ell, j\in [N] \}, \hat{\rho}^{i,\ell}( \beta_{t_\ell^*} ))$, and $[\hat{\mathbf{Q}}_i(b')]_{i',\ell}] = 0, i'\not=i$.

\begin{remark}
Algorithm \ref{Alg:FAME:Algorithm_fixbaseandlabel} effectively solves the problem  \begin{equation}\label{prob:FAME:fixbaselabel}
    \begin{split}
        \max_{\rho, \pmb \Pi\in \Omega_2}\textit{ } &  \hat{\mathbf{r}}^{b}_{k,\ell}(\mathbf{1}_{K\times L}) + \rho \pmb \Pi^T \cdot \tilde{\mathbf{r}}^{k,\ell}(\rho)        \\
        \textit{s.t.        } & \rho (\pmb \Pi - \pmb \Pi \odot \mathbf{e}_k)^T \mathbf{c}   \leq \beta,  
    \end{split}
\end{equation}
where   $\Omega_2 =\{(\pmb \rho, \pmb \Pi) |   \pmb \rho \in [0,1], \pmb \Pi \in \R^{ K},
        \mathbf{0} \preccurlyeq \pmb \Pi \preccurlyeq \mathbf{1}, \pmb \Pi^T \cdot \mathbf{1}_K = 1 \}$ and $\tilde{\mathbf{r}}^{kl}(\rho): \R \mapsto \R^{K}$ is the transpose of  $ \hat{\mathbf{r}}^{[-]}_{k,K(\ell-1)+1:K\ell}(\rho\mathbf{1}_{K\times L})$.
        Observe that the function $\hat{\mathbf{r}}^{[-]}(\cdot)$ by construction is piece wise linear, and thus $\tilde{\mathbf{r}}^{k\ell}(\rho)$ is also piece wise linear. Thus, $\psi_i(\cdot)$ and $\psi_{i,j}(\cdot)$ are piece wise quadratic functions.
        Thus, the optimization problems in Algorithm \ref{Alg:FAME:Algorithm_fixbaseandlabel} (line 4 and line 5) can be efficiently solved, simply by optimizing a quadratic function for each piece.
\end{remark}

\paragraph{Constructing $g_i(\cdot)$.}
We construct an approximation to  $g_i(\cdot)$, denoted by $\hat{g}^{LI}_i(\cdot)$.
The construction is based on linear interpolation using $\hat{g}_i(\theta_m)$ as well as $\hat{g}_i(\mathbf{c}_i)$ which by definition is 0. 
More precisely, $\hat{g}^{LI}_i(\theta)\triangleq 0, \theta\leq \mathbf{c}_i$, $\hat{g}^{LI}_i(\theta)\triangleq \frac{\hat{g}_i(\theta_m)-\hat{g}_i(\theta_{m+1})}{\theta_m-\theta_{m+1}} (\theta-\theta_{m+1})+\hat{g}_i(\theta_{m+1}), \theta_{m+1}\geq \theta \geq \theta_m \geq \mathbf{c}_i$,
 and $\hat{g}^{LI}_i(\theta)\triangleq \frac{\hat{g}_i(\theta_m)}{\theta_m} \theta, \theta_{m} \geq \theta \geq \mathbf{c}_i\geq \theta_{m-1}$. Here, $\theta_m \triangleq b'_m=\frac{\|2\mathbf{c}\|_{\infty}}{M}$.

\paragraph{Solving Master Problem \ref{prob:FAME:master}.}
To solve Problem \ref{prob:FAME:master}, let us first denote   $\Omega_3= \{\mathbf{x}\in \R^4| x  \succcurlyeq 0, \mathbf{x}_1+\mathbf{x}_2 = 1\}$ and  $\Omega_{3,m_1,m_2} \triangleq \{ \mathbf{x}\in\Omega_3| \theta_{m_{i}-1}\mathbf{x}_{i}\leq \mathbf{x}_{i+3} \leq \theta_{m_{i}} \mathbf{x}_{i+3},i=1,2\}$, for $m_1\in[M], m_2\in [M]$.
For each $i_1, i_2, m_1, m_2$,  first compute $ \hat{g}^{\sum}(i_1,i_2,m_1,m_2) \triangleq \max_{(p_1, p_2, b_1, b_2)\in \Omega_{3,m_1,m_2}}\textit{ }  p_1  \hat{g}^{LI}_{i_1}( b_1/ p_1 ) + p_2 \hat{g}^{LI}_{i_2}(b_2/p_2)        \textit{ s.t. } b_1+b_2 = b$, a linear programming by construction.
Next compute $i_1^*,i_2^*, m_1^*, m_2^* \triangleq \arg \max_{i_1,i_2,m_1,m_2} \hat{g}^{\sum}(i_1,i_2,m_1,m_2)$ and $ (p_1^*,p_2^*,b_1^*,b_2^*) \triangleq \max_{(p_1, p_2, b_1, b_2)\in \Omega_{3,m_1^*,m_2^*}}\textit{ }  p_1  \hat{g}^{LI}_{i_1^*}( b_1/ p_1 ) + p_2 \hat{g}^{LI}_{i_2^*}(b_2/p_2)        \textit{ s.t. } b_1+b_2 = b$.
Finally return the corresponding solution $i_1^*, i_2^*, p_1^*, p_2^*, b_1^*, b_2^*$.

\section{Missing Proofs}
\subsection{Helpful Lemmas}
We first provide some useful lemmas throughout this section.
\begin{lemma}\label{Lemma:FAME:LPSparseStructure}
	Suppose the linear optimization problem
	\begin{equation*}
	\begin{split}
	\max_{\mathbf{z} \in \mathbb{R}^{K} } \textit{ }& \mathbf{u}^T\mathbf{z} \\
	s.t. \textit{ } & \mathbf{v}^T \mathbf{z} \leq w, \mathbf{1}^T \mathbf{z} \leq 1,\mathbf{z}\geq 0 
	\end{split}
	\end{equation*} is feasible.  Then there exists one optimal solution $\mathbf{z}^*$ such that $\|\mathbf{z}^*\|_0\leq 2$.
\end{lemma}	
\begin{proof}
	Let $\mathbf{z}^*$ be one solution.
	If $\|\mathbf{z}^*\|_0\leq 2$, then the statement holds.
	Suppose $\|\mathbf{z}^*\|_0 = \textit{nnz} > 2$ (and thus $K\geq 3$).
	W.l.o.g., let the first $\textit{nnz}$ elements in $\mathbf{z}^*$ be the nonzero elements.
	Let $i_{\min}=\arg \min_{i:i\leq \textit{nnz}} \mathbf{v}_i$ and $i_{\max}=\arg \max_{i:i\leq \textit{nnz}} \mathbf{v}_i$. 
	If $\mathbf{v}_{i_{\max}}>\mathbf{v}_{ i_{\min} } $, construct $\mathbf{z}'$ by 	
	\begin{equation*}
	\begin{split}
	\mathbf{z}_i' = \begin{cases}
	\mathbf{z}^*_i(=0), & \textit{if $i>\text{nnz}$} \\
	\frac{ \mathbf{v}^T_{1:\text{nnz}} \mathbf{z}^*_{1:\textit{nnz}} - \mathbf{v}_{i_{\min}}\mathbf{1}^T \mathbf{z}_{1:\textit{nnz}}  }{\mathbf{v}_{i_{\max}}-\mathbf{v}_{ i_{\min} } }, & \textit{if $i=i_{\max}$}\\
	\frac{ - \mathbf{v}^T_{1:\textit{nnz}} \mathbf{z}^*_{1:\textit{nnz}} + \mathbf{v}_{i_{\max}}\mathbf{1}^T \mathbf{z}_{1:\textit{nnz}}  }{\mathbf{v}_{i_{\max}}-\mathbf{v}_{ i_{\min} } }, & \textit{if $i=i_{\min}$}\\
	0,& \textit{otherwise} 
	\end{cases}
	\end{split}
	\end{equation*}
	Otherwise, construct $\mathbf{z}'$ by 	
	\begin{equation*}
	\begin{split}
	\mathbf{z}_i' = \begin{cases}
	\mathbf{z}^*_i(=0), & \textit{if $i>\text{nnz}$} \\
	\mathbf{1}^T \mathbf{z}_{1:\textit{nnz}}^{*}, & \textit{if $i=i_{\max}$}\\
	0,& \textit{otherwise} 
	\end{cases}
	\end{split}
	\end{equation*}	
	
	Now our goal is to prove that $\mathbf{z}'$ is one optimal solution and $\|\mathbf{z}'\|_0\leq 2$.
	
	(i) We first show that $\mathbf{z}'$ is a feasible solution.
	
	(1) $\mathbf{v}_{i_{\max}}>\mathbf{v}_{ i_{\min} } $:
	If $i\not\in\{i_{\max},i_{\min}\}$, clearly $\mathbf{z}'_i=0\geq 0$. 
	Since $\mathbf{z}^*$ is feasible, $\mathbf{z}^*_{1:\textit{nnz}}\geq 0$.
	By definition, $\mathbf{z}'_{i_{\max}} = \frac{ \mathbf{v}^T_{1:\text{nnz}} \mathbf{z}^*_{1:\textit{nnz}} - \mathbf{v}_{i_{\min}}\mathbf{1}^T \mathbf{z}_{1:\textit{nnz}}  }{\mathbf{v}_{i_{\max}}-\mathbf{v}_{ i_{\min} } } = \frac{ 1 }{\mathbf{v}_{i_{\max}}-\mathbf{v}_{ i_{\min} } }  \sum_{j=1}^{\textit{nnz}} (\mathbf{v}_{j} - \mathbf{v}_{i_{\min}}) \mathbf{z}^*_{j} \geq 0$, and similarly $\mathbf{z}_{i_{\min}}'\geq 0$.
	Thus, we have $\mathbf{z}'\geq 0 $.
	
	In addition, 
	\begin{equation*}
	\begin{split}
	\mathbf{v}^T \mathbf{z}' & = \frac{ \mathbf{v}^T_{1:\text{nnz}} \mathbf{z}^*_{1:\textit{nnz}} - \mathbf{v}_{i_{\min}}\mathbf{1}^T \mathbf{z}_{1:\textit{nnz}}  }{\mathbf{v}_{i_{\max}}-\mathbf{v}_{ i_{\min} } } \mathbf{v}_{i_{\max}} + 	\frac{ - \mathbf{v}^T_{1:\textit{nnz}} \mathbf{z}^*_{1:\textit{nnz}} + \mathbf{v}_{i_{\max}}\mathbf{1}^T \mathbf{z}_{1:\textit{nnz}}  }{\mathbf{v}_{i_{\max}}-\mathbf{v}_{ i_{\min} } } \mathbf{v}_{i_{min}} \\
	&=\mathbf{v}^T_{1:\text{nnz}} \mathbf{z}^*_{1:\textit{nnz}}  = \mathbf{v}^T \mathbf{z}^*\leq w
	\end{split}
	\end{equation*}
	where the last equality is due to the fact that $\mathbf{z}_{i}=0, \forall i >\textit{nnz}$.
	Similiarly, we have 
	\begin{equation*}
	\begin{split}
	\mathbf{1}^T \mathbf{z}' & = \frac{ \mathbf{v}^T_{1:\text{nnz}} \mathbf{z}^*_{1:\textit{nnz}} - \mathbf{v}_{i_{\min}}\mathbf{1}^T \mathbf{z}_{1:\textit{nnz}}  }{\mathbf{v}_{i_{\max}}-\mathbf{v}_{ i_{\min} } }  + 	\frac{ - \mathbf{v}^T_{1:\textit{nnz}} \mathbf{z}^*_{1:\textit{nnz}} + \mathbf{v}_{i_{\max}}\mathbf{1}^T \mathbf{z}_{1:\textit{nnz}}  }{\mathbf{v}_{i_{\max}}-\mathbf{v}_{ i_{\min} } }  \\
	&=\mathbf{1}^T_{1:\text{nnz}} \mathbf{z}^*_{1:\textit{nnz}}  = \mathbf{1}^T \mathbf{z}^*\leq 1.
	\end{split}
	\end{equation*}
	
	(2) $\mathbf{v}_{i_{\max}}=\mathbf{v}_{ i_{\min} } $:
	It is clear that $\mathbf{z}'\geq 0$ and $\mathbf{1}^T \mathbf{z} \leq 1$ by definition.
	Note that  by definition $\mathbf{v}^T \mathbf{z}' = \mathbf{v}_{i_{\max}} \mathbf{1}^T \mathbf{z}^*_{1:\textit{nnz}}$.
	$\mathbf{v}_{i_{\max}}=\mathbf{v}_{ i_{\min} } $ implies that for $i=1,2,\cdots, \textit{nnz}, \mathbf{v}_i = \mathbf{v}_{i_{\max}}$, and thus $ \mathbf{v}_{i_{\max}} \mathbf{1}^T \mathbf{z}^*_{1:\textit{nnz}} =  \mathbf{v}^T_{1:\textit{nnz}} \mathbf{z}^*_{1:\textit{nnz}}$.
	Note that only the first $\textit{nnz}$ elements in $\mathbf{z}^*$ are nonzeros, we have  $\mathbf{v}^T_{1:\textit{nnz}} \mathbf{z}^*_{1:\textit{nnz}} = \mathbf{v}^T \mathbf{z}^*$.
	That is to say,
	\begin{equation*}
	\mathbf{v}^T \mathbf{z}' = \mathbf{v}^T \mathbf{z}^* \leq w
	\end{equation*}
	Hence, we have shown that $	\mathbf{v}^T \mathbf{z}' \leq w, 	\mathbf{1}^T \mathbf{z}' \leq 1, \mathbf{z}' \geq 0$ always hold, i.e., $\mathbf{z}'$ is a feasible solution to the linear optimization problem.
	
	(ii) Now we show that $\mathbf{z}'$ is one optimal solution, i.e., $\mathbf{u}^T \mathbf{z}' = \mathbf{u}^T \mathbf{z}^*$.
	
	The Lagrangian function of the linear optimization problem is 
	\begin{equation*}
	\begin{split}
	\mathcal{L}(\mathbf{z},\pmb{\mu}) =  \mathbf{u}^T \mathbf{z} +  \pmb \mu_1  (\mathbf{v}^T \mathbf{z}-w) + \pmb\mu_2(\mathbf{1}^T \mathbf{z} )+ \sum_{i=1}^{K} \pmb \mu_{i+2}(-\mathbf{z}_i)
	\end{split}
	\end{equation*}
	Since $\mathbf{z}^*$ is one optimal solution and clearly LCQ (linearity constraint qualification) is satisfied, KKT conditions must hold. 
	That is, there exists $\pmb \mu$ such that 
	\begin{equation*}
	\begin{split}
	& \frac{\partial \mathcal{L}(\mathbf{z}^*, \pmb\mu) }{\partial \mathbf{z}_i} =  \mathbf{u}_i+ \pmb\mu_1  \mathbf{v}_i   + \pmb\mu_2 -\pmb \mu_{i+2} = 0, \forall i \\
	&	\pmb \mu_1  (\mathbf{v}^T \mathbf{z}^*-w) =0,  \pmb\mu_2(\mathbf{1}^T \mathbf{z}^* ) = 0 ,  \pmb \mu_{i+2}\mathbf{z}^*_i=0, \forall i\\
	&  \pmb \mu \geq 0 \\
	& \mathbf{v}^T \mathbf{z}^* \leq w, \mathbf{1}^T \mathbf{z}^* \leq 1,\mathbf{z}\geq 0 
	\end{split}
	\end{equation*}
	For ease of exposition, denote $\hat{\pmb \mu} = [\pmb\mu_{3},\pmb\mu_4,\cdots, \pmb\mu_{K+2}]^T$.
	The first condition implies $ \mathbf{u}_i = - \pmb\mu_1  \mathbf{v}_i   - \pmb\mu_2 + \pmb \mu_{i+2}, \forall i$, which is equivalent to $\mathbf{u} = -\pmb \mu_1\mathbf{v} -\pmb \mu_2 \mathbf{1} +\hat{\pmb\mu}$.
	Thus, we have
	\begin{equation*}
	\begin{split}
	\mathbf{u}^T \mathbf{z}' = -\pmb \mu_1\mathbf{v}^T \mathbf{z}'  -\pmb \mu_2 \mathbf{1}^T \mathbf{z}' +\hat{\pmb\mu}^T \mathbf{z}'
	\end{split}
	\end{equation*}
	\begin{equation*}
	\begin{split}
	\mathbf{u}^T \mathbf{z}^* = -\pmb \mu_1\mathbf{v}^T \mathbf{z}^*  -\pmb \mu_2 \mathbf{1}^T \mathbf{z}^* +\hat{\pmb\mu}^T \mathbf{z}^*
	\end{split}
	\end{equation*}
	
	The condition  $\pmb \mu_{i+2}\mathbf{z}^*_i=0$ implies that at least one of the terms must  be 0. 
	Since it holds for every $i$, the summation over $i$ is also 0, i.e.,
	$\hat{\pmb\mu}^T \mathbf{z}^* = \sum_{i=1}^{K}\pmb \mu_{i+2}\mathbf{z}^*_i =0 $.
	Noting that the first $\textit{nnz}$ elements in $\mathbf{z}^*$ are nonzeros, we must have $\pmb\mu_{i+2} = 0,  i\leq \textit{nnz}$, and in particular, $\pmb\mu_{i_{\max}+2} = \pmb\mu_{i_{\min}+2} = 0 $.
	Hence,    $\hat{\pmb\mu}^T \mathbf{z}' = \sum_{i=1}^{K}\pmb \mu_{i+2}\mathbf{z}'_i = \pmb\mu_{i_{\max}+2} \mathbf{z}'_{i_{\max}}+ \pmb\mu_{i_{\min}+2} \mathbf{z}_{i_{\min}}' = 0$.
	Thus, we have
	\begin{equation*}
	\begin{split}
	\mathbf{u}^T \mathbf{z}' = -\pmb \mu_1\mathbf{v}^T \mathbf{z}'  -\pmb \mu_2 \mathbf{1}^T \mathbf{z}' 
	\end{split}
	\end{equation*}
	\begin{equation*}
	\begin{split}
	\mathbf{u}^T \mathbf{z}^* = -\pmb \mu_1\mathbf{v}^T \mathbf{z}^*  -\pmb \mu_2 \mathbf{1}^T \mathbf{z}^* 
	\end{split}
	\end{equation*}
	In part (i), it is shown that $\mathbf{v}^T\mathbf{z}^*=\mathbf{v}^T\mathbf{z}'$ and $\mathbf{1}^T\mathbf{z}^*=\mathbf{1}^T\mathbf{z}'$.
	Hence, we must have
	\begin{equation*}
	\begin{split}
	\mathbf{u}^T \mathbf{z}^* = \mathbf{u}^T \mathbf{z}'
	\end{split}
	\end{equation*}
	In other words, $\mathbf{z}'$ has the same objective function value as $\mathbf{z}^*$.
	Since $\mathbf{z}^*$ is one optimal solution, $\mathbf{z}'$ must also be one optimal solution (since it is also feasible as shown in part (i)).
	By definition, $\|\mathbf{z}'\|_0 \leq 2$, which finishes the proof.
\end{proof}

\begin{lemma}\label{Lemma:FAME:LPContinousProperty}
	Let $F(w)$ be the optimal value of the linear optimization problem
	\begin{equation*}
	\begin{split}
	\max_{\mathbf{z} \in \mathbb{R}^{K} } \textit{ }& \mathbf{u}^T\mathbf{z} \\
	s.t. \textit{ } & \mathbf{v}^T \mathbf{z} \leq w, \mathbf{z}\geq 0 , \mathbf{C} \mathbf{z} \leq \mathbf{d}
	\end{split}
	\end{equation*} 
	where $\mathbf{u}, \mathbf{C}, \mathbf{v} \geq \mathbf{0}$, $\mathbf{d} > 0$.
	Then $F(w)$ is Lipschitz continuous.
\end{lemma}	
\begin{proof}
Note that since $\mathbf{d}>\mathbf{0}$, there exists some $w^*$, such that its corresponding optimal $\mathbf{z}^*$ satisfies $\mathbf{C}\mathbf{z}^*< d$. 
Thus, $\mathbf{z}^*$ must also be the optimal solution to 
	\begin{equation}\label{prob:FAME:temp1}
	\begin{split}
	\max_{\mathbf{z} \in \mathbb{R}^{K} } \textit{ }& \mathbf{u}^T\mathbf{z} \\
	s.t. \textit{ } & \mathbf{v}^T \mathbf{z} \leq w^*, \mathbf{z}\geq 0 
	\end{split}
	\end{equation} 
If $\mathbf{v}^T\mathbf{z}^*<w^*$, then $\hat{\mathbf{z}} = \frac{w^*}{\mathbf{v}^T\mathbf{z}}\mathbf{z}^*$ is also a feasible solution, but $\mathbf{u}^T \hat{\mathbf{z}}= \frac{w^*}{\mathbf{v}^T\mathbf{z}}\mathbf{u}^T\mathbf{z}^*> \mathbf{u}^T\mathbf{z}^*$, a contradiction.
Thus, we must have $\mathbf{v}^T\mathbf{z}^*=w^*$.
Now we claim  that $\mathbf{z}' = \frac{w'}{w^*}\mathbf{z}^*$ is one optimal solution to 
	\begin{equation}\label{prob:FAME:temp2}
	\begin{split}
	\max_{\mathbf{z} \in \mathbb{R}^{K} } \textit{ }& \mathbf{u}^T\mathbf{z} \\
	s.t. \textit{ } & \mathbf{v}^T \mathbf{z} \leq w', \mathbf{z}\geq 0
	\end{split}
	\end{equation}
Suppose not. Then there exists another optimal solution $\mathbf{z}''$. Since $\mathbf{z}'$ is not optimal, we must have $\mathbf{u}^T \mathbf{z}'' > \mathbf{u}^T \mathbf{z}' =  \frac{w'}{w^*}\mathbf{u}^T \mathbf{z}^*$.
Now let $\mathbf{z}''' = \frac{w^*}{w'} \mathbf{z}''$.
Then by definition, we must have $\mathbf{z}'''$ is a solution to problem \ref{prob:FAME:temp1}, since $\mathbf{z}'''\geq 0$ and $\mathbf{v}^T \mathbf{z}''' = \mathbf{v}^T \frac{w^*}{w'} \mathbf{z}'' \leq \frac{w^*}{w'} \mathbf{v}^T  \mathbf{z}'' \leq \frac{w^*}{w'} w' = w^*$.
However, $\mathbf{u}^T \mathbf{z}''' = \mathbf{u}^T \frac{w^*}{w'} \mathbf{z}'' = \frac{w^*}{w'} \mathbf{u}^T \mathbf{z}'' > \frac{w^*}{w'} \mathbf{u}^T \mathbf{z}' = \frac{w^*}{w'} \frac{w'}{w^*} \mathbf{u}^T \mathbf{z}^* = \mathbf{u}^T \mathbf{z}^*$.
That is to say, $\mathbf{z}'''$ is a feasible solution to problem \ref{prob:FAME:temp1} but also have a objective value that is strictly higher than that of the optimal solution.
A contradiction.

Thus we must have that $\mathbf{z}'$ is one optimal solution to problem \ref{prob:FAME:temp2}.

Now we consider $0\leq w' \leq w^*$ and $ w' \geq w^*$ separately.

case (i): Suppose $0\leq w' \leq w^*$.
Note that $\mathbf{v}^T\mathbf{z}^* \leq w^*$ since $\mathbf{z}^*$ is a feasible solution to problem \ref{prob:FAME:temp1} and by assumption $\mathbf{C}\mathbf{z}^* < \mathbf{d}$.
Hence, we must have $\mathbf{C} \mathbf{z}' = \frac{w'}{w^*}\mathbf{C} \mathbf{z}'<\mathbf{d}$.
That is to say, adding a constraint $\mathbf{C}\mathbf{z}\leq \mathbf{d}$ to problem \ref{prob:FAME:temp2} does not change the optimal solution. Thus, $\mathbf{z}'$ is also an optimal solution to 
 	\begin{equation*}
	\begin{split}
	\max_{\mathbf{z} \in \mathbb{R}^{K} } \textit{ }& \mathbf{u}^T\mathbf{z} \\
	s.t. \textit{ } & \mathbf{v}^T \mathbf{z} \leq w^*, \mathbf{z}\geq 0, \mathbf{C}\mathbf{z} \leq \mathbf{d} 
	\end{split}
	\end{equation*} 
Hence, we have $F(w') = \mathbf{u}^T \mathbf{z}' = \mathbf{u}^T \frac{w'}{w^*} \mathbf{z}^* = \frac{w'}{w^*} F(w^*)$.
That is to say, if $w'\leq w$, then $F(w')$ is a linear function of $w'$ and thus must be Lipschitz continuous.

case (ii): Suppose $w' \geq w^*$.
Note that we have just shown that $\mathbf{z}'$ is one optimal solution to problem \ref{prob:FAME:temp2}.
Adding a constraint to problem problem \ref{prob:FAME:temp2} only leads to smaller objective value.
That is to say, $\mathbf{u}^T \mathbf{z}' \geq F(w')$, which is the optimal value to 
 	\begin{equation*}
	\begin{split}
	\max_{\mathbf{z} \in \mathbb{R}^{K} } \textit{ }& \mathbf{u}^T\mathbf{z} \\
	s.t. \textit{ } & \mathbf{v}^T \mathbf{z} \leq w^*, \mathbf{z}\geq 0, \mathbf{C}\mathbf{z} \leq \mathbf{d} 
	\end{split}
	\end{equation*} 
On the other hand, by definition, we have $\mathbf{u}^T \mathbf{z}' = \mathbf{u}^T \frac{w'}{w^*} \mathbf{z}^* = \frac{w'}{w^*} F(w^*)$, and thus we have 
\begin{equation}\label{prob:FAME:equation:temp1}
    \begin{split}
\frac{w'}{w^*} F(w^*) \geq F(w')
    \end{split}
\end{equation}

Now let us consider $ w^1 \geq w^1 \geq w^*$.
Let $\mathbf{z}^1, \mathbf{z}^2$ be their corresponding solutions.
Since $w^2 \geq w^1$, we have $F(w^2) \geq F(w^1)$. Let $\mathbf{z}^3 = \frac{w^1}{w^2} \mathbf{z}^2$. Then $\mathbf{v}^T \mathbf{z}^3 = \mathbf{v}^T \frac{w^1}{w^2} \mathbf{z}^2 \leq \frac{w^1}{w^2} w^2 =w^1$ and $\mathbf{C} \mathbf{z}^3 = \mathbf{C} \frac{w^1}{w^2} \mathbf{z}^2\leq \frac{w^1}{w^2} \mathbf{d} \leq \mathbf{d}$.
That is to say, $\mathbf{z}^3$ is also a solution to 
	\begin{equation*}
	\begin{split}
	\max_{\mathbf{z} \in \mathbb{R}^{K} } \textit{ }& \mathbf{u}^T\mathbf{z} \\
	s.t. \textit{ } & \mathbf{v}^T \mathbf{z} \leq w^1, \mathbf{z}\geq 0, \mathbf{C}\mathbf{z} \leq \mathbf{d} 
	\end{split}
	\end{equation*} 
Thus, the objective value must be smaller than the optimal one, i.e., $\mathbf{u}^T \mathbf{z}^3 \leq F(w^1)$.
Noting that $\mathbf{u}^T \mathbf{z}^3 = \mathbf{u}^T \frac{w^1}{w^2} \mathbf{z}^2 = \frac{w^1}{w_2} F(w^2)$, we have $\frac{w^1}{w_2} F(w^2) \leq F(w^1)$.
 which is $F(w^2) - F(w^1) \leq \frac{w_2-w_1}{w_1}F(w^1)$.
Note that we have proved $\frac{w'}{w^*}F(w^*) \geq F(w')$ in \ref{prob:FAME:equation:temp1}, i.e., $\frac{F(w')}{w'} \leq \frac{F(w^*)}{w^*}$, for any $w' \geq w^*$.
Thus, we have $F(w^1)/w^1 \leq \frac{F(w^*)}{w^*}$.
That implies $F(w^2) - F(w^1) \leq \frac{w_2-w_1}{w_1}F(w^1) \leq (w^2-w^1) \frac{F(w^*)}{w^*}$.
We also have $F(w^1) \leq F(w^2)$.
That is to say, for any $w^2\geq w^1 \geq w^*$, we have $-(w^2-w^1) \frac{F(w^*)}{w^*} \leq 0 \leq F(w^2) - F(w^1) \leq (w^2-w^1) \frac{F(w^*)}{w^*}$
and thus we have just proved that $f(w')$ is Lipschitz continuous for $w'\geq w^*$.

Now let us consider all $w$. We have shown that $F(w)$ is Lipschitz continuous when $w\leq w^*$ and when $w\geq w^*$. 
Let $\gamma_1$ and $\gamma_2$ denote the Lipschitz constant for both case. 
Now we can prove that $F(w)$ is Lipschitz continuous with constant $\gamma_1+\gamma_2$ for any $w\geq 0$.

Let us consider any two $w_1, w_2$. If they are both smaller than $w^*$ or larger than $w^*$, then clearly we must have $|F(w_1) - F(w_2)|\leq (\gamma_1+\gamma_2) |w_1-w_2|$.
We only need to consider when $w_1 \leq w^*$ and $w_2 \geq w^*$.
As $F(w)$ is  Lipschitz continuous on each side, we have 
\begin{equation*}
    \begin{split}
        |F(w_1) - F(w_2)| & = |F(w_1) - F(w^*) + F(w^*) - F(w_2)|\\
        &\leq |F(w_1) - F(w^*)| + |F(w^*) - F(w_2)|\\
        & \leq \gamma_1 |w_1-w^*| + \gamma_2 |w_2-w^*|\\
        & \leq \gamma_1 |w_1 -w_2| + \gamma |w_2-w_1| = (\gamma_1+\gamma_2) |w_1-w_2|
    \end{split}
\end{equation*}
where the first inequality is by triangle inequality,  the second ineuqaltiy is by the Lipschitz continuity of $F(w)$ on each side, and the last inequality is due to the assumption that $w_1\leq w^*$ and $w_2\geq w^*$.
Thus, we can conclude that $F(w)$ must be Lipschitz continuous on for any $w\geq 0$.

\end{proof}

\begin{lemma}\label{lemma:FAME:interpolationerrorbound}
Suppose function $f(x)$ is a Lipschitz continuous with constant $\Delta^1$ on the interval $[a,b]$. Let $x_i = \frac{i}{M}(b-a), i=0,1,\cdots,M$.
Assume for all $i$, $|\hat{f}(x_i) - f(x_i)|\leq \Delta^2$.
Let $\hat{f}^{LI}(x)$ be the linear interpolation using $\hat{f}(x_i)$, i.e., 
$\hat{f}^{LI}(x) \triangleq \frac{\hat{f}(x_i) - \hat{f}(x_{i-1})}{x_{i}-x_{i-1}} (x-x_i) + \hat{f}(x_i), x_{i-1}\leq x\leq x_{i}, \forall i\in [M]$.
Then we have $|f(x) - \hat{f}^{LI}(x)| \leq 3\Delta^2 + \frac{2\Delta_1(b-a)}{M} $
\end{lemma}
\begin{proof}
For simplicity, let $\mu = \frac{b-a}{M}$.
Suppose $x_{i-1} \leq x \leq x_i$. 
By construction of $\hat{f}^{LI}(x)$, we must have 
\begin{equation*}
    \begin{split}
        |\hat{f}^{LI}(x_i) - \hat{f}^{LI}(x)| \leq & |\hat{f}^{LI}(x_i) - \hat{f}^{LI}(x_{i-1})| = |\hat{f}(x_i) - \hat{f}(x_{i-1})|\\
        = &|\hat{f}(x_i) - f(x_i) + f(x_i) - f(x_{i-1})+ f(x_{i-1})- \hat{f}(x_{i-1})|\\
        \leq & |\hat{f}(x_i) - f(x_i)|  + |f(x_i) - f(x_{i-1})|+ |f(x_{i-1}) - \hat{f}(x_{i-1})|\\
        \leq & \Delta^2 + \Delta^1 |x_i-x_{i-1}|  + \Delta^2 = 2\Delta^2+ \Delta^1 \mu
    \end{split}
\end{equation*}
where the last inequality is due to the Lipschitz continuity and assumption $|\hat{f}(x_i) - f(x_i)|\leq \Delta^2$.
Since function $f(x)$ is Lipschitz continuous with constant $\Delta^1$, we have
\begin{equation*}
    \begin{split}
        |f(x)-f(x_i)| \leq \Delta^1 |x-x_i| \leq \Delta^1 |x_i-x_{i-1}| = \Delta^1 \mu
    \end{split}
\end{equation*}
By assumption, we have $|\hat{f}(x_i) - f(x_i)|\leq \Delta^2$.

Combining the above, we have
\begin{equation*}
\begin{split}
    |f(x) - \hat{f}^{LI}(x)| & =  |f(x)-f(x_i)+ f(x_i) - \hat{f}^{LI}(x_i) + \hat{f}^{LI}(x_i) - \hat{f}^{LI}(x)|\\
    &\leq  |f(x)-f(x_i)| + |f(x_i) - \hat{f}^{LI}(x_i)| + |\hat{f}^{LI}(x_i) - \hat{f}^{LI}(x)|\\
    &=  |f(x)-f(x_i)| + |f(x_i) - \hat{f}(x_i)| + |\hat{f}^{LI}(x_i) - \hat{f}^{LI}(x)|\\
    & \leq \Delta^1 \mu + \Delta^2 + 2\Delta^2 + \Delta^1 \mu =  3\Delta^2 + 2\Delta^1 \mu = 3\Delta^2 + \frac{2\Delta_1(b-a)}{M}
\end{split}
\end{equation*}
where the first inequality is due to triangle inequality, and the second inequality is simply applying what we have just shown.
Note that this holds regardless of the value of $i$. Thus, this holds for any $x$, which completes the proof.
\end{proof}

\begin{lemma}\label{Lemma:FAME:OptimiziationDiff}
Let $f, f', g, g'$ be functions defined on $\Omega_{\mathbf{z}}$, such that $\max_{\mathbf{z}\in \Omega_{\mathbf{z}}}|(f\mathbf{z}) - f'(\mathbf{z})|\leq \Delta_1$
and $\max_{\mathbf{z}\in \Omega_{\mathbf{z}}}|g(\mathbf{z}) - g'(\mathbf{z})|\leq  \Delta_2$.
If
\begin{equation*}
    \begin{split}
        \mathbf{z}^* = \arg \max_{\mathbf{z}\in \Omega_{\mathbf{z}}} & f(\mathbf{z})\\
        s.t. &  g(\mathbf{z}) \leq 0
    \end{split}
\end{equation*}
and
\begin{equation*}
    \begin{split}
        \mathbf{z}' = \arg \max_{\mathbf{z}\in \Omega_{\mathbf{z}}} & f'(\mathbf{z})\\
        s.t. &  g'(\mathbf{z}) \leq  \Delta_2,
    \end{split}
\end{equation*}
then we must have 
\begin{equation*}
    \begin{split}
        f(\mathbf{z}') & \geq f(\mathbf{z}^*) -2\Delta_1\\
        g(\mathbf{z}') & \leq 2 \Delta_2.
    \end{split}
\end{equation*}
\end{lemma}	
\begin{proof}
Note that  $\max_{\mathbf{z}\in \Omega_{\mathbf{z}}}|(f\mathbf{z}) - f'(\mathbf{z})|\leq \Delta_1$ implies $f(\mathbf{z})\geq f'(\mathbf{z})-\Delta_1$ for any $\mathbf{z}\in\Omega_{\mathbf{z}}$.
Specifically, 
\begin{equation*}
f(\mathbf{z}')\geq f'(\mathbf{z}')-\Delta_1
\end{equation*}
Noting  $\max_{\mathbf{z}\in \Omega_{\mathbf{z}}}|g(\mathbf{z}) - g'(\mathbf{z})|\leq  \Delta_2$,
we have $g'(\mathbf{z}^*) \leq g(\mathbf{z}^*)+\Delta_2 \leq \Delta_2$, where the last inequality is due to $g(\mathbf{z}^*)\leq 0$ by definition.
Since, $\mathbf{z}^*$ is a feasible solution to the second optimization problem, and the optimal value must be no smaller than the value at $\mathbf{z}^*$.
That is to say,
\begin{equation*}
    \begin{split}
        f'(\mathbf{z}') \geq f'(\mathbf{z}^*)
    \end{split}
\end{equation*}
Hence we have 
\begin{equation*}
f(\mathbf{z}')\geq f'(\mathbf{z}')-\Delta_1 \geq  f'(\mathbf{z}^*) - \Delta_1
\end{equation*}
In addition, $\max_{\mathbf{z}\in \Omega_{\mathbf{z}}}|(f\mathbf{z}) - f'(\mathbf{z})|\leq \Delta_1$ implies $f'(\mathbf{z})\geq f(\mathbf{z})-\Delta_1$ for any $\mathbf{z}\in\Omega_{\mathbf{z}}$.
Thus, we have $f'(\mathbf{z}^*)\geq f(\mathbf{z})^*-\Delta_1$ and thus
\begin{equation*}
f(\mathbf{z}')\geq   f'(\mathbf{z}^*) - \Delta_1 \geq 
f(\mathbf{z}^*) - 2 \Delta_1
\end{equation*}

By  $\max_{\mathbf{z}\in \Omega_{\mathbf{z}}}|g(\mathbf{z}) - g'(\mathbf{z})|\leq  \Delta_2$,
we have $g(\mathbf{z}') \leq g'(\mathbf{z}')+\Delta_2 \leq 2 \Delta_2$, where the last inequality is by definition of $\mathbf{z}'$.
\end{proof}

\subsection{Proof of Lemma \ref{lemma:FAME:sparsitymainpaper}}
\begin{proof}
Given the expected accuracy and cost provided by Lemma \ref{lemma:FAME:expectedacccostmainpaper}, the problem  \ref{prob:FAME:optimaldefinition} becomes a linear programming over $\Pr[A_s^{[1]}=i] = \mathbf{p}^{[1]}_i$, where the constraints are $\mathbf{p}^{[1]}\geq 0, \mathbf{1}^{T} \mathbf{p}^{[1]} = 1$ and another linear constraint on $\mathbf{p}^{[1]}$. 
Note that all items in the optimization are positive, and thus changing the constraint to 
$\mathbf{p}^{[1]} = 1$ to $\mathbf{p}^{[1]} \leq  1$ does not change the optimal solution.
Now, given the constraint $\mathbf{p}^{[1]}\geq 0, \mathbf{1}^{T} \mathbf{p}^{[1]} \leq 1$ and one more constraint on $\mathbf{p}^{[2]}$ for the linear programing problem over $\mathbf{p}^{[1]}$, we can apply Lemma  \ref{Lemma:FAME:LPSparseStructure}, and conclude that there exists an optimal solution where $\|\mathbf{p}^{[1]*} \| \leq 2$.
\end{proof}

\subsection{Proof of Lemma \ref{lemma:FAME:expectedacccostmainpaper}}
\begin{proof}

Let us first consider the expected accuracy.
By law of total expectation, we have
\begin{equation*}
    \Exp[r^s(x)] = \sum_{i=1}^{K} \Pr[A_s^{[1]} = i ]\Exp[r^s(x)|A_s^{[1]} = i]
\end{equation*}
And we  can further expand the conditional expectation by 
\begin{equation*}
    \begin{split}
   & \Exp[r^s(x) | A^{[1]}_s=i] \\
= &  \Pr[D_s=0|A^{[1]}_s=i, ] \Exp[r^s(x) | A^{[1]}_s=i, D_s=0] +  \Pr[D_s=1|A^{[1]}_s=i, ] \Exp[r^s(x) | A^{[1]}_s=i, D_s=1]  \\
= &  \Pr[D_s=0|A^{[1]}_s=i, ] \Exp[r^i(x) | A^{[1]}_s=i, D_s=0] +  \Pr[D_s=1|A^{[1]}_s=i, ] \Exp[r^s(x) | A^{[1]}_s=i, D_s=1]
    \end{split}
\end{equation*}
where the last equality is because when $D_s=0$, i.e., no add-on service is called, the strategy always uses the base service's prediction and thus $r^s(x) = r^i(x)$.
For the second term, we can bring in $A_s^{[2]}$ and apply law of total expectation, to obtain
\begin{equation*}
\begin{split}
    \Exp[r^s(x) | A^{[1]}_s=i, D_s=1]  &= \sum_{j=1}^{K}\Pr[ A_{s}^{[2]} = j| A^{[1]}_s=i, D_s=1 ] \Exp[r^s(x) | A^{[1]}_s=i, D_s=1, A_{s}^{[2]} = j]\\
    &=\sum_{j=1}^{K}\Pr[ A_{s}^{[2]} = j| A^{[1]}_s=i, D_s=1 ] \Exp[r^j(x) | A^{[1]}_s=i, D_s=1, A_{s}^{[2]} = j]
    \end{split}
\end{equation*}
where the last equality is by observing that given the second add-on service is $j$, the reward simply becomes $r^j(x)$.
Combining the above, we have  $\Exp[r^s(x)]=\sum_{i=1}^{K} \Pr[A^{[1]}_s=i] \Pr[D_s=0|A^{[1]}_s=i] \Exp[r^i(x)|D_s=0,A^{[1]}_s=i]        + \sum_{i,j=1}^{K} \Pr[A^{[1]}_s=i] \Pr[D_s=1|A^{[1]}_s=i]
      \Pr[A_s^{[2]}=j|D_s=1,A_s^{[1]}=i]\Exp[r^j(x)|D_s=1,A_s^{[1]}=i)]$, which is the desired property.
      
Similarly, we can expand the expected cost by law of total expectation  
\begin{equation*}
    \Exp[\eta^s(x)] = \sum_{i=1}^{K} \Pr[A_s^{[1]} = i ]\Exp[\eta^s(x)|A_s^{[1]} = i]
\end{equation*}
And we  can further expand the conditional expectation by 
\begin{equation*}
    \begin{split}
   & \Exp[\eta^s(x) | A^{[1]}_s=i] \\
= &  \Pr[D_s=0|A^{[1]}_s=i, ] \Exp[\eta^i(x) | A^{[1]}_s=i, D_s=0] +  \Pr[D_s=1|A^{[1]}_s=i, ] \Exp[\eta^s(x) | A^{[1]}_s=i, D_s=1]\\
= &  \Pr[D_s=0|A^{[1]}_s=i, ] \mathbf{c}_i +  \Pr[D_s=1|A^{[1]}_s=i, ] \Exp[\eta^s(x) | A^{[1]}_s=i, D_s=1]
    \end{split}
\end{equation*}
where the last equality is because when $D_s=0$, i.e., no add-on service is called, the strategy always uses the base service's prediction and incurs the base service's cost $\eta^s(x) = \mathbf{c}_i$.
For the second term, we can bring in $A_s^{[2]}$ and apply law of total expectation, to obtain
\begin{equation*}
\begin{split}
     \Exp[\eta^s(x) | A^{[1]}_s=i, D_s=1] & = \sum_{j=1}^{K}\Pr[ A_{s}^{[2]} = j| A^{[1]}_s=i, D_s=1 ] \Exp[\eta^s(x) | A^{[1]}_s=i, D_s=1, A_{s}^{[2]} = j]\\
& = \sum_{j=1}^{K}\Pr[ A_{s}^{[2]} = j| A^{[1]}_s=i, D_s=1 ] (\mathbf{c}_i+\mathbf{c}_j)     
    \end{split}
\end{equation*}  
where the last equality is because given the base service is $i$ and add-on service is $j$, the cost is simply the sum of their cost $\mathbf{c}_i+\mathbf{c}_j$.
Combining all the above equations, we have the expected cost 
$\Exp[\eta^s(x)]=\sum_{i=1}^{K} \Pr[A_s^{[1]}=i] \Pr[D_s=0|A_s^{[1]}=i] \mathbf{c}_i    + \sum_{i,j=1}^{K} \Pr[A_s^{[1]}=i] \Pr[D_s=1|A_s^{[1]}=i]
      \Pr[A_s^{[2]}=j|D_s=1,A_s^{[1]}=i]\left(\mathbf{c}_i+\mathbf{c}_j\right)$, which is the desired term.
Thus, we have shown a  form of expected accuracy and cost  which is exactly the same as in Lemma \ref{lemma:FAME:expectedacccostmainpaper}, which completes the proof.
\end{proof}

\subsection{Proof of Theorem \ref{thm:FAME:mainbound}}
\begin{proof}
To prove Theorem  \ref{thm:FAME:mainbound}, we need a few new definitions
and lemmas, which are stated below. 
\begin{definition}
Let $\hat{\Exp}[r^{s}(x)]$ and $\hat{\Exp}[\eta^{[s]}(x,\mathbf{c})]$ denote the empirically estimated accuracy and cost for the strategy $s$.
More precisely, let the empirically estimated accuracy be $ \hat{\Exp}[r^s(x)] \triangleq \sum_{i=1}^{K} \Pr[A^{[1]}_s=i] \hat{\Pr}[D_s=0|A^{[1]}_s=i] \hat{\Exp}[r^i(x)|D_s=0,A^{[1]}_s=i]        + \sum_{i,j=1}^{K} \Pr[A^{[1]}_s=i] \hat{\Pr}[D_s=0|A^{[1]}_s=i]
      \Pr[A_s^{[2]}=j|D_s=1,A_s^{[1]}=i]\hat{\Exp}[r^j(x)|D_s=1,A_s^{[1]}=i)]$, and the empirically estimated cost be $\hat{\Exp}[\eta^s(x)]=\sum_{i=1}^{K} \Pr[A_s^{[1]}=i] \hat{\Pr}[D_s=0|A_s^{[1]}=i] \mathbf{c}_i    + \sum_{i,j=1}^{K} \Pr[A_s^{[1]}=i] \hat{\Pr}[D_s=0|A_s^{[1]}=i]
      \Pr[A_s^{[2]}=j|D_s=1,A_s^{[1]}=i]\left(\mathbf{c}_i+\mathbf{c}_j\right)$.
\end{definition}

\begin{definition}
Let $s'=(\mathbf{p}^{[1]'}, \mathbf{Q}', \mathbf{P}^{[2]'})$ be the optimal solution to 
\begin{equation*}
    \begin{split}
        \max_{s\in S} &  \hat{\Exp}[r^s(x)] s.t. \hat{\Exp}[\eta^{[s]}(x,\mathbf{c})] \leq b
    \end{split}
\end{equation*}
\end{definition}

Note that $s^*$ is the optimal strategy, and $s'$ is the optimal strategy when the data distribution is unknown and estimated from $N$ samples, and $\hat{s}$ is the strategy we actually generate with finite computational complexity by Algorithm \ref{Alg:FAME:TrainingAlgorithm}.

The following lemma shows the computational complexity of Algorithm \ref{Alg:FAME:TrainingAlgorithm}. 
\begin{lemma}\label{lemma:FAME:ComputationalComplexity}
    The complexity of Algorithm \ref{Alg:FAME:TrainingAlgorithm} is $O\left(NMK^2+K^3M^3L+M^LK^2\right)$.
\end{lemma}
\begin{proof}
Estimating $\mathbf{A}$ requires a pass of all the training data, which gives a $O(NK)$ cost.
For each $k_1, k_2, \alpha_m$, we need a pass over training data for the $k_1$th and $k_2$th services to estimate $\psi_{k_1,k_2,\ell}(\alpha_M)$.
There are in total $K$ services, and thus this takes $O(NMK^2)$ computational cost. 

Algorithm \ref{Alg:FAME:Algorithm_fixbaseandlabel} has a complexity of $O(K^2)$.
Solving Problem \ref{prob:FAME:subproblem1} invokes Algorithm \ref{Alg:FAME:Algorithm_fixbaseandlabel} for each $\ell \in L$ and $m \in [M]$, and thus takes $O(K^2ML)$.
Computing $t^*_i$ takes ${M \choose L}$, which is $O(M^{L-1})$.
That is to say, solving the subproblem \ref{prob:FAME:subproblem1} once requires $O(K^2 ML + M^{L-1})$ computational cost.
Solving the master problem \ref{prob:FAME:master} requires invoking the subproblem $MK$ times, where $K$ times stands for each service, and $M$ stands for the linear interpolation.
Thus, the total computational cost for optimization process takes $O(K^3M^3L M^L K^2)$.
Combining this with the estimation cost $O(NMK^2)$ completes the proof.
\end{proof}

Next we evaluate how far the estimated accuracy and cost can be from the true expected accuracy and cost for each strategy, which is stated in Lemma \ref{lemma:FAME:InformationBound}.
\begin{lemma}\label{lemma:FAME:InformationBound}
With probability $1-\epsilon$, we have for all $s\in S$,
\begin{equation}
    \begin{split}
        \left|\hat{\Exp}[r^{s}(x)] - {\Exp}[r^{s}(x)] \right| & \leq O\left(\sqrt{\frac{\log \epsilon + \log M +\log K +\log L }{N}} + \frac{\gamma}{M} \right) \\
        \left| \hat{\Exp]}[\eta^{[s]}(x,\mathbf{c})] - \Exp[\eta^{[s]}(x,\mathbf{c})]   \right| & \leq O\left(\sqrt{\frac{\log \epsilon + \log K +\log L}{N}}\right)\\
        \end{split}
\end{equation}
and also
\begin{equation}
    \begin{split}
        \left|\Exp[r^{s'}(x)] - {\Exp}[r^{s^*}(x) ] \right| & \leq O\left(\sqrt{\frac{\log \epsilon + \log M +\log K +\log L }{N}}+\frac{\gamma}{M} \right) \\
        \left| \Exp[\eta^{[s']}(x,\mathbf{c})] - \Exp[\eta^{[s^*]}(x,\mathbf{c})]   \right| & \leq O\left(\sqrt{\frac{\log \epsilon + \log K +\log L}{N}}\right)
    \end{split}
\end{equation}
\end{lemma}
\begin{proof}
For each element in $\mathbf{A}$, we simply use a sample mean estimator.
Thus, by Chernoff bound, we have $|\hat{\mathbf{A}}_{i,\ell}-\mathbf{A}_{i,\ell}|\geq O( \sqrt{
\frac{\log \epsilon }{N}})$ w.p. at most $\epsilon$.
For each $\psi_{k_1, k_2, \ell}(\alpha_m)$, we again use a sample mean estimator for the true conditional expected accuracy.
We again apply the Chernoff bound, and obtain that for each of $k_1,k_2, \ell, \alpha_m$,  $|\psi_{k_1, k_2, \ell}(\alpha_m) - \hat{\psi}_{k_1, k_2, \ell}(\alpha_m)|\geq O( \sqrt{\frac{\log \epsilon}{N}})$ w.p. at most $\epsilon$.
Now applying the union bound, we have w.p. $1-\epsilon$,   $|\hat{\mathbf{A}}_{i,\ell}-\mathbf{A}_{i,\ell}| \leq O(\sqrt{\frac{\log \epsilon + \log K + \log L}{N}}) $ and $|\psi_{k_1,k_2,\ell}(\alpha_m)-\hat{\psi}_{k_1,k_2,\ell}(\alpha_m)| \leq O(\sqrt{\frac{\log \epsilon + \log M + \log K + \log L}{N}})$, for all $\ell, i, k_1, k_2, m$.

Recall that the function $\hat{\phi}_{k_1,k_2,\ell}(\cdot)$ is estimated by linear interpolation over $\alpha_1,\alpha_2,\cdots, \alpha_M$.
By assumption, $\phi_{k_1,k_2,\ell}(\cdot)$ is Lipschitz continuous, and $\alpha\in[0,1]$.
Now applying Lemma \ref{lemma:FAME:interpolationerrorbound}, we have that the estimated function $\hat{\phi}_{k_1,k_2,\ell}(\cdot)$ cannot be too far away from its true value, i.e., 
\begin{equation*}
    |\hat{\phi}_{k_1,k_2,\ell}(\alpha) - {\phi}_{k_1,k_2,\ell}(\alpha)| \leq  O(\sqrt{\frac{\log \epsilon + \log M + \log K + \log L}{N}} + \frac{\gamma}{M})
\end{equation*}
Recall the definition
$\hat{{\mathbf{r}}}^{a}_{k_1,K(\ell-1)+k_2}(\pmb \rho) \triangleq \psi_{k_1,k_2,\ell}(\pmb \rho_{k_1,\ell})$, $\hat{{\mathbf{r}}}^{b}_{k,\ell}(\pmb \rho) \triangleq \hat{\psi}_{k,k,\ell}(\pmb \rho_{k,\ell})$,   and 
$ \hat{\mathbf{r}}^{[-]}(\pmb \rho) \triangleq \hat{\mathbf{r}}^{a}(\pmb \rho)- \hat{\mathbf{r}}^{b}(\pmb \rho) \otimes \mathbf{1}^T_K$.
Then we know that for each element in those matrix function, its estimated value can be at most $O(\sqrt{\frac{\log \epsilon + \log M + \log K + \log L}{N}} + \frac{\gamma}{M})$ away from its true value.
Since the true accuracy is the (weighted) average over those functions, its estimated difference is also $O(\sqrt{\frac{\log \epsilon + \log M + \log K + \log L}{N}} + \frac{\gamma}{M})$.
The expected cost can be viewed as a (weighted) average over elements in the matrix $\mathbf{A}_{i,\ell}$, and thus the estimation difference is at most $O(\sqrt{\frac{\log \epsilon + \log K + \log L}{N}})$, which completes the proof.

\end{proof}

Then we need to bound how much error is incurred due to our computational approximation.
In other words, the difference between $s'$ and $\hat{s}$, which is given in Lemma \ref{lemma:FAME:computionbound}.

\begin{lemma}\label{lemma:FAME:computionbound}
    \begin{equation}
        \begin{split}
    \left|\hat{\Exp}[r^{\hat{s}}(x)] - \hat{\Exp}[r^{s'}(x)] \right| & \leq O\left(\frac{\gamma }{M}\right) \\
     \hat{\Exp}[\eta^{[\hat{s}]}(x,\mathbf{c})] - \hat{\Exp}[\eta^{[s']}(x,\mathbf{c})]    & \leq 0
        \end{split}
    \end{equation}
\end{lemma}
\begin{proof}
This lemma requires a few steps.
The first step is to show that the subroutine to solve subproblem gives a good approximation.
Then we can show that subroutine for solving the master problem gives a good approximation.
Finally combing those two, we can prove this lemma.

Let us start by  showing that the subroutine to solve subproblem gives a good approximation.

\begin{lemma}\label{lemma:FAME:subproblembound}
    For any $b'$, The subroutine for solving problem \ref{prob:FAME:subproblem1} produces a strategy $s(i,b') \triangleq (\mathbf{e}_i, \hat{\mathbf{Q}}_i(b'), \hat{\mathbf{P}^{[2]}}_i(b')) $ with the empirical accuracy $\hat{g}_i(b')$ s.t. the empirical accuracy is within $O(\frac{\gamma L}{M})$ from the optimal, i.e.,  $|\hat{g}_i(b') - g'_i(b')|\leq O(\frac{\gamma }{M})$, and the cost constraint is satisfied, i.e., $\hat{\Exp}[\gamma^{s(i,b')(x)}] \leq b'$ .
\end{lemma}
\begin{proof}
This requires two lemmas. 
\begin{lemma}\label{lemma:FAME:fixbaselabeloptimality}
    For any input, Algorithm \ref{Alg:FAME:Algorithm_fixbaseandlabel} gives the exact optimal solution and optimal value to problem \ref{prob:FAME:fixbaselabel}.
\end{lemma}
\begin{proof}
To prove this lemma, we simply note that the problem \ref{prob:FAME:fixbaselabel} also has a sparse structure, which is stated below.
\begin{lemma}\label{Lemma:FAME:SparseStructure}
For any constant $\eta$, function $\phi(\cdot): \R \mapsto \R^K$, and $\Omega_2=\{\rho, \pmb \Pi| 0\leq \rho \leq 1, \pmb \Pi^T \mathbf{1} = 1, \pmb\Pi \succcurlyeq 0 \}$. 
Suppose  the following optimization problem
 \begin{equation*}
    \begin{split}
        \max_{\rho, \pmb \Pi\in \Omega_2}\textit{ } & 
        \eta  + \rho \pmb \Pi^T \cdot \phi(\rho)        \\
        \textit{s.t.        } & \rho (\pmb \Pi - \pmb \Pi \odot \mathbf{e}_k)^T \mathbf{c}   \leq \beta,  
    \end{split}
\end{equation*}
is feasible. 
Then there exists one optimal solution 
	$(\rho^*,\pmb\Pi^*)$, such that $\pmb\Pi^*$ is sparse and $\|\pmb\Pi^*\|_0 \leq 2$. 	More specifically, one of the following must hold: 
	\begin{itemize}
		\item $\pmb \Pi^*_i = 1$ for some $i$,  and $\pmb \Pi^*_{k'}=0$, for all $k'\not=i$
    	\item $\pmb \Pi^*_i = \frac{\beta}{\rho \mathbf{c}_i}$ for some  $i$, $\pmb \Pi^*_k = 1- \frac{\beta}{\rho \mathbf{c}_i}$and $\pmb  \Pi^*_{k'}=0$, for all $k'\not\in\{i,k\}$
		\item $\pmb \Pi^*_i = \frac{\beta/\rho  -\mathbf{c}_j}{\mathbf{c}_i -\mathbf{c}_j}, \pmb \Pi^*_j=\frac{\mathbf{c}_i-\beta/\rho }{\mathbf{c}_i -\mathbf{c}_j}$, for some distinct $i,j$, and $\pmb\Pi^*_{k'}=0$, for all $k'\not=i,j$
	\end{itemize}
\end{lemma}

\begin{proof}
	Let $(p^*, \pmb\Pi')$ be one solution. Our goal is to show that there exists a solution $(p^*, \pmb\Pi^*)$ which satistfies the above conditions.
	
	(i) $p^*=0$: the optimal value does not depend on $\pmb\Pi'$, and thus any $(p^*, \pmb\Pi)$ is a solution. In particular, $(p^*, \pmb\Pi^*)$ is a solution where $\pmb\Pi^*$ satisfies the first condition in the statement.
	
	(ii) $p^*\not=0$: According to Lemma \ref{Lemma:FAME:LPSparseStructure}, the following linear optimization problem
	\begin{equation}\label{equ:sparseLP}
	\begin{split}
	\max_{\pmb \Pi\in\R^{K}} \textit{ } &    \sum_{i=1}^{K} \pmb \Pi_i \bar{r}_{i,p^*}   \\
	s.t. \textit{ } &  \sum_{i=1}^{K} \mathbf{c}_i \pmb \Pi_i \leq \frac{B}{p^*}, \sum_{i=1}^{K} \pmb \Pi_i \leq 1, \pmb \Pi_i\geq 0
	\end{split}
	\end{equation}
	has a solution $\pmb \Pi^*$ such that $\|\pmb \Pi^*\|_0\leq 2$.
	
	We first show that $(p^*,\pmb\Pi^*)$ is one optimal solution to the confidence score approach.
	By definition, it is clear that $(p^*,\pmb\Pi^*)$ is a feasible solution.
	All that is needed is to show the solution is optimal.
	Suppose not.
	We must have
	\begin{equation*}
	\begin{split}
	\bar{r}_0 + p^* \left[ \sum_{i=1}^{K} \pmb \Pi'_i \bar{r}_{i,p} -\bar{r}_{0,p} \right]  &> 	   \bar{r}_0 + p^* \left[ \sum_{i=1}^{K} \pmb \Pi^*_i \bar{r}_{i,p} -\bar{r}_{0,p} \right] \\
	\sum_{i=1}^{K} \pmb \Pi'_i \bar{r}_{i,p} &> 	    \sum_{i=1}^{K} \pmb \Pi^*_i \bar{r}_{i,p} \\
	\end{split}
	\end{equation*}
	But noting that $\pmb\Pi'$ by definition is also a feasible solution to the problem \ref{equ:sparseLP}, this inequality implies that the objective function achieved by $\pmb\Pi'$ is strictly larger than that achieved by one optimal solution to \ref{equ:sparseLP}.
	A contradiction.
	Hence, $(p^*,\pmb\Pi^*)$ is one optimal solution.

	Next we show that $\pmb\Pi^*$ must follow the presented form. 
	Since $\|\pmb \Pi^*\|_0\leq 2$, we can consider the cases separately.
	
	(i) $\|\pmb \Pi^*\|_0= 1$: Assume $\pmb\Pi^*_i\not=0$.
	Then problem \ref{equ:sparseLP} becomes
	\begin{equation*}
	\begin{split}
	\max_{\pmb \Pi_i\in\R^{+}} \textit{ } &    \pmb \Pi_i \bar{r}_{i,p^*}   \\
	s.t. \textit{ } &   \mathbf{c}_i \pmb \Pi_i \leq \frac{B}{p^*},  \pmb \Pi_i \leq 1
	\end{split}
	\end{equation*}
	Since the objective function is monotonely increasing w.r.t. $\pmb\Pi_i$, we must have $\pmb\Pi^*_i = \min\{\frac{B}{p^*\mathbf{c}},1\}$
	
	(ii) $\|\pmb \Pi^*\|_0= 2$: Assume $\pmb\Pi^*_i\not=0, \pmb\Pi^*_j\not=0$.
	Then problem \ref{equ:sparseLP} becomes
	\begin{equation*}
	\begin{split}
	\max_{\pmb \Pi_i\in\R^{+},\pmb \Pi_j\in\R^{+}} \textit{ } &    \pmb \Pi_i \bar{r}_{i,p^*} + \pmb \Pi_j \bar{r}_{j,p^*}   \\
	s.t. \textit{ } &   \mathbf{c}_i \pmb \Pi_i + \mathbf{c}_j \pmb \Pi_j\leq \frac{B}{p^*}, \pmb \Pi_i +  \pmb \Pi_j \leq 1
	\end{split}
	\end{equation*}
	As a linear programming, if it has a solution, then there must exist one solution on the corner point. Since $\pmb\Pi^*_i\not=0, \pmb\Pi^*_j\not=0$, the two constraints must be satisfied to achieve a corner point.
	The two constraints form a system of linear equations, and solving it gives $\pmb \Pi^*_i = \frac{B/p -\mathbf{c}_j}{\mathbf{c}_i -\mathbf{c}_j}, \pmb \Pi^*_j=\frac{\mathbf{c}_i-B/p }{\mathbf{c}_i -\mathbf{c}_j}$, which completes the proof.
\end{proof}	
Now we are ready to prove Lemma \ref{lemma:FAME:fixbaselabeloptimality}.
Recall that in Algorithm \ref{Alg:FAME:Algorithm_fixbaseandlabel},  we  compute $(\mu_1, i_1) = \arg \max_{\mu \in[0,1], i\in [K] }  \phi_i(\mu)$ and $(\mu_2, i_2, j_2) = \arg \max_{\mu \in[\frac{\beta}{\mathbf{c}_i},\min\{\frac{\beta}{\mathbf{c}_j},1\}], i,j\in [K], \mathbf{c}_i>\mathbf{c}_j }  \phi_{i,j}(\mu)$.
If $\phi_{i_1}(\mu_1) \geq \phi_{i_2,j_2}(\mu_2)$, let $\rho = \mu_1 $ and $\pmb \Pi = \left[ \mathbbm{1}_{\mu_1< \frac{\beta}{\mathbf{c}_{i_1}}} + \frac{\beta}{\mathbf{c}_i} \mathbbm{1}_{\mu_1\geq \frac{\beta}{\mathbf{c}_{i_1}}} \right]\mathbf{e}_{i_1}$.
Otherwise, let $\rho = \mu_2 $ and $\pmb \Pi =  \frac{\beta/\mu_2-\mathbf{c}_{j_2}}{\mathbf{c}_{j_2}-\mathbf{c}_{j_2}}\mathbf{e}_{i_2}  + \frac{\mathbf{c}_{i_2}-\beta/\mu_2}{\mathbf{c}_{i_2}-\mathbf{c}_{i_2}}\mathbf{e}_{j_2}$.
Recall that $\phi_{i}(\mu) \triangleq \bar{\mathbf{r}}_{k,\ell}(\mathbf{1}_{K\times L})  + \min\{\frac{\beta}{\mathbf{c}_i},\mu\} \tilde{\mathbf{r}}^{k,\ell}_{i}(\mu) $
and $\phi_{i,j}(\mu) \triangleq \bar{\mathbf{r}}_{k,\ell}(\mathbf{1}_{K\times L})  + \frac{\beta - \mu \mathbf{c}_j}{\mathbf{c}_i-\mathbf{c}_j} \tilde{\mathbf{r}}^{k,\ell}_{i}(\mu)  + \frac{\mu \mathbf{c}_i-\beta}{\mathbf{c}_i-\mathbf{c}_j} \tilde{\mathbf{r}}^{k,\ell}_{j}(\mu) $.

Let us consider the two cases separately.

(i): $\phi_{i_1}(\mu_1) \geq \phi_{i_2,j_2}(\mu_2)$, and thus $\rho = \mu_1 $ and $\pmb \Pi = \left[ \mathbbm{1}_{\mu_1< \frac{\beta}{\mathbf{c}_{i_1}}} + \frac{\beta}{\mathbf{c}_i} \mathbbm{1}_{\mu_1\geq \frac{\beta}{\mathbf{c}_{i_1}}} \right]\mathbf{e}_{i_1}$.

According to Lemma \ref{Lemma:FAME:SparseStructure}, there exists a solution $\check{\rho}^*, \check{\pmb{\Pi}}^*$ to the above problem, such that
	\begin{itemize}
		\item $\check{\pmb\Pi}^*_i = 1$ for some $i$,  and $\check{\pmb \Pi}^*_k=0$, for all $k\not=i$
		\item $\check{\pmb \Pi}^*_i = \frac{\beta}{\rho \mathbf{c}_i}$ for some  $i$, and $\check{\pmb  \Pi}^*_k=0$, for all $k\not=i$
		\item $\check{\pmb\Pi}^*_i = \frac{\beta/\rho  -\mathbf{c}_j}{\mathbf{c}_i -\mathbf{c}_j}, \check{\pmb \Pi}^*_j=\frac{\mathbf{c}_i-\beta/\rho }{\mathbf{c}_i -\mathbf{c}_j}$, for some distinct $i,j$, and $\check{\pmb\Pi}^*_k=0$, for all $k\not=i,j$
	\end{itemize}
If the first or second condition happens, the objective then becomes $\phi_{i}(\check{\rho}^*)$.
If the third condition happens, then the objective becomes $\phi_{i,j,}(\check{\rho}^*)$.
Since  $\phi_{i_1}(\mu_1) \geq \phi_{i_2,j_2}(\mu_2)$, we must have $\phi_{i}(\check{\rho}^*)\geq \phi_{i,j,}(\check{\rho}^*)$ and thus it must be either first or second condition.
By construction of $\mu_1$, we must have $\mu_1 = \hat{\rho}^*$.
If $\check{\rho}^* =  \mu_1 <\frac{\beta}{\mathbf{c}_{i_1}}$, i.e., $\frac{\beta}{\mathbf{c}_{i_1}\check{\rho}^*}>1$, and thus second case cannot happen, and it has to be the first case and thus $\check{\pmb \Pi}^*_{i_1}=1$.
By definition, we also have $\pmb \Pi = \mathbf{e}_{i_1}$.
And thus, we have $\check{\pmb \Pi}^*=\pmb\Pi$.
If $\check{\rho}^* =  \mu_1 \geq \frac{\beta}{\mathbf{c}_{i_1}}$, i.e., $\frac{\beta}{\mathbf{c}_{i_1}\check{\rho}^*}\leq 1$, then the second case must happen.
Thus, we must have $\check{\pmb\Pi}^*_{i_1} = \frac{\beta}{\rho\mathbf{c}}$.
Meanwhile, by definition, we have $\pmb \Pi = \frac{\beta}{\mathbf{c}_i \rho} \mathbf{e}_{i_1} = \check{\pmb\Pi}^*$.

(ii):  $\phi_{i_1}(\mu_1) \geq \phi_{i_2,j_2}(\mu_2)$, and thus $\rho = \mu_2 $ and $\pmb \Pi =  \frac{\beta/\mu_2-\mathbf{c}_{j_2}}{\mathbf{c}_{j_2}-\mathbf{c}_{j_2}}\mathbf{e}_{i_2}  + \frac{\mathbf{c}_{i_2}-\beta/\mu_2}{\mathbf{c}_{i_2}-\mathbf{c}_{i_2}}\mathbf{e}_{j_2}$.
We can use a similar argument to show that $\pmb\Pi = \check{\pmb \Pi}^*$.

That is to say, no matter which case we are in, the optimal solution is always returned.
\end{proof}

\begin{lemma}\label{lemma:FAME:hfunction_Lipschitz}
 The function $\hat{h}_{k,\ell}(\beta)$ is Lipschitz continuous with constant $O(\gamma)$ for $\beta \geq 0$.    
\end{lemma}
\begin{proof}
    Let us use $\phi_{k,\ell}()$ to denote $\hat{h}_{k,\ell}()$ for notation simplification. Consider $\beta$ and $\beta+\Delta$, and our goal is to bound $\phi_{k,\ell}(\beta+\Delta)-\phi_{k,\ell}(\beta)$.
    Let $\rho^{\beta+\Delta}, \pmb\Pi^{\beta+\Delta}$ be the  corresponding solution to $\beta+\Delta$, i.e., the solution to     
    \begin{equation*}
    \begin{split}
        \max_{\rho, \pmb \Pi\in \Omega_2}\textit{ } &  \bar{\mathbf{r}}_{k,\ell}(\mathbf{1}_{K\times L}) + \rho \pmb \Pi^T \cdot \tilde{\mathbf{r}}^{k,\ell}(\rho)        \\
        \textit{s.t.        } & \rho (\pmb \Pi - \pmb \Pi \odot \mathbf{e}_k)^T \mathbf{c}   \leq \beta+\Delta.  
    \end{split}
    \end{equation*}

Let $\rho' =  \frac{\beta}{\beta+\Delta} \rho^{\beta+\Delta}$.
It is clear that $\rho', \pmb \Pi^{\beta+\Delta}$ is one solution to 
    \begin{equation*}
    \begin{split}
        \max_{\rho, \pmb \Pi\in \Omega_2}\textit{ } &  \bar{\mathbf{r}}_{k,\ell}(\mathbf{1}_{K\times L}) + \rho \pmb \Pi^T \cdot \tilde{\mathbf{r}}^{k,\ell}(\rho)        \\
        \textit{s.t.        } & \rho (\pmb \Pi - \pmb \Pi \odot \mathbf{e}_k)^T \mathbf{c}   \leq \beta.  
    \end{split}
    \end{equation*}
Thus, $\bar{\mathbf{r}}_{k,\ell}(\mathbf{1}_{K\times L}) + \rho' \pmb \Pi^{(\beta+\Delta)T} \cdot \tilde{\mathbf{r}}^{k,\ell}(\rho')$ must be smaller or equal to $\phi_{k,\ell}(\beta)$, which is the optimal solution.
Thus we must have 
\begin{equation*}
\begin{split}
    \phi_{k,\ell}(\beta+\Delta)-\phi_{k,\ell}(\beta) & \leq      \phi_{k,\ell}(\beta+\Delta)- \bar{\mathbf{r}}_{k,\ell}(\mathbf{1}_{K\times L}) - \rho' \pmb \Pi^{(\beta+\Delta)T} \cdot \tilde{\mathbf{r}}^{k,\ell}(\rho')\\
\end{split}
\end{equation*}
Note that by definition, 
\begin{equation*}
\begin{split}
\phi_{k,\ell}(\beta+\Delta) = \bar{\mathbf{r}}_{k,\ell}(\mathbf{1}_{K\times L}) + \rho^{\beta+\Delta} \pmb \Pi^{(\beta+\Delta)T} \cdot \tilde{\mathbf{r}}^{k,\ell}(\rho^{\beta+\Delta})
\end{split}    
\end{equation*}
The above inequality becomes 
\begin{equation*}
\begin{split}
    \phi_{k,\ell}(\beta+\Delta)-\phi_{k,\ell}(\beta) & \leq  \rho^{\beta+\Delta} \pmb \Pi^{(\beta+\Delta)T} \cdot \tilde{\mathbf{r}}^{k,\ell}(\rho^{\beta+\Delta})    -  \rho' \pmb \Pi^{(\beta+\Delta)T} \cdot \tilde{\mathbf{r}}^{k,\ell}(\rho')\\
& =  \rho^{\beta+\Delta} \pmb \Pi^{(\beta+\Delta)T} \cdot \tilde{\mathbf{r}}^{k,\ell}(\rho^{\beta+\Delta})    - \rho' \pmb \Pi^{(\beta+\Delta)T} \cdot \tilde{\mathbf{r}}^{k,\ell}(\rho^{\beta+\Delta})\\ &+ \rho' \pmb \Pi^{(\beta+\Delta)T} \cdot \tilde{\mathbf{r}}^{k,\ell}(\rho^{\beta+\Delta}) - \rho' \pmb \Pi^{(\beta+\Delta)T} \cdot \tilde{\mathbf{r}}^{k,\ell}(\rho')\\ 
& =  \rho^{\beta+\Delta} \pmb \Pi^{(\beta+\Delta)T} \cdot \tilde{\mathbf{r}}^{k,\ell}(\rho^{\beta+\Delta})    - \frac{\beta}{\beta+\Delta}\rho^{\beta+\Delta} \pmb \Pi^{(\beta+\Delta)T} \cdot \tilde{\mathbf{r}}^{k,\ell}(\rho^{\beta+\Delta})\\ &+ \rho' \pmb \Pi^{(\beta+\Delta)T} \cdot \tilde{\mathbf{r}}^{k,\ell}(\rho^{\beta+\Delta}) - \rho' \pmb \Pi^{(\beta+\Delta)T} \cdot \tilde{\mathbf{r}}^{k,\ell}(\rho')\\
& =   \frac{\Delta}{\beta+\Delta}\rho^{\beta+\Delta} \pmb \Pi^{(\beta+\Delta)T} \cdot \tilde{\mathbf{r}}^{k,\ell}(\rho^{\beta+\Delta})\\ &+ \rho' \pmb \Pi^{(\beta+\Delta)T} \cdot \tilde{\mathbf{r}}^{k,\ell}(\rho^{\beta+\Delta}) - \rho' \pmb \Pi^{(\beta+\Delta)T} \cdot \tilde{\mathbf{r}}^{k,\ell}(\rho')\\
\end{split}
\end{equation*}
where the first equality is by adding and subtracting $\rho' \pmb \Pi^{(\beta+\Delta)T} \cdot \tilde{\mathbf{r}}^{k,\ell}(\rho^{\beta+\Delta})$, and the second equality is simply plugging in the value of $\rho'$.
According to Lemma \ref{Lemma:FAME:SparseStructure}, $\pmb \Pi^{\beta+\Delta}$ must be sparse.

(i) If $\pmb \Pi^{\beta+\Delta}_k=1$, then only the base service ($k$th service) is used when budget is $\beta+\Delta$
When the budget becomes smaller, i.e., becomes $\beta$, it is always possible to always use the base service.
Hence, we must have $ \phi_{k,\ell}(\beta+\Delta)-\phi_{k,\ell}(\beta) = 0$.

(ii) Otherwise, since $\|\pmb \Pi^{\beta+\Delta}\|\leq 2$, there are at most two elements in $\pmb \Pi^{\beta+\Delta}$ that are not zeros.
Let $k_1,k_2\not=k$ denote the indexes. 
Then the constraint gives 
\begin{equation*}
    \rho^{\beta+\Delta} \pmb \Pi^{\beta+\Delta}_{k_1} \mathbf{c}_{k_1} + \rho^{\beta+\Delta} \pmb \Pi^{\beta+\Delta}_{k_2} \mathbf{c}_{k_2} \leq \beta + \Delta
\end{equation*}
\begin{equation*}
    (\rho^{\beta+\Delta} \pmb \Pi^{\beta+\Delta}_{k_1} + \rho^{\beta+\Delta} \pmb \Pi^{\beta+\Delta}_{k_2}) \min_{j\not=k}\mathbf{c}_{j}\leq     \rho^{\beta+\Delta} \pmb \Pi^{\beta+\Delta}_{k_1} \mathbf{c}_{k_1} + \rho^{\beta+\Delta} \pmb \Pi^{\beta+\Delta}_{k_2} \mathbf{c}_{k_2} \leq \beta + \Delta
\end{equation*}
That is to say, 
\begin{equation*}
    \rho^{\beta+\Delta} (\pmb \Pi^{\beta+\Delta}_{k_1} +  \pmb \Pi^{\beta+\Delta}_{k_2}) \leq (\beta + \Delta) / (\min_{j\not=k}\mathbf{c}_{j})
\end{equation*}
Thus we have
\begin{equation*}
    \begin{split}
         \frac{\Delta}{\beta+\Delta}\rho^{\beta+\Delta} \pmb \Pi^{(\beta+\Delta)T} \cdot \tilde{\mathbf{r}}^{k,\ell}(\rho^{\beta+\Delta}) & = \frac{\Delta}{\beta+\Delta}\rho^{\beta+\Delta} (\pmb \Pi^{(\beta+\Delta)T}_{k_1} \cdot \tilde{\mathbf{r}}_{k_1}^{k,\ell}(\rho^{\beta+\Delta}) + \pmb \Pi^{(\beta+\Delta)T}_{k_1} \cdot \tilde{\mathbf{r}}_{k_1}^{k,\ell}(\rho^{\beta+\Delta})) \\
         & \leq \frac{\Delta}{\beta+\Delta}\rho^{\beta+\Delta} (\pmb \Pi^{(\beta+\Delta)T}_{k_1} + \pmb \Pi^{(\beta+\Delta)T}_{k_1} ) \\
         & \leq \frac{\Delta}{\beta+\Delta} (\beta+\Delta)/(\min_{j\not=k} \mathbf{c}_j) \\ 
         &= \frac{\Delta}{\min_{j\not= k}\mathbf{c}_j}
    \end{split}
\end{equation*}
In addition, note that by assumption, 
$\tilde{\mathbf{r}}^{k,\ell}(\rho)$ is Lipschitz continuous with constant $\gamma$.
Hence, we must have 
\begin{equation*}
    \begin{split}
     \tilde{\mathbf{r}}_{j}^{k,\ell}(\rho')\geq  \tilde{\mathbf{r}}_{j}^{k,\ell}(\rho^{\beta+\Delta})-\gamma |(\rho'-\rho^{\beta+\Delta}|) = \tilde{\mathbf{r}}_{j}^{k,\ell}(\rho^{\beta+\Delta})-\gamma \frac{\Delta}{\beta+\Delta} \rho^{\beta+\Delta}\\
    \end{split}
\end{equation*}
And thus
\begin{equation*}
    \begin{split}
        & \rho' \pmb \Pi^{(\beta+\Delta)T} \cdot \tilde{\mathbf{r}}^{k,\ell}(\rho^{\beta+\Delta}) - \rho' \pmb \Pi^{(\beta+\Delta)T} \cdot \tilde{\mathbf{r}}^{k,\ell}(\rho')\\ \leq    &     \rho' \pmb \Pi^{(\beta+\Delta)T} \cdot \tilde{\mathbf{r}}^{k,\ell}(\rho^{\beta+\Delta}) - \rho' \pmb \Pi^{(\beta+\Delta)T} \cdot \tilde{\mathbf{r}}^{k,\ell}(\rho^{\beta+\Delta}) + \rho' \pmb \Pi^{(\beta+\Delta)T} \cdot \mathbf{1} \gamma \frac{\Delta}{\beta+\Delta} \rho^{\beta+\Delta}\\
        = & \rho' (\pmb \Pi^{(\beta+\Delta)T}_{k_1} + \pmb \Pi^{(\beta+\Delta)T}_{k_2})   \gamma \frac{\Delta}{\beta+\Delta} \rho^{\beta+\Delta}\\
        = & \rho^{\beta+\Delta} \frac{\beta}{\beta+\Delta} (\pmb \Pi^{(\beta+\Delta)T}_{k_1} + \pmb \Pi^{(\beta+\Delta)T}_{k_2})   \gamma \frac{\Delta}{\beta+\Delta} \rho^{\beta+\Delta}        
    \end{split}
\end{equation*}
Now plugging in 
\begin{equation*}
    \rho^{\beta+\Delta} (\pmb \Pi^{\beta+\Delta}_{k_1} +  \pmb \Pi^{\beta+\Delta}_{k_2}) \leq (\beta + \Delta) / (\min_{j\not=k}\mathbf{c}_{j})
\end{equation*}
We can further have 
\begin{equation*}
    \begin{split}
        & \rho' \pmb \Pi^{(\beta+\Delta)T} \cdot \tilde{\mathbf{r}}^{k,\ell}(\rho^{\beta+\Delta}) - \rho' \pmb \Pi^{(\beta+\Delta)T} \cdot \tilde{\mathbf{r}}^{k,\ell}(\rho')\\ \leq & \rho^{\beta+\Delta} \frac{\beta}{\beta+\Delta} (\pmb \Pi^{(\beta+\Delta)T}_{k_1} + \pmb \Pi^{(\beta+\Delta)T}_{k_2})   \gamma \frac{\Delta}{\beta+\Delta} \rho^{\beta+\Delta}   \\
        \leq & \frac{\beta}{\beta+\Delta} (\beta+\Delta) /(\min_{j\not=k}\mathbf{c}_j) \gamma \frac{\Delta}{\beta+\Delta} \rho^{\beta+\Delta}  \\
        \leq &  \Delta \gamma /( \min_{j\not=k}\mathbf{c}_j)        
    \end{split}
\end{equation*}
Combining it with 
\begin{equation*}
    \begin{split}
         \frac{\Delta}{\beta+\Delta}\rho^{\beta+\Delta} \pmb \Pi^{(\beta+\Delta)T} \cdot \tilde{\mathbf{r}}^{k,\ell}(\rho^{\beta+\Delta}) 
         &\leq \frac{\Delta}{\min_{j\not= k}\mathbf{c}_j}
    \end{split}
\end{equation*}
we have
\begin{equation*}
\begin{split}
    \phi_{k,\ell}(\beta+\Delta)-\phi_{k,\ell}(\beta)
& \leq   \frac{\Delta}{\beta+\Delta}\rho^{\beta+\Delta} \pmb \Pi^{(\beta+\Delta)T} \cdot \tilde{\mathbf{r}}^{k,\ell}(\rho^{\beta+\Delta})\\ &+ \rho' \pmb \Pi^{(\beta+\Delta)T} \cdot \tilde{\mathbf{r}}^{k,\ell}(\rho^{\beta+\Delta}) - \rho' \pmb \Pi^{(\beta+\Delta)T} \cdot \tilde{\mathbf{r}}^{k,\ell}(\rho')\\
& \leq \Delta /( \min_{j\not=k}\mathbf{c}_j) +  \Delta \gamma /( \min_{j\not=k}\mathbf{c}_j)\\
&= \frac{\Delta(1+\gamma)}{\min_{j\not=k}\mathbf{c}_j}
\end{split}
\end{equation*}

Thus, no matter which case, we always have 
\begin{equation*}
\begin{split}
    \phi_{k,\ell}(\beta+\Delta)-\phi_{k,\ell}(\beta)
&\leq  \frac{1+\gamma}{\min_{j\not=k}\mathbf{c}_j} \cdot \Delta
\end{split}
\end{equation*}
In addition, since $\phi_{k,\ell}(\beta)$ must be monotone, we have
\begin{equation*}
\begin{split}
    \phi_{k,\ell}(\beta+\Delta)-\phi_{k,\ell}(\beta)
&\geq 0 \geq - \frac{1+\gamma}{\min_{j\not=k}\mathbf{c}_j} \cdot \Delta
\end{split}
\end{equation*}
That is to say, $\phi_{k,\ell}(\beta)$ is Lipschitz continuous with constant $\frac{1+\gamma}{\min_{j\not=k}\mathbf{c}_j}$, which finishes the proof.
\end{proof}

Now we are ready to prove Lemma \ref{lemma:FAME:subproblembound}.
By definition, there must exist a $\pmb \lambda'$, such that $g'_i(b') = \sum_{\ell=1}^{L} \hat{\mathbf{A}}_{i,\ell} \hat{h}_{k,\ell}(\pmb \lambda'_\ell (b'-\mathbf{c}_i))$.
Let $\Lambda_{M} = \{\pmb \lambda \in \R ^{L} |\pmb  \lambda \geq 0, \mathbf{1}^{T}\pmb \lambda =1 , \pmb \lambda_\ell M \in [M] \cup \{0\}\}$.
Then there must exists a $\hat{\pmb \lambda} \in \Lambda_M$ such that $|\pmb \lambda'_\ell - \hat{\pmb \lambda}_\ell| \leq \frac{1}{M}$.
By Lemma \ref{lemma:FAME:hfunction_Lipschitz}, we have $| \hat{h}_{k,\ell}(\pmb \lambda'_\ell(b'-\mathbf{c}_i))) - \hat{h}_{k,\ell}(\hat{\pmb \lambda}_\ell (b'-\mathbf{c}_i)) | \leq O(\frac{\gamma}{M})$.
Note that $\hat{\mathbf{A}}$ is empirical probability matrix, by construction, $\sum_{\ell}^{L} \hat{\mathbf{A}_{i,\ell}}=1$
and each $\mathbf{A}_{i,\ell}$ is non-negative.
Thus, we must have 
$|\sum_{\ell=1}^{L} \hat{\mathbf{A}}_{i,\ell} \hat{h}_{k,\ell}(\hat{\pmb \lambda}_\ell (b'-\mathbf{c}_i)) - g'_i(b') |= |\sum_{\ell=1}^{L} \hat{\mathbf{A}}_{i,\ell} \hat{h}_{k,\ell}(\hat{\pmb \lambda}_\ell (b'-\mathbf{c}_i)) - \sum_{\ell=1}^{L} \hat{\mathbf{A}}_{i,\ell} \hat{h}_{k,\ell}(\pmb \lambda'_\ell (b'-\mathbf{c}_i))| \leq O(\frac{\gamma}{M})$.
On the other hand, by construction, $\hat{g}_i(b')$ produced by the subroutine to solve problem \ref{prob:FAME:subproblem1} is $\sum_{\ell=1}^{L} \hat{\mathbf{A}}_{i,\ell} \hat{h}_{k,\ell}(\beta_{t_\ell^*}) = \max_{t_1,t_2,\cdots, t_L} \sum_{\ell=1}^{L} \hat{\mathbf{A}}_{i,\ell} \hat{h}_{k,\ell}(\beta_{t_\ell}) = \max_{t_1,t_2,\cdots, t_L} \sum_{\ell=1}^{L} \mathbf{A}_{i,\ell} \hat{h}_{k,\ell}(\frac{t_\ell^*}{M}(b'-\mathbf{c}_i)) \geq \sum_{\ell=1}^{L} \mathbf{A}_{i,\ell} \hat{h}_{k,\ell}(\hat{\pmb\lambda}_\ell (b'-\mathbf{c}_i)) $.
Combing this with $\sum_{\ell=1}^{L} \hat{\mathbf{A}}_{i,\ell} \hat{h}_{k,\ell}(\hat{\pmb \lambda}_\ell (b'-\mathbf{c}_i)) - g'_i(b') \geq - O(\frac{\gamma}{M})$,  we immediately obtain $\hat{g}_i(b') - g_i'(b') \geq -O(\frac{\gamma}{M})$.
Since by definition $g'_i(b')$ must be the optimal solution and thus we have $\hat{g}_i(b') - g_i'(b') \leq 0$.
Thus, we have $|\hat{g}_i(b') - g_i'(b')| \leq O(\frac{\gamma}{M})$.
By Lemma \ref{lemma:FAME:fixbaselabeloptimality}, the produced solution to problem \ref{prob:FAME:fixbaselabel} is exactly the optimal solution.
That is to say, for the generated solution $\hat{\rho}^{i,\ell}(\beta_{t^*_\ell}), \hat{\pmb \Pi}^{i,\ell}(\beta_{t^*_\ell})$, at most $\beta_{t_\ell^*}$ budget might be used.
Since the total budget is $\sum_{\ell=1}^{L}\beta_{t_\ell^*} = b'-\mathbf{c}_i$, at most $b'- \mathbf{c}_i$ budget might be used.
Calling the base service requires $\mathbf{c}_i$ cost, and thus the total cost is at most $b'$.
As a result, we must have $\hat{\Exp}[\gamma^{s(i,b')(x)}] \leq b'$, which completes the proof.
\end{proof}

\begin{lemma}\label{lemma:FAME:masterbound}
    $|\hat{g}_i^{LI}(b') - {g}'_i(b')|\leq O(\frac{\gamma L }{M})$ for all $b'$ and $i$.    
\end{lemma}
\begin{proof}
Let us consider three cases separately.

Case 1: $b' \leq \mathbf{c}_i$.  By definition, $g_i'(b')=0$. By construction, $\hat{g}_i^{LI}(b')=0$, and thus $|\hat{g}_i^{LI}(b') - {g}'_i(b')|\leq O(\frac{\gamma }{M})$.

Case 2: $\theta_{m+1} \geq b'  \geq \theta_m \geq \mathbf{c}_i$.

We first note that $g_i'(b')$ by definition, is 
\begin{equation*}
    \begin{split}
        \max_{s=(\mathbf{e}_1),\mathbf{Q},\mathbf{P}\in S} & \hat{\Exp}[r^s(x) | A^{[1]}_s=i]\\
        \textit{ s.t. } & \hat{\Exp}[\eta^s(x)] \leq b'
    \end{split}
\end{equation*}
Abusing the notation a little bit, let us use $\Exp$ to  denote $\hat{\Exp}$ for simplicity (as well as $\Pr$ for $\hat{\Pr}$). We can expand the objective function by 
\begin{equation*}
    \begin{split}
   & \Exp[r^s(x) | A^{[1]}_s=i] \\
   =& \sum_{\ell=1}^{L}   \Pr[y_i(x)=\ell] \Exp[r^s(x) | A^{[1]}_s=i, y_i(x)=\ell]\\
  =& \sum_{\ell=1}^{L}   \Pr[D_s=0|A^{[1]}_s=i, y_i(x)=\ell] \Pr[y_i(x)=\ell] \Exp[r^s(x) | A^{[1]}_s=i, y_i(x)=\ell,D_s=0] + \\
 & \sum_{\ell=1}^{L}   \Pr[D_s=1|A^{[1]}_s=i, y_i(x)=\ell] \Pr[y_i(x)=\ell] \Exp[r^s(x) | A^{[1]}_s=i, y_i(x)=\ell,D_s=1]  \\
  =& \sum_{\ell=1}^{L}   \Pr[D_s=0|A^{[1]}_s=i, y_i(x)=\ell] \Pr[y_i(x)=\ell] \Exp[r^i(x) | A^{[1]}_s=i, y_i(x)=\ell,D_s=0] + \\
 & \sum_{\ell=1}^{L}   \Pr[D_s=1|A^{[1]}_s=i, y_i(x)=\ell] \Pr[y_i(x)=\ell] \Pr[A^{[2]}_s=j| A^{[1]}_s=i, y_i(x)=\ell,D_s=1 ]\cdot \\
 & \Exp[r^s(x) | A^{[1]}_s=i, y_i(x)=\ell,D_s=1, A^{[2]}_s=j]  \\
  =& \sum_{\ell=1}^{L}   \Pr[D_s=0|A^{[1]}_s=i, y_i(x)=\ell] \Pr[y_i(x)=\ell] \Exp[r^i(x) | A^{[1]}_s=i, y_i(x)=\ell,D_s=0] + \\
 & \sum_{\ell=1}^{L}   \Pr[D_s=1|A^{[1]}_s=i, y_i(x)=\ell] \Pr[y_i(x)=\ell] \mathbf{P}_{i,\ell,j}\cdot \\
 & \Exp[r^s(x) | A^{[1]}_s=i, y_i(x)=\ell,D_s=1, A^{[2]}_s=j]  \\ 
    \end{split}
\end{equation*}
where all qualities are simply by applying the conditional expectation formula. 
That is to say, conditional on the quality score $\mathbf{Q}$, the objective function is a linear function over $\mathbf{P}$ where all coefficients are positive.
Similarly, we can expand the budget constraint by
\begin{equation*}
    \begin{split}
   & \Exp[\eta^s(x) | A^{[1]}_s=i] \\
   =& \sum_{\ell=1}^{L}   \Pr[y_i(x)=\ell] \Exp[\eta^s(x) | A^{[1]}_s=i, y_i(x)=\ell]\\
  =& \sum_{\ell=1}^{L}   \Pr[D_s=0|A^{[1]}_s=i, y_i(x)=\ell] \Pr[y_i(x)=\ell] \Exp[\eta^s(x) | A^{[1]}_s=i, y_i(x)=\ell,D_s=0] + \\
 & \sum_{\ell=1}^{L}   \Pr[D_s=1|A^{[1]}_s=i, y_i(x)=\ell] \Pr[y_i(x)=\ell] \Exp[\eta^s(x) | A^{[1]}_s=i, y_i(x)=\ell,D_s=1]  \\
  =& \sum_{\ell=1}^{L}   \Pr[D_s=0|A^{[1]}_s=i, y_i(x)=\ell] \Pr[y_i(x)=\ell] \Exp[\eta^i(x) | A^{[1]}_s=i, y_i(x)=\ell,D_s=0] + \\
 & \sum_{\ell=1}^{L} \sum_{j=1}^{K}  \Pr[D_s=1|A^{[1]}_s=i, y_i(x)=\ell] \Pr[y_i(x)=\ell] \Pr[A^{[2]}_s=j| A^{[1]}_s=i, y_i(x)=\ell,D_s=1 ]\cdot \\
 & \Exp[\eta^s(x) | A^{[1]}_s=i, y_i(x)=\ell,D_s=1, A^{[2]}_s=j]  \\
  =& \sum_{\ell=1}^{L}  \Pr[D_s=0|A^{[1]}_s=i, y_i(x)=\ell] \Pr[y_i(x)=\ell] \mathbf{c}_i+\\
  &
 \sum_{\ell=1}^{L} \sum_{j=1}^{K}  \Pr[D_s=1|A^{[1]}_s=i, y_i(x)=\ell] \Pr[y_i(x)=\ell] \mathbf{P}_{i,\ell,j} \mathbf{c}_j 
    \end{split}
\end{equation*}
which is also linear in $\mathbf{P}$ conditional on $\mathbf{Q}$.
Let $(\mathbf{e}_i, \mathbf{Q}^*(b'), \mathbf{P}^*(b'))$ be the optimal solution that leads to $g_i'(b')$.
Now let us consider

\begin{equation*}
    \begin{split}
        \max_{s=(\mathbf{e}_1,\mathbf{Q}^*(b'),\mathbf{P})\in S} & \hat{\Exp}[r^s(x) | A^{[1]}_s=i]\\
        \textit{ s.t. } & \hat{\Exp}[\eta^s(x)] \leq b''
    \end{split}
\end{equation*}
which is a linear programming over $\mathbf{P}$ which satisfies all conditions in Lemma \ref{Lemma:FAME:LPContinousProperty}.
Let us denote its optimal value by $g^{'b'}_i(b'')$.
When $b''=b'$, its optimal value must be $g_i'(b')$, i,e., $g_i'('b') = g_i^{'b'}(b')$.
By Lemma \ref{Lemma:FAME:LPContinousProperty},  we have $g^{b'}_i(\cdot)$ is Lipschitz continuous.
In other words, we have $|g^{'b'}_i(b^1) - g^{'b'}_i(b^2)| \leq O(b_1-b_2)$ for any $b_1,b_2,b'$.
On the other hand, note that the optimal solution corresponding to $g_i^{'b'}(b'')$ is also a feasible solution to the original optimization without fixing $\mathbf{Q} = \mathbf{Q}^*(b')$. Hence, for any $b''$, we must have $g_i'(b') \geq g_i^{'b''}(b')$.
Thus, for any $b_1 \leq b_2$, we have 
\begin{equation*}
    \begin{split}
        g_i'(b^2) - g_i'(b^1)& = g_i'(b^2)-g_i^{b^2}(b^2) + g_i^{b^2}(b^2) - g_i^{b^2}(b^1)+ g_i^{b^2}(b^1)- g_i'(b^1)\\
        & = g_i^{b^2}(b^2) - g_i^{b^2}(b^1)+ g_i^{b^2}(b^1)- g_i'(b^1)\\
        & \leq  O(b_1-b_2)\\        
    \end{split}
\end{equation*}
In addition,by definition $g_i'(b^2) - g_i'(b^1) \geq 0$.
Hence, we have just shown that $|g_i'(b^2) - g_i'(b^1)| \leq  O(b_1-b_2)$, which implies $g'(b')$ is a Lipschitz continuous function. 
Lemma \ref{lemma:FAME:subproblembound} implies that $|\hat{g}_{\theta_m} - g'(\theta_m)|\leq O(\frac{\gamma}{M})$ for every $m$. Now by Lemma \ref{lemma:FAME:interpolationerrorbound}, we obtain that 
\begin{equation*}
    \begin{split}
        |\hat{g}^{LI}(b') - g'(b')| \leq O(\frac{\gamma}{M}) + O(\frac{1}{M})
    \end{split}
\end{equation*}

Case 3: $\theta_{m} \geq b' \geq  \mathbf{c}_i \geq \theta_{m-1}$. Exactly the same argument from case 2 can be applied, while noting that we use $\mathbf{c}_i$ as the interpolation point. 

Thus, we have just proved that on three separate intervals, we have $ |\hat{g}^{LI}(b') - g'(b')| \leq O(\frac{\gamma}{M}) + O(\frac{1}{M})$.
Therefore, for any $b',i$, we must have $ |\hat{g}^{LI}(b') - g'(b')| \leq O(\frac{\gamma}{M}) + O(\frac{1}{M})$, which completes the proof.
\end{proof}

Now we are ready to prove Lemma \ref{lemma:FAME:computionbound}.

Note that by definition, $s'  \triangleq (\mathbf{p}^{[1]'}, \mathbf{Q}', \mathbf{P}^{[2]'})$ is the optimal solution to the empirical accuracy and cost joint optimization problem.
By Lemma \ref{lemma:FAME:sparsitymainpaper}, $\mathbf{p}^{[1]'}$ should also be 2-sparse.
Let $i_1'$ and $i_2'$ be the corresponding indexes of the nonzero components, $p_1', p_2'$ are the probability of using them as the base service, and $b_1', b_2'$ be the budget allocated to them in strategy $s'$.
Then this must be the optimal solution to the master problem 
\begin{equation}
   \max_{(i_1, i_2, p_1, p_2, b_1, b_2) \in \mathit{C} }\textit{ }  p_1 g_{i_1}'(b_1/p_1) + p_2 g_{i_2}'(b_2/p_2)  \textit{ } s.t. b_1+b_2\leq b
\end{equation}
On the other hand, due to the linear interpolation, the subroutine to solve master problem \ref{prob:FAME:master} in Algorithm \ref{Alg:FAME:TrainingAlgorithm} is effectively solving 
\begin{equation}
  \max_{(i_1, i_2, p_1, p_2, b_1, b_2) \in \mathit{C} }\textit{ }  p_1 g_{i_1}^{LI}(b_1/p_1) + p_2 g_{i_2}^{LI}(b_2/p_2)  \textit{ } s.t. b_1+b_2\leq b
\end{equation}
and returns its optimal solution $\hat{i}_1, \hat{i}_2, \hat{p}_1, \hat{p}_2, \hat{b}_1, \hat{b}_2$.
By Lemma \ref{lemma:FAME:masterbound}, we have $ |g_{i}'(b') - g_{i}^{LI}(b')|  \leq O(\frac{\gamma L }{M}) $ for all $b'$ and $i$, and thus for any $i_1, i_2, b_1, b_2, p_1, p_2 \in C$, we must have 
$|p_1 g_{i_1}'(b_1/p_1) + p_2 g_{i_2}'(b_2/p_2) - (p_1 g_{i_1}^{LI}(b_1/p_1) + p_2 g_{i_2}^{LI}(b_2/p_2))|\leq O(\frac{\gamma L}{M}) $, since $p_1+p_2=1$.
Note that the constraints of the above two optimization are the same.
Now we can apply Lemma 
\ref{Lemma:FAME:OptimiziationDiff}, and obtain
\begin{equation}
      \hat{p}_1 g_{\hat{i}_1}'(\hat{b}_1/\hat{p}_1) + \hat{p}_2 g_{\hat{i}_2}'(\hat{b}_2/\hat{p}_2) \geq p_1' g_{i_1'}'(b_1'/p_1') + p_2' g_{i_2'}'(b_2'/p_2') -O(\frac{\gamma L}{M})
\end{equation}
By definition, we have $\Exp[r^{s'}(x)] = p_1' g_{i_1'}'(b_1'/p_1') + p_2' g_{i_2'}'(b_2'/p_2')$, and thus the above simply becomes 
\begin{equation}\label{equ:FAME:temp0}
      \hat{p}_1 g_{\hat{i}_1}'(\hat{b}_1/\hat{p}_1) + \hat{p}_2 g_{\hat{i}_2}'(\hat{b}_2/\hat{p}_2) \geq \Exp[r^{s'}(x)] -O(\frac{\gamma L}{M})
\end{equation}
Next note that the final strategy is produced by calling subproblem \ref{prob:FAME:subproblem1} solver for $b'=\hat{b}_j/\hat{p}_j$ and $i=\hat{i}_j$, where $j=1,2$, and then aligning those two solutions.
Thus, the empirical accuracy is simply $\Exp[r^{\hat{s}}(x)] = \hat{p}_1 \hat{g}_{\hat{i}_1}(\hat{b}_1/\hat{p}_1) + \hat{p}_2 \hat{g}_{\hat{i}_2}(\hat{b}_2/\hat{p}_2)$. 
By Lemma \ref{lemma:FAME:subproblembound}, we have 
\begin{equation*}
    |\hat{p}_1 \hat{g}_{\hat{i}_1}(\hat{b}_1/\hat{p}_1) - \hat{p}_1 {g}'_{\hat{i}_1}(\hat{b}_1/\hat{p}_1)| \leq O(\frac{\gamma L }{M})
    \end{equation*}
and
\begin{equation*}
    |\hat{p}_2 \hat{g}_{\hat{i}_2}(\hat{b}_2/\hat{p}_2) - \hat{p}_2 {g}'_{\hat{i}_2}(\hat{b}_1/\hat{p}_2)| \leq O(\frac{\gamma L }{M})
    \end{equation*}
Adding those two terms we have
\begin{equation*}
    \hat{p}_1 \hat{g}_{\hat{i}_1}(\hat{b}_1/\hat{p}_1) - \hat{p}_1 {g}'_{\hat{i}_1}(\hat{b}_1/\hat{p}_1)
    + 
    \hat{p}_2 \hat{g}_{\hat{i}_2}(\hat{b}_2/\hat{p}_2) - \hat{p}_2 {g}'_{\hat{i}_2}(\hat{b}_1/\hat{p}_2) \geq  - O(\frac{\gamma L }{M})
    \end{equation*}
That is to say,
\begin{equation*}
\Exp[r^{\hat{s}}(x)] -  
     \hat{p}_1 {g}'_{\hat{i}_1}(\hat{b}_1/\hat{p}_1)
     - \hat{p}_2 {g}'_{\hat{i}_2}(\hat{b}_1/\hat{p}_2) \geq  - O(\frac{\gamma L }{M})
    \end{equation*}
Adding the inequality \ref{equ:FAME:temp0}, we have 
\begin{equation*}
\Exp[r^{\hat{s}}(x)] -\Exp[r^{s'}(x)]  
      \geq  - O(\frac{\gamma L }{M})
    \end{equation*}
which completes the proof.
\end{proof}

Now let us  prove Theorem \ref{thm:FAME:mainbound_weak}, a slightly weaker version of 
Theorem \ref{thm:FAME:mainbound}.
\begin{theorem}\label{thm:FAME:mainbound_weak}
Suppose $\Exp[r_i(x)|D_s=0,A_s^{[1]}=i]$ is Lipschitz continuous with constant $\gamma$ w.r.t. each element in $\mathbf{Q}$.
Given $N$ i.i.d. samples  $\{y(x_i), \{(y_k(x_i), q_k(x_i))\}_{k=1}^{K}\}_{i=1}^{N}$, the computational cost of Algorithm \ref{Alg:FAME:TrainingAlgorithm} is $O\left(NMK^2+K^3M^3L+M^LK^2\right)$. 
With probability $1-\epsilon$, the produced strategy $\hat{s}$ satisfies  
 $\Exp[r^{\hat{s}}(x)]-\Exp[r^{s^*}(x)]\geq  - O\left(\sqrt{\frac{\log \epsilon + \log M +\log K +\log L }{N}} + \frac{\gamma L}{M}\right)$, and 
$\Exp[\gamma^{[\hat{s}]}(x,\mathbf{c})]  \leq b+O\left(\sqrt{\frac{\log \epsilon + \log M +\log K +\log L }{N}} \right)$.
\end{theorem}

\begin{proof}
There are three main parts: (i) the computational complexity, (ii) the accuracy drop, and (iii) the excessive cost. Let us handle them sequentially. 

(i) Computational Complexity: Lemma \ref{lemma:FAME:ComputationalComplexity} directly gives the computational complexity bound.

(ii) Accuracy Loss: By Lemma \ref{lemma:FAME:InformationBound}, with probability $1-\epsilon$, we have 
\begin{equation*}
    \begin{split}
        \Exp[\mathbbm{1}_{\hat{y}^{\hat{s}}(x) = y(x)}] - \hat{\Exp}[\mathbbm{1}_{\hat{y}^{\hat{s}}(x) = y(x)}]  & \geq - O\left(\sqrt{\frac{\log \epsilon + \log M +\log K +\log L }{N}} +\frac{\gamma}{M}\right) \\
        \hat{\Exp}[\mathbbm{1}_{\hat{y}^{s'}(x) = y(x)}] - \Exp[\mathbbm{1}_{\hat{y}^{s'}(x) = y(x)}]  & \geq - O\left(\sqrt{\frac{\log \epsilon + \log M +\log K +\log L }{N}} \right)
        \end{split}
\end{equation*}
and also
\begin{equation*}
    \begin{split}
        \Exp[\mathbbm{1}_{\hat{y}^{s'}(x) = y(x)}] - {\Exp}[\mathbbm{1}_{\hat{y}^{s^*}(x) = y(x)}]  & \geq -O\left(\sqrt{\frac{\log \epsilon + \log M +\log K +\log L }{N} + +\frac{\gamma}{M}} \right)
    \end{split}
\end{equation*}
By Lemma \ref{lemma:FAME:computionbound}, we have 
    \begin{equation*}
        \begin{split}
    \hat{\Exp}[\mathbbm{1}_{\hat{y}^{\hat{s}}(x) = y(x)}] - \hat{\Exp}[\mathbbm{1}_{\hat{y}^{s'}(x) = y(x)}] & \geq - O\left(\frac{\gamma }{M}\right)
        \end{split}
    \end{equation*}
Combining those four inequalities, we have     
\begin{equation*}
    \begin{split}
        & \Exp[\mathbbm{1}_{\hat{y}^{\hat{s}}(x)=y(x)}]-\Exp[\mathbbm{1}_{\hat{y}^{s^*}(x)=y(x)}] \\
         = & \Exp[\mathbbm{1}_{\hat{y}^{\hat{s}}(x)=y(x)}]-\hat{\Exp}[\mathbbm{1}_{\hat{y}^{\hat{s}}(x)=y(x)}] + \hat{\Exp}[\mathbbm{1}_{\hat{y}^{\hat{s}}(x)=y(x)}]-\hat{\Exp}[\mathbbm{1}_{\hat{y}^{s'}(x)=y(x)}] \\
         + & \hat{\Exp}[\mathbbm{1}_{\hat{y}^{s'}(x)=y(x)}]-\Exp[\mathbbm{1}_{\hat{y}^{s'}(x)=y(x)}] + \Exp[\mathbbm{1}_{\hat{y}^{s'}(x)=y(x)}]-\Exp[\mathbbm{1}_{\hat{y}^{s^*}(x)=y(x)}]\\
         \geq & - O\left(\sqrt{\frac{\log \epsilon + \log M +\log K +\log L }{N}} \right) -O\left(\frac{\gamma L}{M}\right) \\
         -& O\left(\sqrt{\frac{\log \epsilon + \log M +\log K +\log L }{N}} \right)
         - O\left(\sqrt{\frac{\log \epsilon + \log M +\log K +\log L }{N}} \right)\\
         \geq & - O\left(\sqrt{\frac{\log \epsilon + \log M +\log K +\log L }{N}} + \frac{\gamma L}{M}\right)
    \end{split}
\end{equation*}

(iii) Excessive Cost: Similar to (ii), by Lemma \ref{lemma:FAME:InformationBound}, with probability $1-\epsilon$, we have 
\begin{equation*}
    \begin{split}
        \Exp[\tau^{[\hat{s}]}(x,\mathbf{c})] - \hat{\Exp}[\tau^{[\hat{s}]}(x,\mathbf{c})] & \leq  O\left(\sqrt{\frac{\log \epsilon  +\log K +\log L }{N}} \right) \\
       \hat{\Exp}[\tau^{[s']}(x,\mathbf{c})] - \Exp[\tau^{[s']}(x,\mathbf{c})] & \leq  O\left(\sqrt{\frac{\log \epsilon  +\log K +\log L }{N}} \right)
        \end{split}
\end{equation*}
and also
\begin{equation*}
    \begin{split}
        \Exp[\tau^{[s']}(x,\mathbf{c})] - \Exp[\tau^{[s^*]}(x,\mathbf{c})]  & \leq O\left(\sqrt{\frac{\log \epsilon  +\log K +\log L }{N}} \right)
    \end{split}
\end{equation*}
By Lemma \ref{lemma:FAME:computionbound}, we have 
    \begin{equation*}
        \begin{split}
         \hat{\Exp}[\tau^{[\hat{s}]}(x,\mathbf{c})] - \hat{\Exp}[\tau^{[s']}(x,\mathbf{c})]    & \leq 0
        \end{split}
    \end{equation*}
Combining those four inequalities, we have     
\begin{equation*}
    \begin{split}
        & \Exp[\tau^{[\hat{s}]}(x,\mathbf{c})] - \Exp[\tau^{[s^*]}(x,\mathbf{c})] \\
         = & \Exp[\tau^{[\hat{s}]}(x,\mathbf{c})] - \hat{\Exp}[\tau^{[\hat{s}]}(x,\mathbf{c})] + \hat{\Exp}[\tau^{[\hat{s}]}(x,\mathbf{c})] - \hat{\Exp}[\tau^{[s']}(x,\mathbf{c})] \\
         + & \hat{\Exp}[\tau^{[s']}(x,\mathbf{c})] - \Exp[\tau^{[s']}(x,\mathbf{c})] + \Exp[\tau^{[s']}(x,\mathbf{c})] - \Exp[\tau^{[s^*]}(x,\mathbf{c})]\\
         \leq &  O\left(\sqrt{\frac{\log \epsilon + \log K +\log L }{N}} \right) +0 \\
         +& O\left(\sqrt{\frac{\log \epsilon +\log K +\log L }{N}} \right)
         + O\left(\sqrt{\frac{\log \epsilon + \log K +\log L }{N}} \right)\\
         \leq &  O\left(\sqrt{\frac{\log \epsilon + \log K +\log L }{N}} \right)
    \end{split}
\end{equation*}
which completes the proof.
\end{proof}

\begin{lemma}\label{lemma:FAME:deltabudget}
    Let $s^{\Delta b}\triangleq \arg \max_{s\in S} \Exp[r^s(x)] \textit{ s.t. } \Exp[\gamma^{[s]}(x)] \leq b-\Delta b$. If $b - \Delta b \geq 0$, then $\Exp[r^{s^{\Delta b}}(x)] - \Exp[r^{s^*}(x)] \geq - O(\Delta b)$.
\end{lemma}
\begin{proof}
We simply need to show that $\Exp[r^{s^{\Delta b}}(x)]$ is Lipschitz continuous in $\Delta b$.
To see this, let us expand $s^{\Delta b} = (\mathbf{p}^{\Delta b}, \mathbf{Q}^{\Delta b}, \mathbf{P}^{\Delta b})$ and consider the following optimization problem 
    \begin{equation}\label{prob:FAMEtemp6}
        \begin{split}
        \max_{s=(\mathbf{p}^{0}, \mathbf{Q}^{\Delta b},\mathbf{P})\in S} \Exp[r^s(x)] \textit{ s.t. } \Exp[\gamma^{[s]}(x)] \leq a.     
        \end{split}
    \end{equation}

By law of total expectation, we have
\begin{equation*}
    \Exp[r^s(x)] = \sum_{i=1}^{K} \Pr[A_s^{[1]} = i ]\Exp[r^s(x)|A_s^{[1]} = i]
\end{equation*}
And we  can further expand the conditional expectation by 
\begin{equation*}
    \begin{split}
   & \Exp[r^s(x) | A^{[1]}_s=i] \\
   =& \sum_{\ell=1}^{L}   \Pr[y_i(x)=\ell] \Exp[r^s(x) | A^{[1]}_s=i, y_i(x)=\ell]\\
  =& \sum_{\ell=1}^{L}   \Pr[D_s=0|A^{[1]}_s=i, y_i(x)=\ell] \Pr[y_i(x)=\ell] \Exp[r^s(x) | A^{[1]}_s=i, y_i(x)=\ell,D_s=0] + \\
 & \sum_{\ell=1}^{L}   \Pr[D_s=1|A^{[1]}_s=i, y_i(x)=\ell] \Pr[y_i(x)=\ell] \Exp[r^s(x) | A^{[1]}_s=i, y_i(x)=\ell,D_s=1]  \\
    \end{split}
\end{equation*}
Note that
\begin{equation*}
\begin{split}
    & \Exp[r^s(x) | A^{[1]}_s=i, y_i(x)=\ell,D_s=0]  = \Exp[r^i(x) | A^{[1]}_s=i, y_i(x)=\ell,D_s=0]
    \end{split}
\end{equation*}
and 
\begin{equation*}
    \begin{split}
   &  \Exp[r^s(x) | A^{[1]}_s=i, y_i(x)=\ell,D_s=1] \\
   =& \sum_{j=1}^{K} \Pr[A^{[2]}_s=j|A^{[1]}_s=i, y_i(x)=\ell,D_s=1] \Exp[r^s(x) | A^{[1]}_s=i, y_i(x)=\ell,D_s=1, A^{[2]}_s=j] \\
= &\sum_{j=1}^{K} \mathbf{Q}_{i,\ell,j}\Exp[r^s(x) | A^{[1]}_s=i, y_i(x)=\ell,D_s=1, A^{[2]}_s=j]   
    \end{split}
\end{equation*}
where all qualities are simply by applying the conditional expectation formula. 
That is to say, conditional on the quality score $\mathbf{Q}$, the objective function is a linear function over $\mathbf{P}$ where all coefficients are positive.
Similarly, we can expand the budget constraint, which turns out to be also linear in $\mathbf{P}$ conditional on $\mathbf{Q}$.

Thus, problem \ref{prob:FAMEtemp6} is a linear programming in $\mathbf{P}$.
Therefore, its optimal value, denoted by $F(a|\Delta_b)$, must also be Lipschitz continuous in $a$, according to Lemma \ref{Lemma:FAME:LPContinousProperty}. In other words, we have $|F(a_2|\Delta_b) - F(a_1|\Delta_b)| \leq O(|a_1-a_2|)$ for any $a_1,a_2,\Delta b$.
When $a = b-\Delta b$, its optimal value must be $\Exp[r^{s^{\Delta b}}(x)]$, i,e., $\Exp[r^{s^{\Delta b}}(x)] = F(b-\Delta b|\Delta_b)$.
On the other hand, note that the optimal solution corresponding to $ F(b-\Delta_b |0)$ is also a feasible solution to the original optimization without fixing $\mathbf{p}^ = \mathbf{p}^{0 }, \mathbf{Q} = \mathbf{Q}^{0}$. Hence,  we must have $\Exp[r^{s^{\Delta b}}(x)] \geq F(b-\Delta b |0)$ since the former is the optimal solution and the latter is only a feasible solution.
Thus, we have 
\begin{equation*}
    \begin{split}
       & \Exp[r^{s^{0}}(x)] -  \Exp[r^{s^{\Delta b}}(x)]  \\
        = &\Exp[r^{s^{0}}(x)] -  F(b|0)  + F(b|0) -  F(b-\Delta b|0)  + F(b-\Delta b|0) -  \Exp[r^{s^{\Delta b}}(x)]  \\
     =& F(b|0) -  F(b-\Delta b|0)  + F(b-\Delta b|0) -  \Exp[r^{s^{\Delta b}}(x)] \\
        \leq & O(b-(b-\Delta b)) = O(\Delta b)\\        
    \end{split}
\end{equation*}
Note that $\Exp[r^{s^{0}}(x)]  = \Exp[r^{s^{*}}(x)] $, we have proved the statement. 
\end{proof}

Now we can relax Theorem \ref{thm:FAME:mainbound_weak} to Theorem \ref{thm:FAME:mainbound} by slightly modifying the subroutines in Algorithm \ref{Alg:FAME:TrainingAlgorithm}.
First, we can compute the excessive part in the cost given in Lemma \ref{lemma:FAME:InformationBound}, denoted by $b_e$.
Next, for all the subroutines in Algorithm  \ref{Alg:FAME:TrainingAlgorithm}, replace $b$ by $b-b_e$ whenever applicable.
Thus, the produced solution with high probability has $\Exp[\gamma^{[s]}(x)] \leq b $, since we already remove the $b_e$ term.
However, noting that by subtracting this $b_e$, we effective change the optimization problem by allowing a smaller budget, which is a conservative approach.
Now by Lemma \ref{lemma:FAME:deltabudget}, this incurs at most $O(b_e)$ accuracy drop.
By Lemma \ref{lemma:FAME:InformationBound}, $b_e=O(\sqrt{(\log \epsilon + \log K + \log L)/N})$, which is subsumed by the accuracy drop in Theorem \ref{thm:FAME:mainbound_weak}, which finishes the proof.
\ref{lemma:FAME:deltabudget}

\end{proof}

\section{Experimental Details}
We provide missing experimental details here.

\paragraph{Experimental Setup.}
 All experiments were run on a machine with 20 Intel Xeon E5-2660 2.6 GHz cores, 160 GB
RAM, and 200GB disk with Ubuntu 16.04 LTS as the OS.
Our code is implemented in python 3.7.

 \paragraph{ML tasks and services.}
Recall that We focus on three main ML tasks, namely, facial emotion recognition (\textit{FER}), sentiment analysis (\textit{SA}), and speech to text (\textit{STT}). 

\textit{FER} is a computer vision task, where give a face image, the goal is to give its emotion (such as happy or sad).
For \textit{FER}, we use 3 different ML cloud services, Google Vision \cite{GoogleAPI}, Microsoft Face (MS Face) \cite{MicrosoftAPI}, and Face++\cite{FacePPAPI}.
We also use a pretrained convolutional neural network (CNN) freely available from github \cite{FER_github}.
Both Microsoft Face and Face++ APIs provide a numeric value in [0,1] as the quality score for their predictions, while Google API gives a value in five categories, namely, ``very unlikely'', ``unlikely'', ``possible'', ``likely'', and ``very likely''. We transform this categorical value into numerical value by linear interpolation, i.e., the five values correspond to 0.2, 0.4, 0.6, 0.8, 1, respectively.

\textit{SA} is a natural language processing (NLP) task, where the goal is to predict if the attitude of a given text is positive or negative.
For \textit{SA}, the ML services used in the experiments are Google Natural Language (Google NLP) \cite{GoogleNLPAPI}, Amazon Comprehend (AMZN Comp) \cite{AmazonAPI}, and Baidu Natural Language Processing (Baidu NLP) \cite{BaiduAPI}. 
For English datasets, we use Vader \cite{VanderICWSM2014}, a rule-based sentiment analysis engine.
For Chinese datasets, we use another rule-based sentiment analysis tool Bixin \cite{SA_Chinese_github}.

\textit{STT} is a speech recognition task where the goal is to transform an utterance into its corresponding text.
for \textit{STT}, we use three common APIs: Google Speech \cite{GoogleSpeechAPI}, Microsoft Speech (MS Speech) \cite{MicrosoftSpeechAPI}, and IBM speech \cite{IBMAPI}.
 a deepspeech model\cite{STT_Deepspeech_github,DeepSpeech_ICML16} from github is also used.
Given the returned text from a API, we determine the API's predicted label as the label with smallest edit distance to the returned text.
For example, if IBM API produces ``for'' for a sample in AUDIOMNIST, then its label becomes ``four'', since all other numbers have larger distance from the predicted text ``for''.

\paragraph{Datasets.}
The experiments were conducted on 12 datasets.
The first four datasets, FER+\cite{dataset_FERP_BarsoumZCZ16}, RAFDB\cite{Dataset_FAFDB_li2017reliable}, EXPW\cite{Dataset_EXPW_SOCIALRELATION_ICCV2015}, and AFFECTNET\cite{Dataset_AFFECTNET_MollahosseiniHM19} are \textit{FER} datasets.
The images in FER+ was originally from the FER dataset for the ICML 2013 Workshop on Challenges in Representation, and the label was recreated by crowdsourcing.
We only use the testing portion of FER+, since the CNN model from github was pretrained on its training set.  
For RAFDB and AFFECTNET, we only use the images for basic emotions since commercial APIs cannot work for compound emotions.
For EXPW, we use the true bounding box associated with the dataset to create aligned faces first, and only pick the images that are faces with confidence larger than 0.6.

For \textit{SA}, we use four datasets, YELP \cite{Dataset_SEntiment_YELP}, IMDB \cite{Dataset_SEntiment_IMDB_ACL_HLT2011}, SHOP \cite{Dataset_SENTIMENT_SHOP}, and WAIMAI \cite{Dataset_SENTIMENT_WAIMAI}.
YELP and IMDB are both English text datasets.
YELP is from the YELP review challenge. 
Each review is associated with a rating from 1,2,3,4,5.
We transform rating 1 and 2 into negative, and rating 4 and 5 into positive.
Then we randomly select 10,000 positive and negative reviews, respectively.
IMDB is already polarized and partitioned into training and testing parts; we use its testing part which has 25,000 images.
SHOP and WAIMAI are two Chinese text datasets.
SHOP contains polarized labels for reviews for various purchases (such as fruits, hotels, computers).
WAIMAI is a dataset for polarized delivery reviews.
We use all samples from SHOP and WAIMAI.

Finally, we use the other four datasets for \textit{STT}, namely, DIGIT \cite{Dataset_Speech_DIGIT}, AUDIOMNIST\cite{Dataset_Speech_AudioMNIST_becker2018interpreting}, 
COMMAND \cite{Dataset_Speech_GoogleCommand} and  FLUENT \cite{Dataset_Speech_Fluent_LugoschRITB19}.
Each utterance in DIGIT and AUDIOMNIST is a spoken digit (i.e., 0-9).
The sampling rate is 8 kHz for DIGIT and 48 kHz for AUDIOMNIST.
Each sample in COMMAND is a spoken command such as ``go'', ``left'', ``right'', ``up'', and ``down'', with a sampling rate of 16 kHz. 
In total, there are 30 commands and a few white noise utterances.
FLUENT is another recently developed dataset for speech command.
The commands in FLUENT are typically a phrase (e.g., ``turn on the light'' or ``turn down the music'').
There are in total 248 possible phrases, which are mapped to 31 unique labels.
The sampling rate is also 16 kHz.

\begin{figure} \centering
\begin{subfigure}[Task: \textit{FER}.]{\label{fig:cost_a}\includegraphics[width=0.31\linewidth]{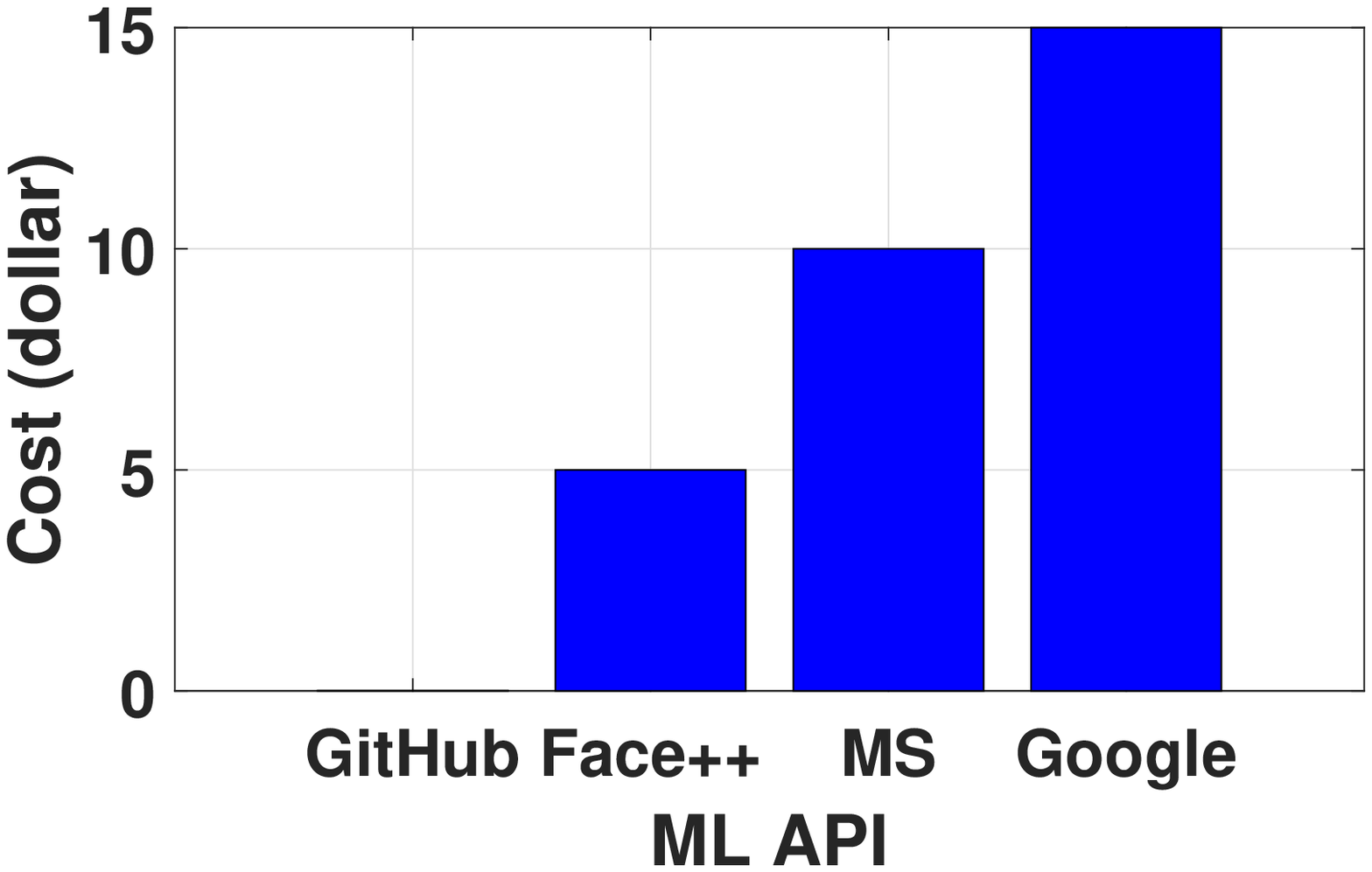}}
\end{subfigure}
\begin{subfigure}[Task: \textit{SA}.]{\label{fig:cost_b}\includegraphics[width=0.31\linewidth]{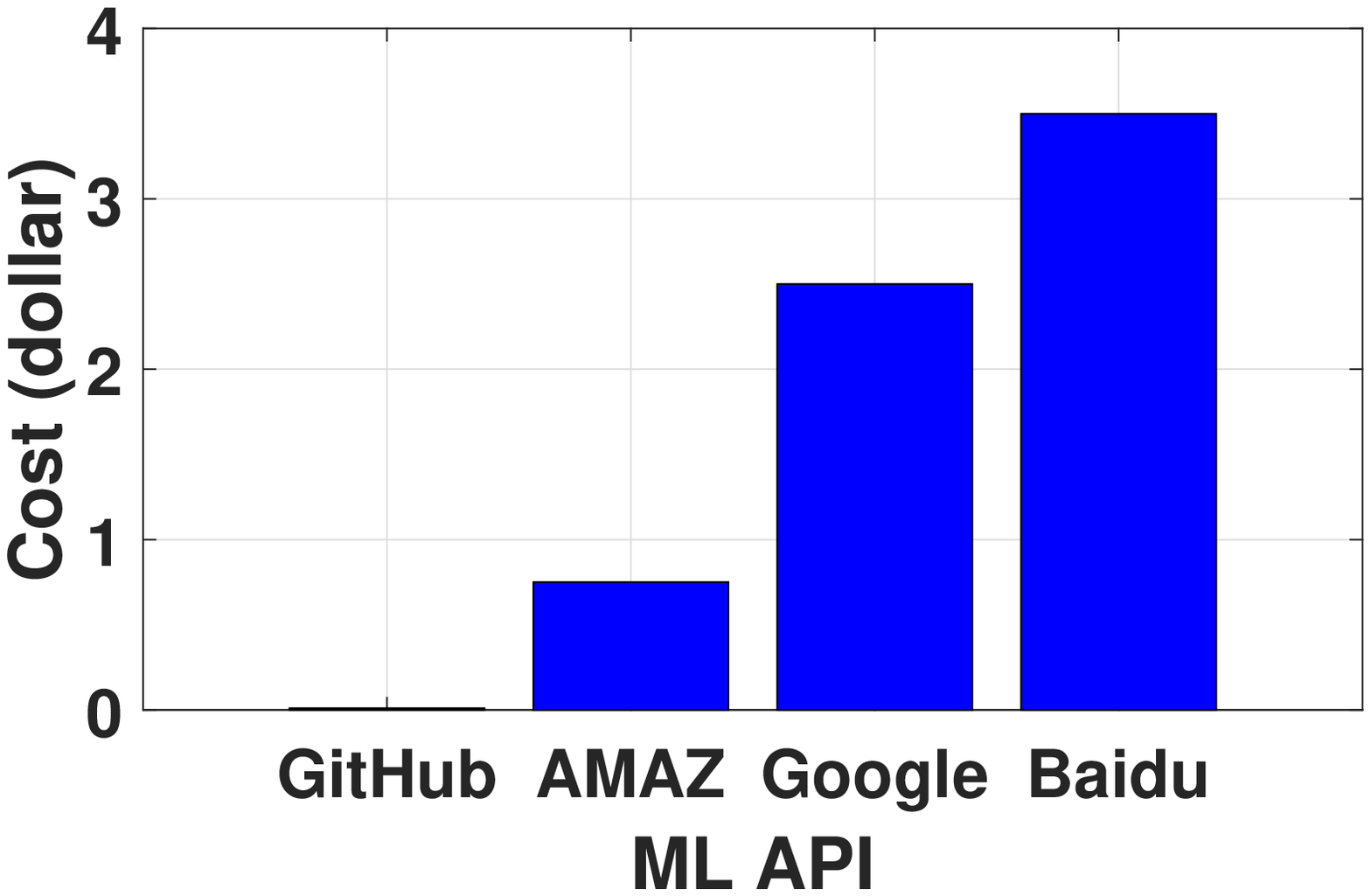}}
\end{subfigure}
\begin{subfigure}[Task: \textit{STT}.]{\label{fig:cost_c}\includegraphics[width=0.31\linewidth]{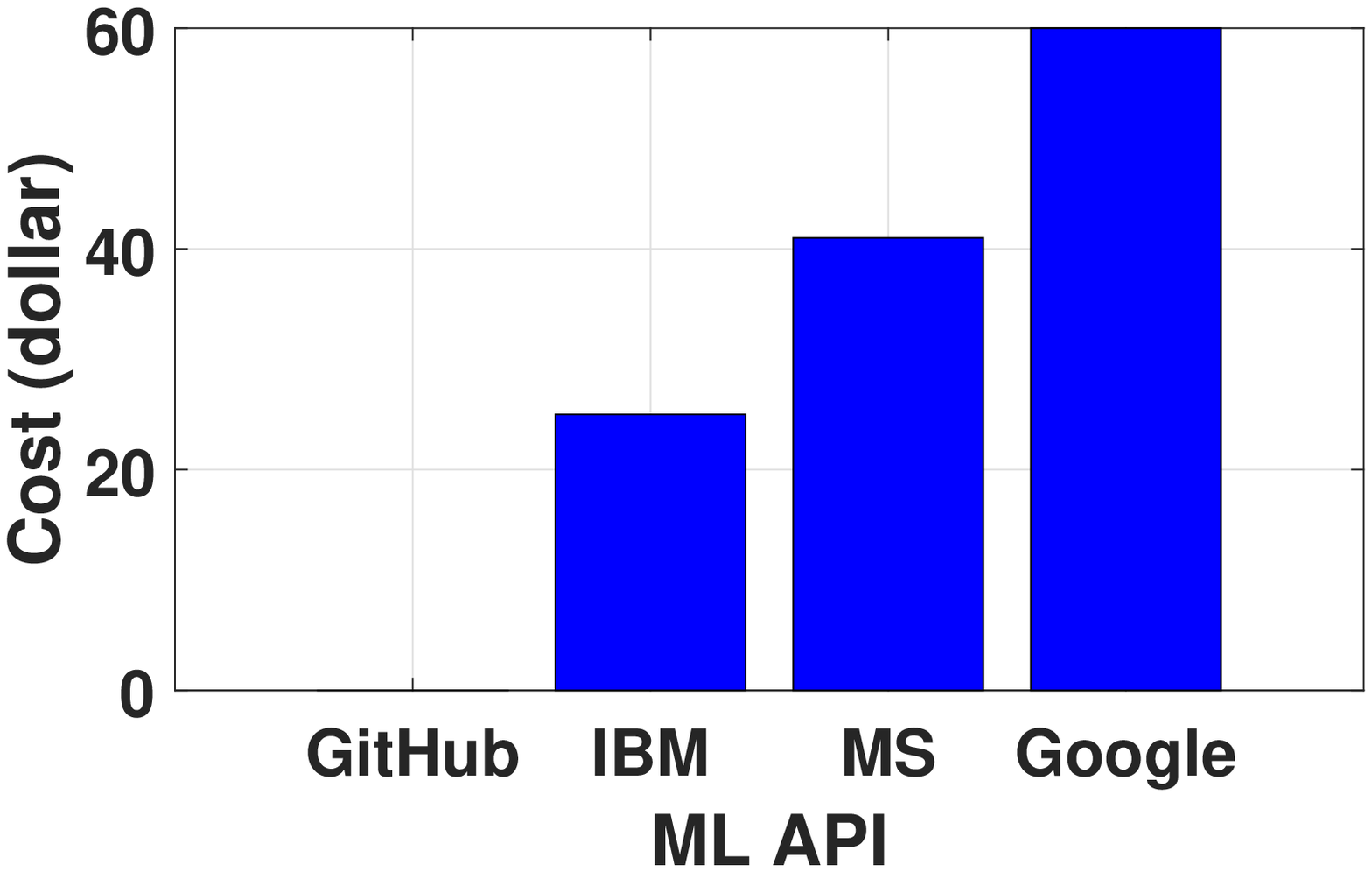}}
\end{subfigure}
	\caption{{TCost per 10,000 queries of different ML APIs. GitHub refers to  the CNN  Model \cite{FER_github} in \textit{FER}, Vader \cite{SA_English_github} and Bixin \cite{SA_Chinese_github} in \text{SA} , and DeepSpeech \cite{STT_Deepspeech_github} in \textit{STT}.  }}\label{fig:FAME:Cost}
\end{figure}
\paragraph{GitHub Model Cost.} We evaluate the inference time of all GitHub models on an Amazon EC2 t2.micro instance, which is \$0.0116 per hour.
The CNN model needs at most 0.016 seconds per 480 x 480 grey image, Bixin and Vader require at most 0.005 seconds for each text with less than 300 words, and DeepSpeech takes at most 0.5 seconds for each less than 15 seconds utterance.   
Hence, their equivalent price is \$0.0005, \$0.00016, and \$0.016 per 10,000 data points.
As shown in Figure \ref{fig:FAME:Cost}, the services from GitHub  are much less expensive than the commercial ML services.

\paragraph{Case Study Details.}
For comparison purposes, we also evaluate the performance of a mixture of experts, a simple majority vote, and a simple cascade approach on FER+ dataset.
For the mixture of experts, we use softmax for the gating network, and linear model on the domain space for the feature generation.
This results in a strategy that ends up with always calling the best expert Microsoft.
For the simple majority vote,  we first transforms each API's confidence score $q$ and predicted label $\ell$ into its probability vector $\mathbf{v}\in \R^{L}$, by $\mathbf{v}_{\ell} = q, \mathbf{v}_j = (1-q)/(L-1), j\not=\ell$.
This can be viewed as that the API gives a distribution of all labels for the input data point.
Assuming independence, we simply sum all APIs' distributions and then produce the label with highest estimated probability.
We also use a simple majority vote, where we simply return the label on which most API agrees on.
For example, if GitHub (CNN), Google, and Face++ all give a label ``surprise'', the no matter what Microsoft produces, we choose ``surprise '' as the label. We break ties randomly.

\begin{figure} \centering
\begin{subfigure}[GitHub (CNN) model]{\label{fig:cost_a}\includegraphics[width=0.30\linewidth]{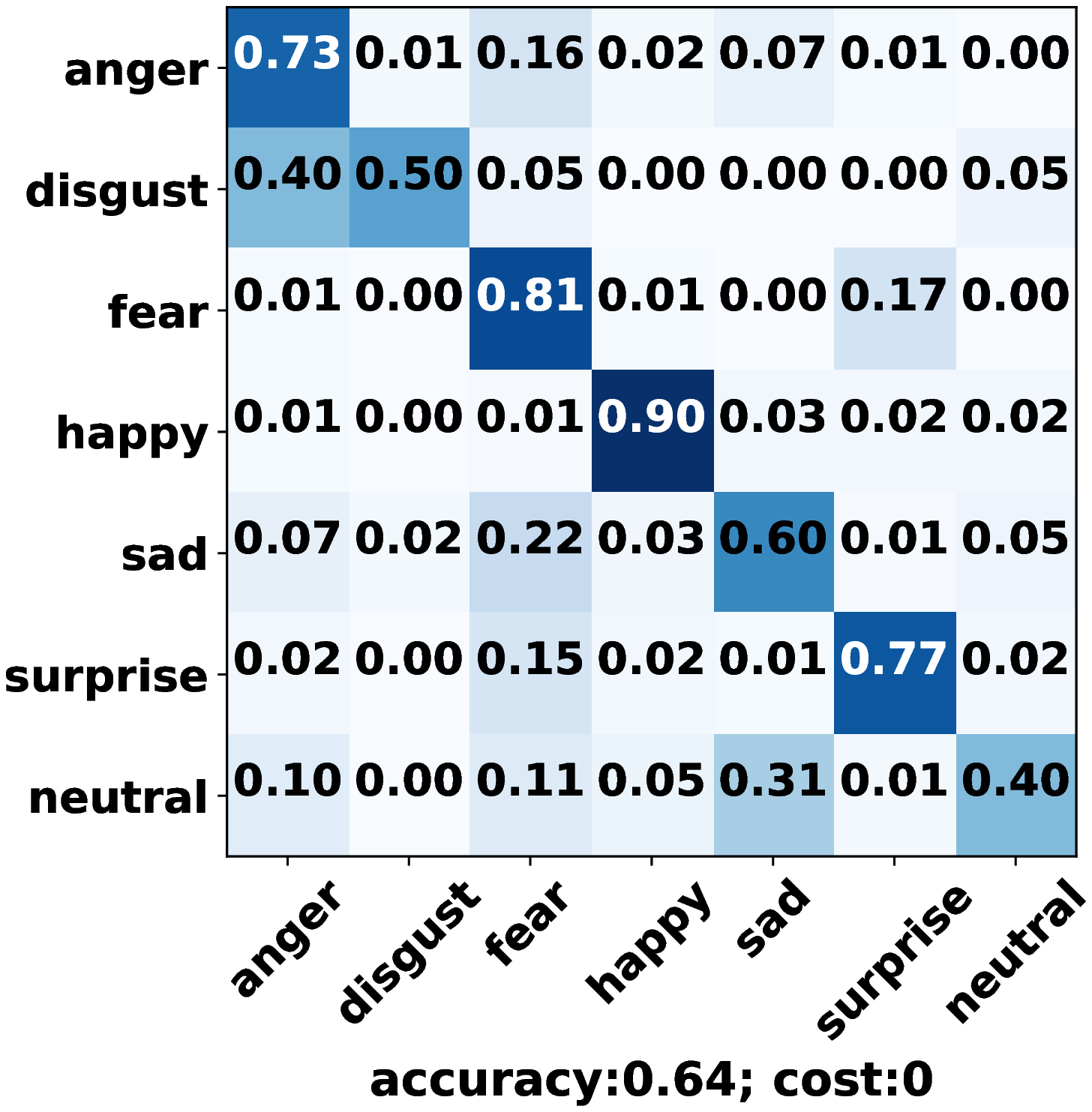}}
\end{subfigure}
\begin{subfigure}[Face++]{\label{fig:cost_b}\includegraphics[width=0.30\linewidth]{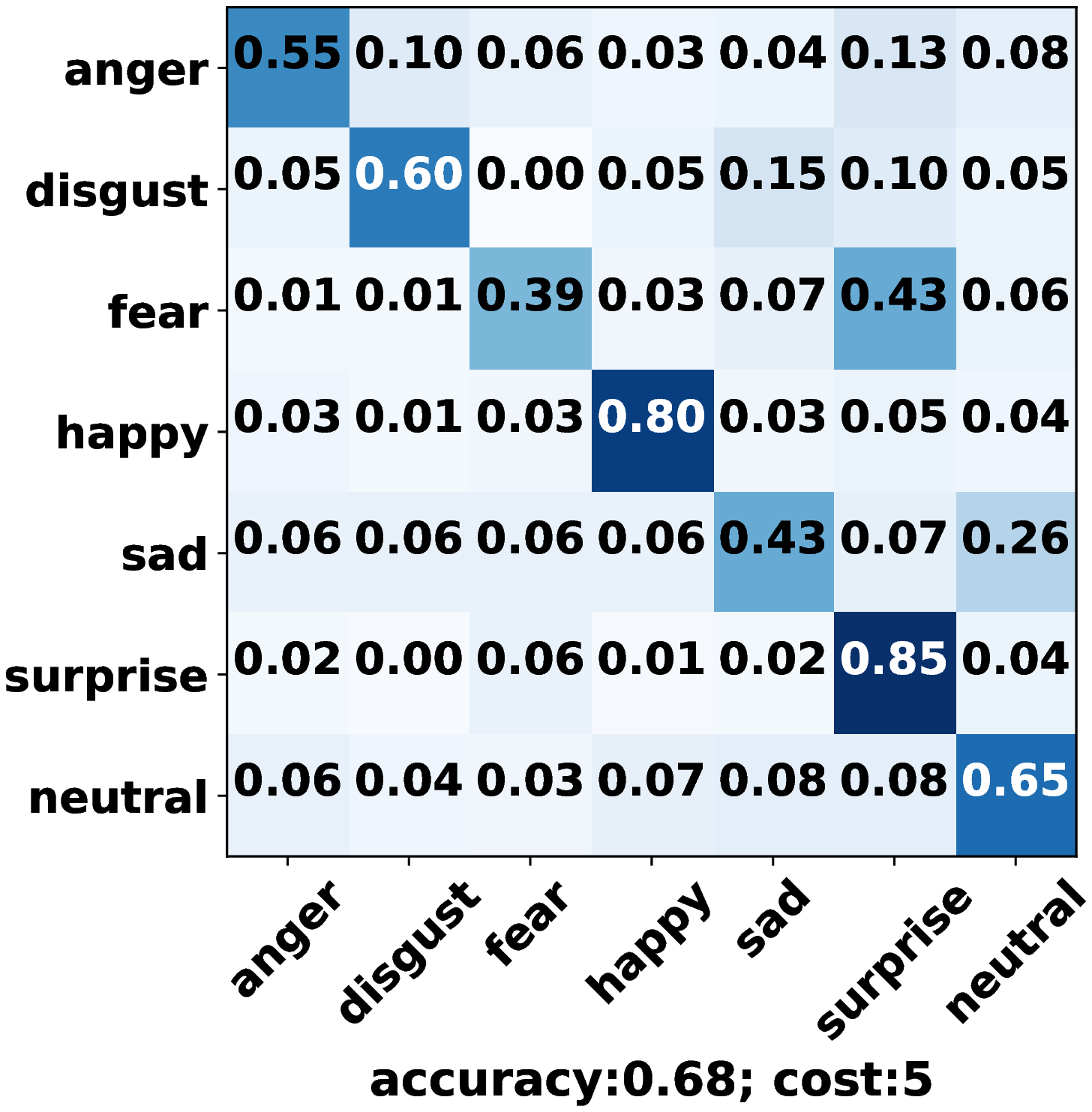}}
\end{subfigure}
\begin{subfigure}[Google]{\label{fig:cost_c}\includegraphics[width=0.30\linewidth]{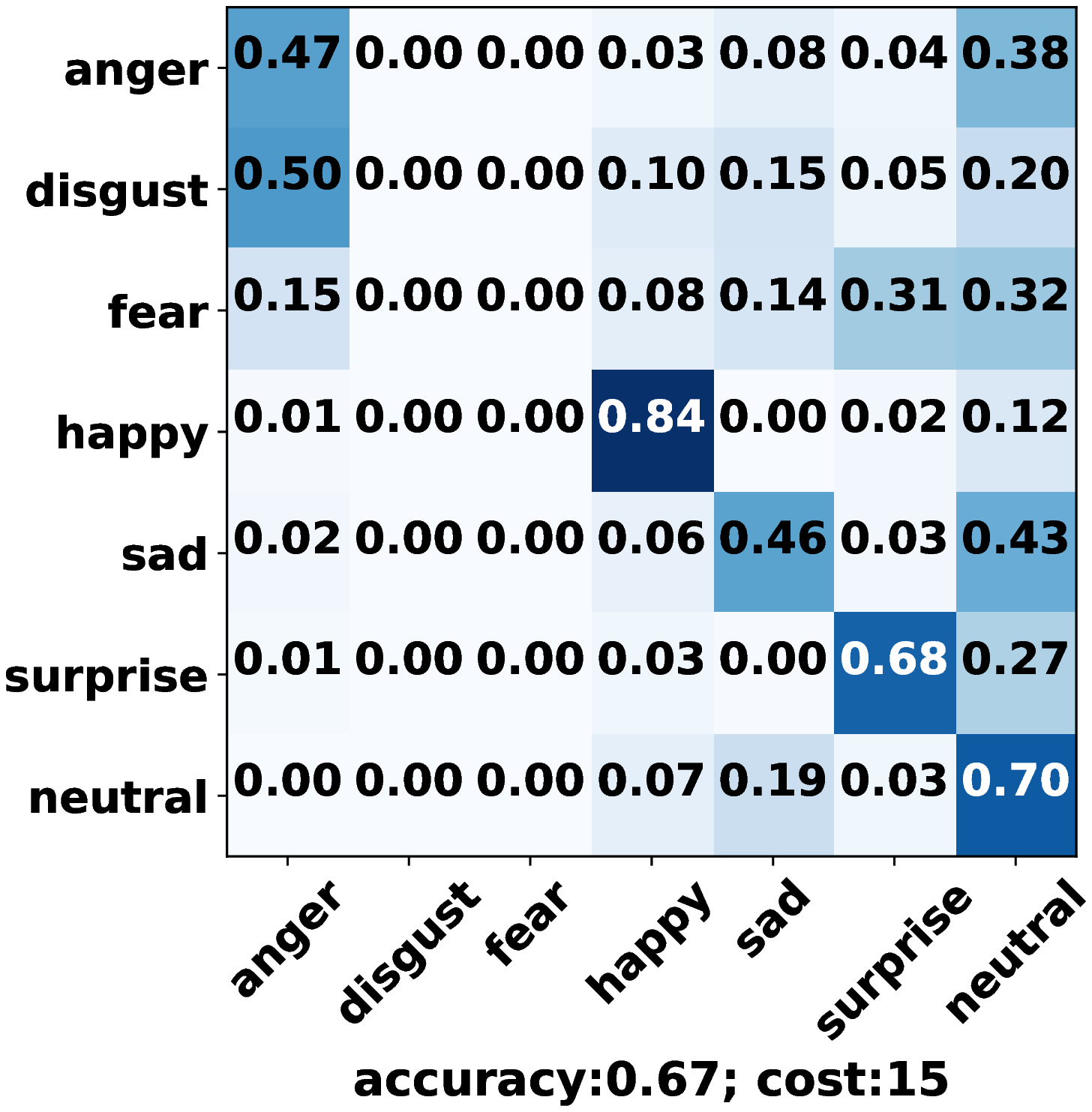}}
\end{subfigure}
\begin{subfigure}[Microsoft]{\label{fig:cost_a}\includegraphics[width=0.30\linewidth]{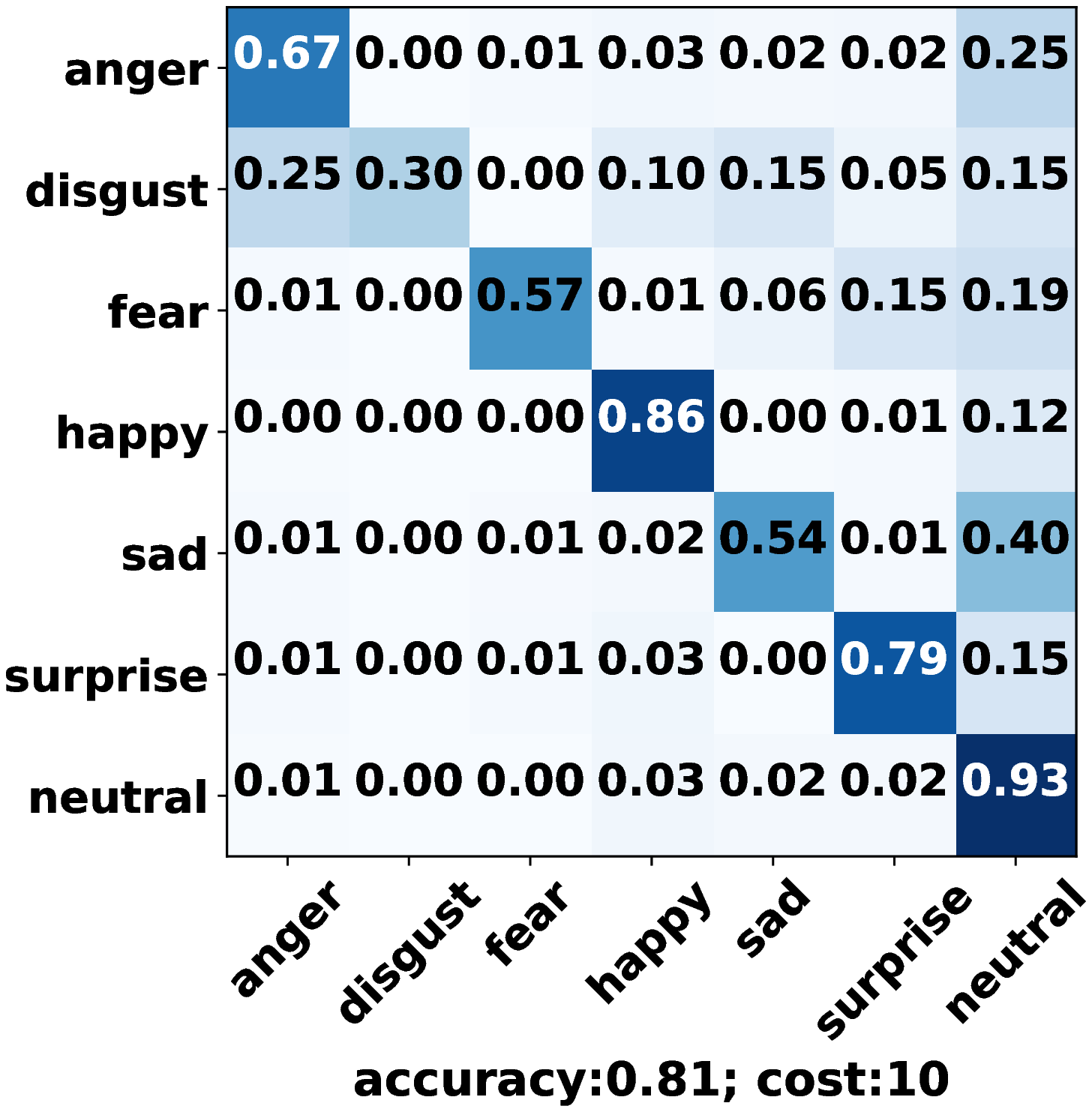}}
\end{subfigure}
\begin{subfigure}[Mix Experts]{\label{fig:cost_b}\includegraphics[width=0.30\linewidth]{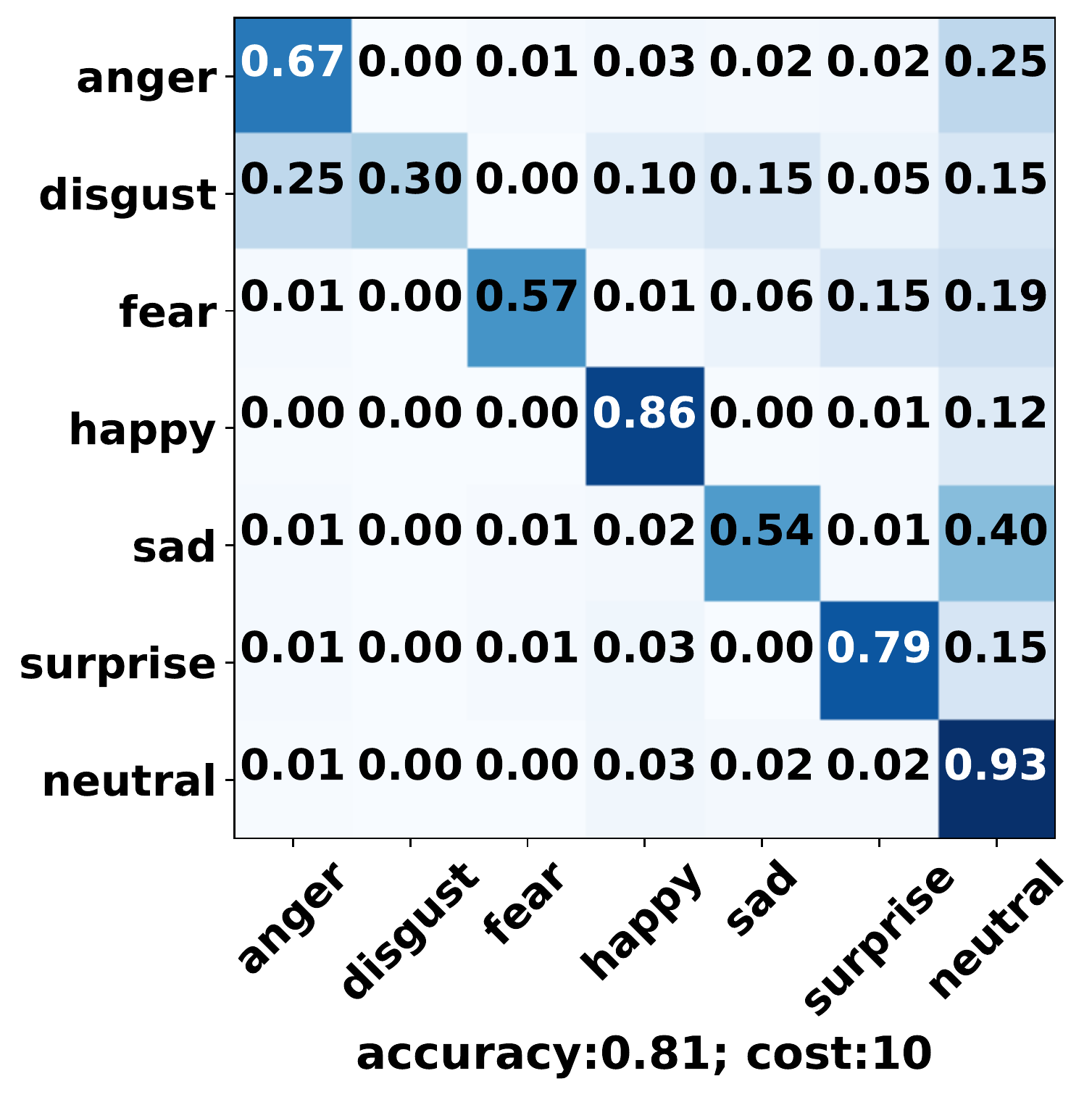}}
\end{subfigure}
\begin{subfigure}[Simple Cascade]{\label{fig:cost_c}\includegraphics[width=0.30\linewidth]{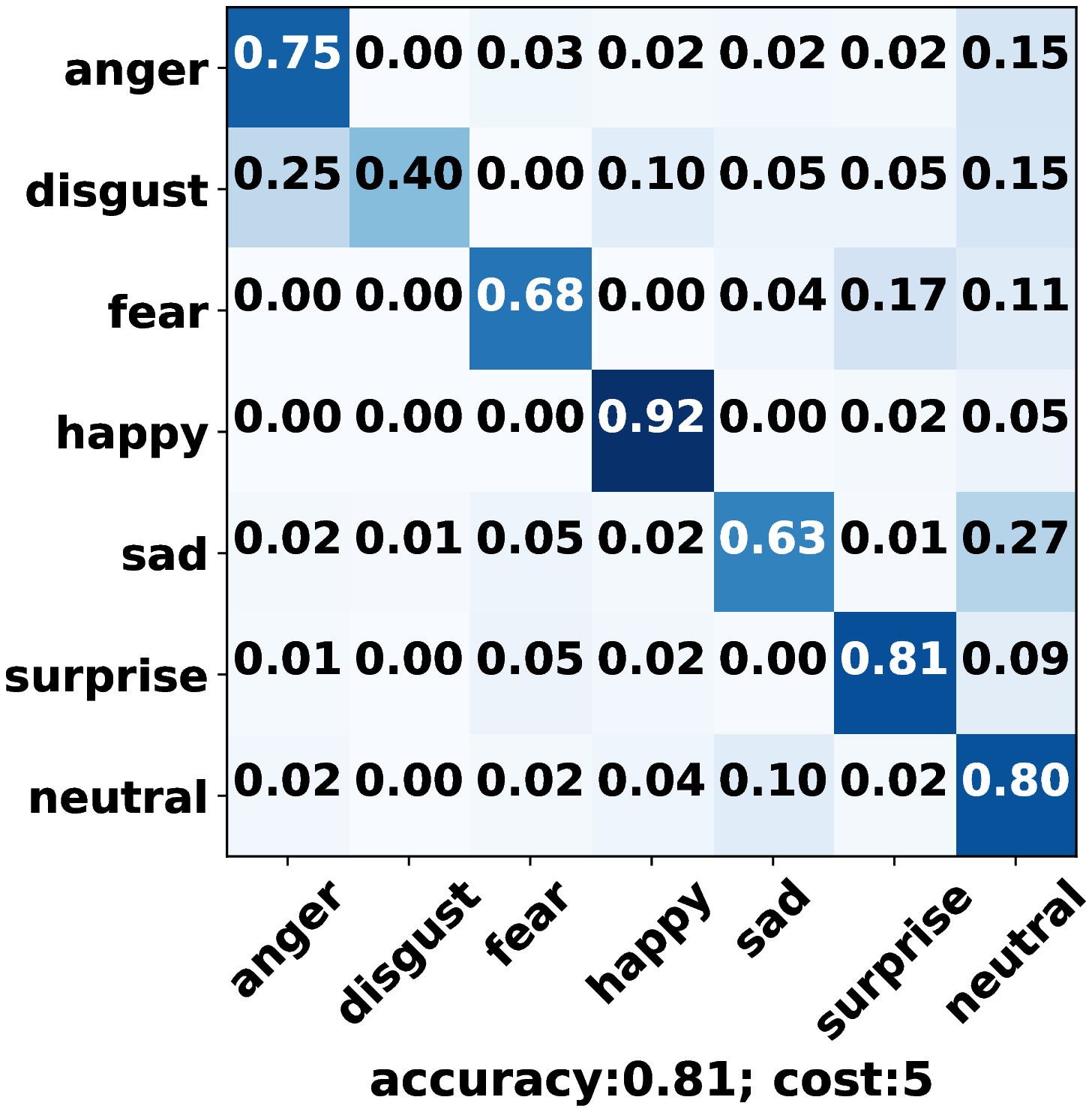}}
\end{subfigure}
\begin{subfigure}[(Simple) Majority Vote]{\label{fig:cost_a}\includegraphics[width=0.30\linewidth]{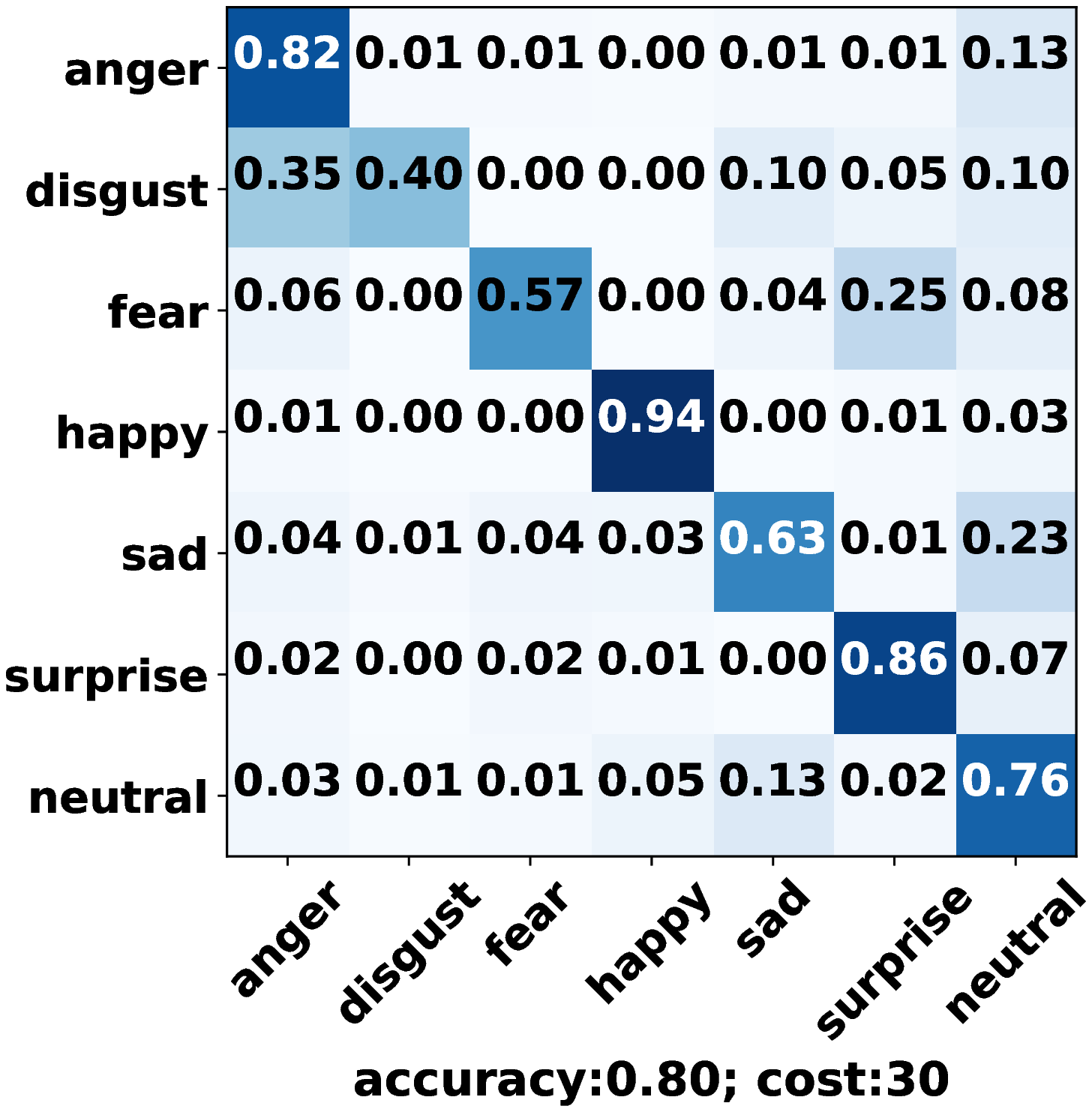}}
\end{subfigure}
\begin{subfigure}[Majority Vote]{\label{fig:cost_b}\includegraphics[width=0.30\linewidth]{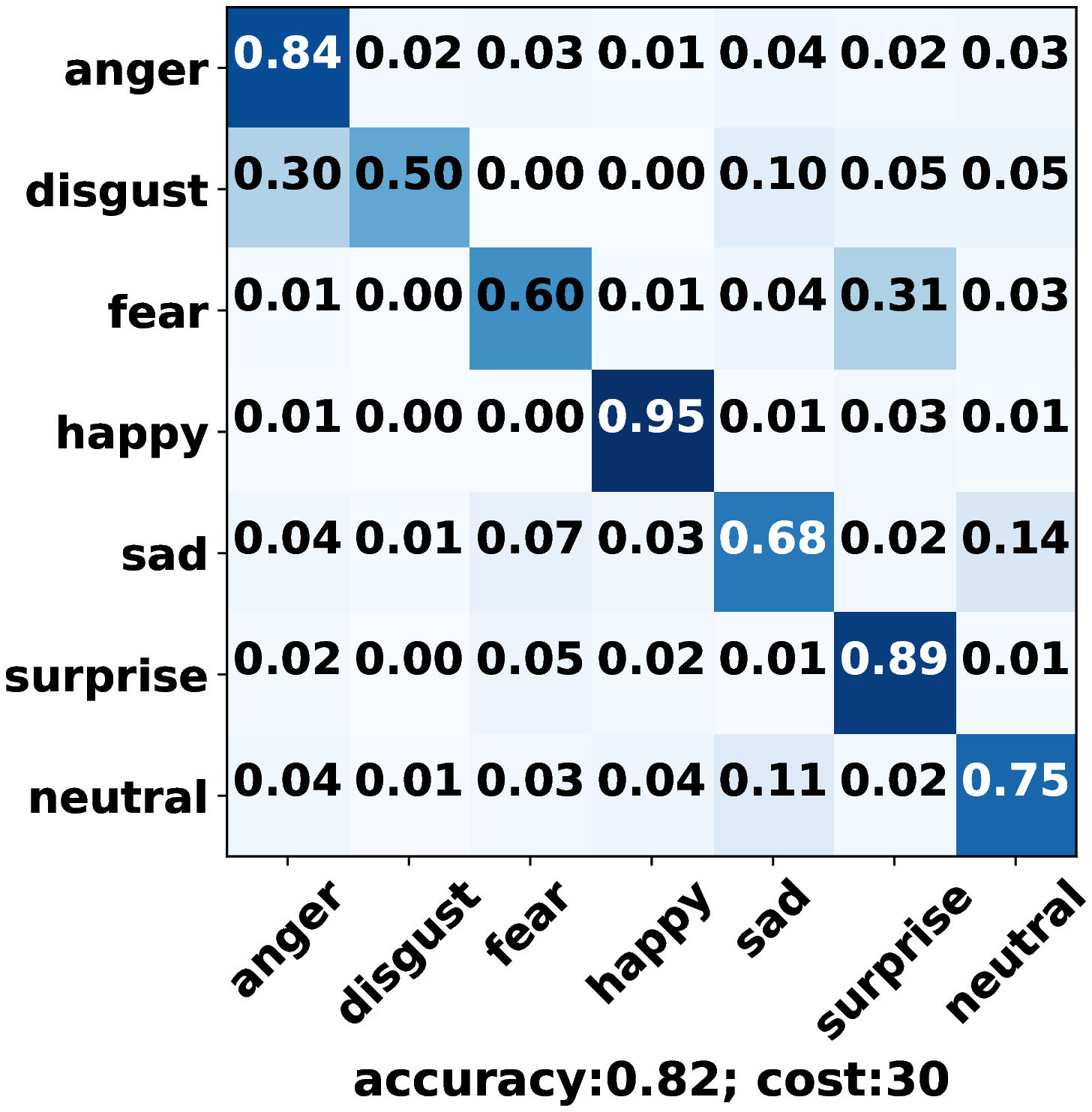}}
\end{subfigure}
\begin{subfigure}[\systemname{}]{\label{fig:cost_c}\includegraphics[width=0.30\linewidth]{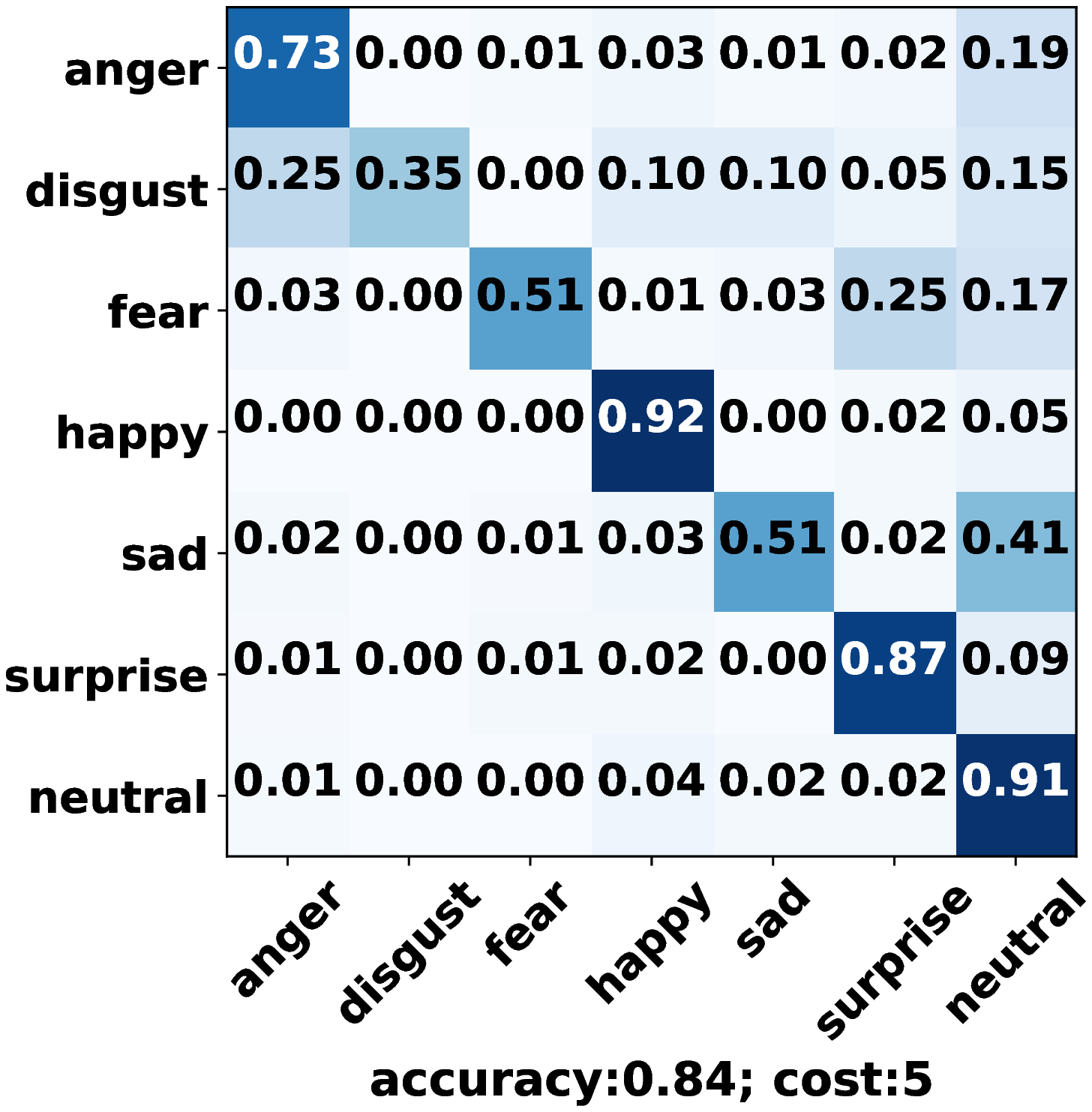}}
\end{subfigure}
		\caption{Confusion matrix annotated with overall accuracy and cost on  FER+ testing. The y-axis corresponds to the true label and x-axis represents the predicted label. Each entry in a confusion matrix is the likelihood that its corresponding label in x-axis is predicted given the corresponding true label in y-axis.
	For example, the 0.87 in (i) means that for all surprise images, \systemname{} correctly predicts 87\% of them as surprise, }\label{fig:FAME:confusionmatrix}
\end{figure}

Figure \ref{fig:FAME:confusionmatrix} shows the confusion matrix of \systemname{}, along with all ML services and the other approaches (namely, mixture of experts, simple cascade, (simple majority vote), and majority vote).
Among all the four services, we first note that there is an accuracy disparity for different facial emotions.
In fact, GitHub (CNN) gives the highest accuracy on anger images (0.73\%), fear (0.81\%), happy (0.90\%) and sad (0.60\%), Face++ is best at disgust emotion (0.60\%) and surprise (0.85\%), while Microsoft is best at neutral (90\%).
Meanwhile, GitHub (CNN) gives a poor performance for neutral images, Face++ can hardly tell the differences between fear and surprise, and Google has a hard time distinguishing between anger and disgust images.
This implies bias (and thus strength and weakness) from each ML API, leading to opportunities for  optimization. 
We would also like to note that such biases may be of independent interest and explored for fairness study in the future.

We notice that the mixture of expert approach has the same confusion matrix as the Microsoft API.
This is because the simple mixture of experts simply learns to always use the Microsoft API.
Noting that we use a simple linear gating on the raw image space, this  probably implies that Microsoft API has the best performance on any subspace in the raw image space produced by any hyperplane. 
More complicated mixture of experts approaches may lead to better performance, but requires more training complexity.
Again, unlike \systemname{}, mixture of experts does not allow users to specify their own budget/accuracy  constraints.

Simple cascade approach allows accuracy cost trade-offs. 
As shown in Figure \ref{fig:FAME:Cost}(f), while reaching the same accuracy as the best commercial API (Microsoft), it only asks for half of the cost.
In fact, simple cascade uses GitHub (CNN) and Microsoft as the base service and add-on service with a fixed threshold for all labels.
As a result, compared to Microsoft API, the prediction accuracy of neutral images drops significantly, while the accuracy on all the other labels increases, and thus resulting in the same accuracy.
\begin{figure}
	\centering
	\includegraphics[width=1\linewidth]{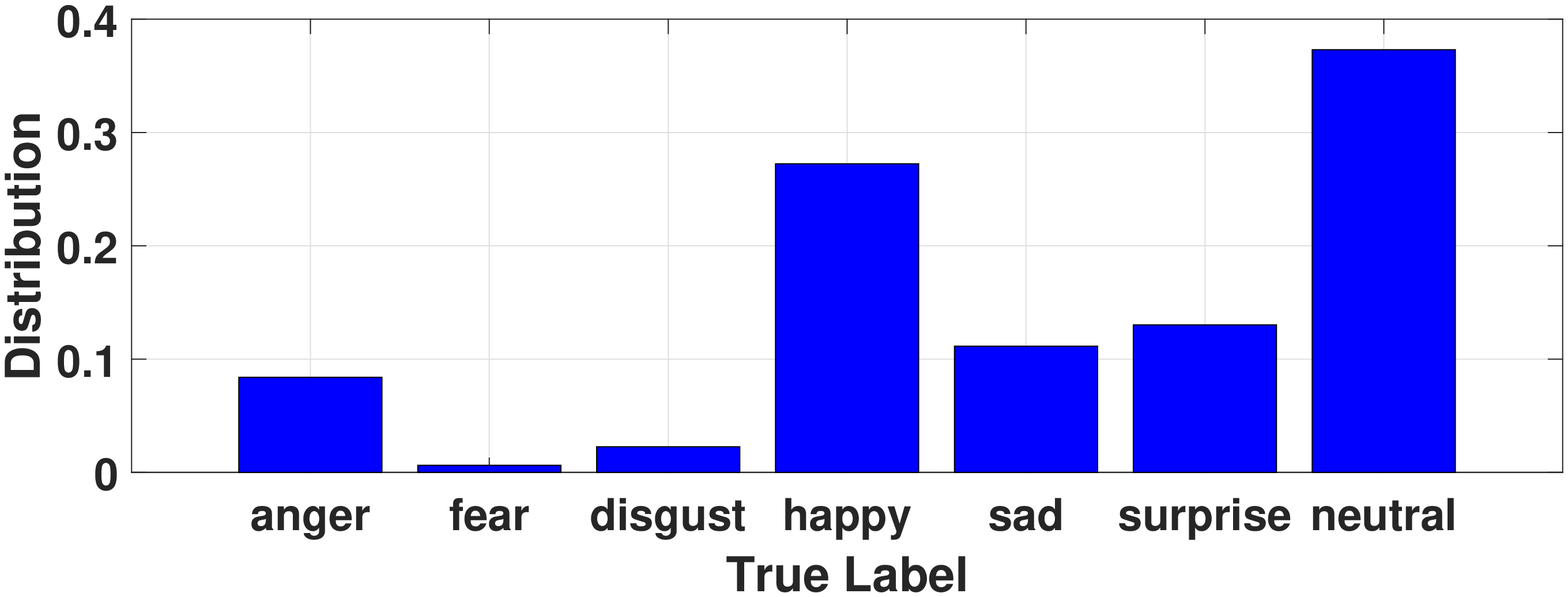}
	\caption{Label distribution on dataset FER+. Most of the facial images are neutral and happy faces, and only a few are fear and disgust.}\label{fig:FAME:FERDataDistribution}
\end{figure}
\systemname{}, also with half of the cost of Microsoft API, actually gives an accuracy (84\%) even higher than that of Microsoft API (81\%).
In fact, \systemname{} identifies that only a vert small portion of images are disgust, and thus slightly sacrifices the accuracy on disgust images to improve the accuracy on all the other images.
Compared to the simple cascade approach in Figure \ref{fig:FAME:Cost} (f), \systemname{}, as shown in Figure \ref{fig:FAME:Cost} (i), produces higher accuracy on all classes of images except disgust images. 
Compared to Microsoft API (Figure \ref{fig:FAME:Cost}), \systemname{} slightly hurts the accuracy on fear, sad, and neutral images, but significantly improve the accuracy on happy and other images.
Note that the strategy learned by \systemname{} depends on the data distribution.
As shown in Figure \ref{fig:FAME:FERDataDistribution},  most images are neutral and happy, and thus a slight drop on neutral images is worthy in exchange of a large improvement on happy images.
Depending on the training data distribution, \systemname{} may have learned different strategies as well. 

Finally we note that while (simple) majority vote gives a poor accuracy (80\% in Figure \ref{fig:FAME:confusionmatrix} (g)), the majority vote approach does lead to an accuracy (82\%) higher than Microsoft API, although it is still lower than \systemname{}'s accuracy (84\%).
In addition, ensemble methods like majority vote need access to all ML APIs, and thus requires a cost of 30\$, which is 5 times as large as the cost of \systemname{}.
Hence, they may not help reduce the cost effectively.

\end{document}